\documentclass{article}
\usepackage{arxiv}

\usepackage{graphicx}
\usepackage{subcaption}
\usepackage{float}
 \usepackage{algorithm}
 \usepackage{algorithmic}
 \usepackage{natbib}
 \usepackage{adjustbox}

\usepackage{booktabs}
\usepackage{array}
\usepackage{longtable}
\usepackage{header}
\usepackage{fancyhdr}
\renewcommand{\thefootnote}{}

\usepackage[utf8]{inputenc} 
\usepackage[T1]{fontenc}    
\usepackage{hyperref}       
\usepackage{url}            
\usepackage{booktabs}       
\usepackage{amsfonts}       
\usepackage{nicefrac}       
\usepackage{microtype}      
\usepackage{xcolor}     

\title{Universal Sequence Preconditioning}

\author{
  Annie Marsden \footnotemark[1] \And
  Elad Hazan\footnotemark[1] \hspace{0.1mm} \footnotemark[2]
}

\begin{document}
\maketitle

\renewcommand{\thefootnote}{\fnsymbol{footnote}}
\footnotetext[1]{Google Deepmind}
\footnotetext[2]{Princeton University}
\renewcommand{\thefootnote}{\arabic{footnote}} 
\begin{abstract}

We study the problem of preconditioning in sequential prediction.
From the theoretical lens of linear dynamical systems, we show that convolving the target sequence corresponds to applying a polynomial to the hidden transition matrix.
Building on this insight, we propose a universal preconditioning method that convolves the target with coefficients from orthogonal polynomials such as Chebyshev or Legendre.
We prove that this approach reduces regret for two distinct prediction algorithms and yields the first ever sublinear and hidden-dimension–independent regret bounds (up to logarithmic factors) that hold for systems with marginally stable and asymmetric transition matrices.
Finally, extensive synthetic and real-world experiments show that this simple preconditioning strategy improves the performance of a diverse range of algorithms, including recurrent neural networks, and generalizes to signals beyond linear dynamical systems.
\end{abstract}

\section{Introduction}

In sequence prediction the goal of the learner is to predict the next token accurately according to a specified loss function, such as the mean square error or cross-entropy.
This fundamental problem in machine learning has gained increased importance with the rise of large language models, which perform sequence prediction on tokens using cross entropy.
The focus of this paper is {\it preconditioning}, i.e. modifying the target sequence to make it easier to learn.
A classic example is {\bf differencing}, introduced by Box and Jenkins in the 1970s \cite{box1976time}, which transforms observations $\y_1,\y_2,\dots$ into successive differences,
$$ \y_1 - \y_0 , \  \y_2-\y_1 ,\ ... ,\  \y_t- \y_{t-1}, ... $$
It is widely acknowledged that learning this sequence can be ``easier" than learning the original sequence for a large number of modalities. In this work we seek a more general framework for sequence preconditioning that captures the same intuition behind differencing and extends it to a broader class of transformations. The question we ask is 

\begin{center}
{\it What is the general form of sequence preconditioning that enables provably accurate learning?}
\end{center}

We address this question by introducing a preconditioning method which takes in $n$ fixed coefficients $c_0, \dots, c_n$ and converts the sequence of observations $\y_1, \dots, \y_t, \dots$ to the sequence of convolved observations\footnote{This recovers differencing when $n=2$, $c_0 = 1$, and $c_1 = -1$.}
\begin{equation*}
    c_0 \y_0, \; c_0 \y_1 + c_1 \y_0, \; \dots, \; \sum_{i = 0}^n c_i \y_{t-i}, \; \dots 
\end{equation*}
From an information-theoretic perspective, approaches of this kind seem futile—predicting $\y_t$ or $\sum_i c_i \y_{t-i}$ seems equally hard in an adversarial setting.
Yet we show that when the data arises from a linear dynamical system (LDS), there exists a \emph{universal} form of preconditioning that provably improves learnability, independent of the specific system.
In the LDS setting, we show that preconditioning significantly strengthens existing prediction methods, leading to new regret bounds.
Here, preconditioning has an elegant effect: the preconditioning filter forms coefficients of an $n$ degree polynomial, and the hidden system transition matrix is evaluated on this polynomial-- potentially shrinking the domain.
In this setting, shrinking the learnable domain is akin to making the problem ``easier to learn'', a relationship that is formalized by \cite{hazan2022introduction}.
This allows us to prove the first dimension-independent sublinear regret bounds for asymmetric linear dynamical systems that are marginally stable.

\subsection{Our results}
Our main contribution is \emph{Universal Sequence Preconditioning}, a novel method of sequence preconditioning which convolves the target sequence with the coefficients of the $n$-th monic Chebyshev polynomial.
We give a more general form of preconditioning, allowing arbitrary user-specified coefficients, in Algorithm \ref{alg:seq-precond-offline} and an online version in Algorithm \ref{alg:seq-precond} (Appendix \ref{appendix:online_version}).
We analyze the effect of Universal Sequence Preconditioning on two canonical sequence prediction algorithms in the online setting: (1) convex regression and (2) spectral filtering.
In either case, the results are impressive-- yielding the first known sublinear regret bounds as compared to the optimal ground-truth predictor that are simultaneously (1) applicable to marginally stable systems, (2) independent of the hidden dimension (up to logarithmic factors), and (3) applicable to systems whose transition matrix is asymmetric\footnote{Our results only hold for asymmetric matrices whose eigenvalues have imaginary component bounded above by $O(1/\log(T))$. This is somewhat tight, see Section~\ref{subsection:memory}.} (see Table~\ref{tab:my_label} ). 

\begin{algorithm}[h]
		\caption{\label{alg:seq-precond-offline} General Sequence Preconditioning (Offline Version)}
		\begin{algorithmic}[1]
            \STATE \textbf{Training}
			\STATE Input: training data $(\uv^{1:N}_{1:T}, \y^{1:N}_{1:T})$ where $(\uv^i_t,\y^i_t)$ is the $t$-th input/output pair in the $i$-th sequence; coefficients $c_{0:n}$; prediction algorithm $\mathcal{A}$.
            \STATE Assert $c_0 = 1$.
			\FOR {$i = 1$ to $N$}
            \STATE $\y^{\textrm{preconditioned},i}_{1:T} \gets \texttt{convolution}(\y^{i}_{1:T}, c_{0:n})$ {\hfill $\lhd$ $\y^{\textrm{preconditioned},i}_t = \y^i_t + \sum_{j = 1}^n c_j \y^i_{t-j}$}
            \ENDFOR
			\STATE Train $\mathcal{A}$ on preconditioned data $\left( \uv^{1:N}_{1:T}, \y^{\mathrm{preconditioned}, 1:N}_{1:T}\right)$.
			\STATE \textbf{Test Time}
			\FOR {$t=1$ to $T$}
            \STATE Receive $\uv_t$.
			\STATE Predict $\hat{\y}_t \gets \mathcal{A} \left( \uv_{1:t}, \y_{1:(t-1)} \right) - \sum_{i = 1}^n c_i \y_{t-i}$.
            \STATE Receive $\y_t$.
			\ENDFOR
		\end{algorithmic}
\end{algorithm}

First, applying USP to standard convex regression results in regret $\tilde{O}(T^{-2/13})$, which holds simultaneously across the three settings above and remains dimension-independent.
For comparison, a naive analysis of regression yields a vacuous regret bound of $O(T^{5/2})$ on marginally stable systems.
Second, combining USP with a variant of \emph{spectral filtering} \cite{hazan2017learning} that uses novel filters, the algorithm is able learn a broader class of linear dynamical systems-- in particular systems whose hidden transition matrix may be asymmetric.
The enhanced method achieves regret $\tilde{O}(T^{-3/13})$, the best known rate under the joint conditions of (1)-(3) discussed above: marginal stability, dimension independence, and asymmetry.
Both results require that the transition matrix eigenvalues have imaginary parts bounded by $O(1/\log T)$—a near-tight condition for achieving dimension-free regret.
Further discussion on this appears in Appendix~\ref{subsection:memory}.

\begin{table}[ht]
    \centering
    \begin{tabular}{|c|c|c|c|}
    \hline
        Method & Marginally stable & $d_{\text{hidden}}$-free & Asymmetric \\
    \hline\hline
        Sys-ID & $\times$ & $\times$ & $\checkmark$  \\
    \hline
        Regression (open-loop) & $\times$ & $\checkmark$  & $\checkmark$  \\
    \hline
        Regression (closed-loop) & $\checkmark$ & $\times$ & $\checkmark$  \\
  
    \hline
        Spectral Filtering & $\checkmark$ & $\checkmark$ & $\times$  \\
    \hline
        USP + Regression & $\checkmark$ & $\checkmark$ & $\sim\checkmark$  \\
    \hline
        USP + Spectral Filtering & $\checkmark$ & $\checkmark$ & $\sim\checkmark$  \\
    \hline
    \end{tabular}
    \caption{Comparison of methods for learning LDS. USP extends learning to mildly asymmetric matrices with bounded complex eigenvalues.}
    \label{tab:my_label}
\end{table}

Empirical results in Section \ref{sec:experiments} demonstrate that USP consistently improves performance across diverse algorithms—including regression, spectral filtering, and neural networks—and across data types extending beyond linear dynamical systems.
\subsection{Intuition for Universal Sequence Preconditioning}
\label{sec:intuition}

We now give some brief intuition for the result.
Linear dynamical systems (LDS) are perhaps the most basic and well studied dynamical systems in engineering and control science.
Given input vectors $\uv_1, \dots, \uv_T \in \mathbb{C}^{d_{\textrm{in}}}$, the system generates a sequence of output vectors $\y_1, \dots \y_T \in \mathbb{C}^{d_{\textrm{out}}}$ according to the law 
\begin{align}
\label{eqn:LDS}
    \mat{x}_{t+1} & = \A \mat{x}_t + \B \uv_t, \qquad 
    \y_t  = \C \mat{x}_t + \mat{D} \uv_t,
\end{align}
where $\mat{x}_0, \dots, \mat{x}_T \in \mathbb{C}^{d_{\textrm{hidden}}}$ is a sequence of hidden states and $(\A, \B, \C, \mat{D})$ are matrices which parameterize the LDS.
We assume w.l.o.g. that $\mat{D} = 0$. We can factor out the hidden state $\mat{x}_t$ so that the observation at time $t$ is
\begin{equation*}
 \y_t = \sum_{s = 1}^{t} \C \A^{t-s} \B \uv_{s}.
\end{equation*} 
Given coefficients $c_{0:n} = ( c_0,\dots, c_n)$ let  
\begin{equation}
\label{def:polynomial}
    p_n^{c}(x) \defeq \sum_{i = 0}^n c_i x^{n-i}.
\end{equation}
Consider a ``preconditioned'' target at time $t$ to be a linear combination of $\y_{t:t-n}$ with coefficients $c_{0:n}$.
A key insight is the following identity,
\begin{equation*}
\sum_{i=0}^n c_i \y_{t-i} = \sum_{s = 0}^{n-1} \left( \sum_{i = 0}^s c_i \C \A^{s-i} \B \right) \uv_{t-s} + \sum_{s = 0}^{t-n-1} \C p_n^{c}(\A) \A^s \B \uv_{t-n-s}.
\end{equation*}
If we take $c_0 = 1$ (i.e. a monic polynomial), we can re-write $\y_t$ as
\begin{equation}
\label{eqn:lds_breakdown}
    \y_t = - \underbrace{\sum_{i=1}^n c_i \y_{t-i}}_{\aleph_0} + \underbrace{ \sum_{s = 0}^{n-1}   \sum_{i = 0}^s c_i \C \A^{s-i} \B  \uv_{t-s}}_{\aleph_1} + \underbrace{ \sum_{s = 0}^{t-n-1} \C p_n^{c}(\A) \A^s \B \uv_{t-n-s}}_{\aleph_2}.
\end{equation}

This expression highlights our approach as a balance of three terms:
\begin{enumerate}
    \item[$\aleph_0$] The {\bf universal preconditioning} term: it depends only on the coefficients $c_{0:n}$and not on any learning algorithm.
	\item[$\aleph_1$] A term learnable via convex relaxation and regression, for example by denoting
\begin{equation*}
    \Q^{\textrm{learned}}_s = \sum_{i = 0}^s c_i \C \A^{s-i} \B.
\end{equation*}
The diameter of the coefficient $\Q_s$ depends on the magnitude of the coefficients $c_{0:n}$.
\item[$\aleph_2$] The residual term with polynomial $p_n^{c}(\A)$. By a careful choice of coefficients $c_{0:n}$, we can force this term to be very small.
\end{enumerate}

The main insight we derive from this expression is the inherent tension between two terms $\aleph_1,\aleph_2$.
The polynomial $p_n^{c}(x)$ and its coefficients $c_0, \dots, c_n$ control two competing effects:
\begin{enumerate}
    \item The preconditioning coefficients grow larger with the degree $n$ of the polynomial and the magnitude of the coefficients $c_i$.
	A higher degree polynomial and larger coefficients increase the diameter of the search space over the preconditioning coefficients, and therefore increase the regret bound stemming from the $\aleph_1$ component learning.
	\item On the other hand, a larger search space can allow a broader class of polynomials $p_n(\cdot)$ which can better control of the magnitude of $p_n(\A)$, and therefore reduce the search space of the $\aleph_2$  component.
\end{enumerate}

What choice of polynomial is best? This work considers the Chebyshev polynomial.
The reason is the following property of the $n$-th monic Chebyshev polynomial:
\begin{equation*}
    \max_{\lambda \in [-1,1]} \abs{p_n(\lambda)} \leq 2^{-(n-1)}.
\end{equation*}
As an example, consider any LDS whose hidden transition matrix $\mat{A}$ is diagonalizable and has eigenvalues in $[-1,1]$\footnote{This work considers a broader class of hidden transition matrices.}.
By the above property, observe that $\norm{p_n(\mat{A})}_{\infty} \leq 2 \cdot 2^{-n}$, therefore shrinking the $\aleph_2$ term at a rate exponential with the number of preconditioning coefficients.
We pause to remark on the \emph{universality} of this choice of polynomial.
Indeed, one could instead have chosen the preconditioning coefficients to depend on $\mat{A}$ so that $p_n^{\mathbf{c}}(\cdot)$ is the characteristic polynomial of $\mat{A}$.
By the Cayley-Hamilton theorem, $ p_n(\A) = 0$. This means that $\aleph_2$ term is canceled out completely.
However this would have required knowledge of the spectrum of $\mat{A}$.
The Chebyshev polynomial, on the other hand, is agnostic to the particular hidden transition matrix.
Moreover, even if the spectrum of $\mat{A}$ were known, choosing the preconditioning coefficients to form the characteristic polynomial would result in an algorithm which must learn hidden dimension many parameters, which is prohibitive.
Instead, the degree of the Chebyshev polynomial must only grow logarithmically with the hidden dimension.
\subsection{Related work}

Our manuscript is technically involved and incorporates linear dynamical systems, spectral filtering, complex Chebyshev and Legendre polynomials, Hankel and Toeplitz matrix eigendecay, Gaussian quadrature and other techniques.
The related work is thus expansive, and due to space limitations we give a detailed treatment in Appendix \ref{sec:related_work}.
Preconditioning in the context of time series analysis has roots in the classical work of Box and Jenkins \cite{box1976time}.
In their foundational text they propose differencing as a method for making the time series stationary, and thus amenable to statistical learning techniques such as ARMA (auto-regressive moving average) \cite{anava2013online}.
The differencing operator can be applied numerous times, and for different lags, giving rise to the ARIMA family of forecasting models.
Identifying the order of an ARIMA model, and in particular the types of differencing needed to make a series stationary, is a hard problem. This is a special case of the problem we consider: differencing corresponds to certain coefficients of preconditioning the time series, whereas we consider arbitrary coefficients.
For a thorough introduction to modern control theory and exposition on open loop / closed loop predictors, learning via regression, and spectral filtering, see \cite{hazan2022introduction}.
The fundamental problem of learning in linear dynamical systems has been studied for many decades, and we highlight several key approaches below: 
\begin{enumerate}
    \item System identification refers to the method of recovering $\A,\B,\C$ from the data.
This is a non-convex problem and while many methods have been considered in this setting, they depend polynomially on the hidden dimension.
\item The (auto) regression method predicts according to $\hat{\y}_t = \sum_{i=1}^h \M_i \uv_{t-i}$.
The coefficients $\M_i$ can be learned using convex regression. The downside of this approach is that if the spectral radius of $\A$ is $1-\delta$, it can be seen that $\sim \frac{1}{\delta}$ terms are needed.
\item The regression method can be further enhanced with ``closed loop" components, that regress on prior observations $\y_{t-1:1}$. It can be shown using the Cayley-Hamilton theorem that using this method, $d_h$ components are needed to learn the system, where $d_h$ is the {\it hidden} dimension of $\A$. 

    \item 
    Filtering involves recovering the state $x_t$ from observations. While Kalman filtering is optimal under specific noise conditions, it generally fails in the presence of marginal stability and adversarial noise.    

    \item Finally, spectral filtering combines the advantages of 
all methods above. It is an efficient method, its complexity does {\bf not} depend on the hidden dimension, and works for marginally stable systems. However, spectral filtering requires $\A$ to be symmetric, or diagonalizable under the real numbers.

\end{enumerate}
\section{Main Results}

In this section we formally state our main algorithms and theorems.
We show that the Universal Sequence Preconditioning method provides significantly improved regret bounds for learning linear dynamical systems than previously known when used in conjunction with two distinct methods.
The first method is simple convex regression, and the second is spectral filtering.
Both algorithms allow for learning in the case of marginally stable linear dynamical systems and allow for certain asymmetric transition matrices of arbitrary high hidden dimension.
The regret bounds are free of the hidden dimension (up to logarithmic factors) -- which significantly extends the state of the art.
\subsection{Universal Sequence Preconditioning Applied to Regression}

Algorithm~\ref{alg:preconditioned_convex_relaxation} is an instantiation of Algorithm~\ref{alg:seq-precond} for the method of convex regression.
We set the preconditioning coefficients to be the coefficients of the $n$-th degree (monic) Chebyshev polynomial.
\begin{algorithm}[ht]
		\caption{\label{alg:preconditioned_convex_relaxation} Universal Sequence Preconditioning for Regression}
		\begin{algorithmic}[1]
			\STATE Input: initial parameter $\Q^0$; preconditioning coefficients $\mathbf{c}_{0:n}$ from the $n$-th degree (monic) Chebyshev polynomial;
            convex constraints $\K = \left \{ \left( \Q_0, \dots, \Q_{n-1} \right) \textrm{ s.t. } \norm{\Q_j} \leq C_{\textrm{domain}} \norm{\mathbf{c}}_1 \right \}$
            \STATE Assert that $\mathbf{c}_0 = 1$.
            \FOR {$t = 1$ to $T$}
            \STATE Receive $\uv_t$.
            \STATE Predict $\hat{\y}_t(\Q^t) = - \sum_{i = 1}^n \mathbf{c}_i \y_{t-i} + \sum_{j = 0}^n \Q^t_j \uv_{t-j}$.
            \STATE Observe true output $\y_t$ and suffer loss $\ell_t(\Q^t) = \norm{\hat{\y}_t(\Q^t) - \y_t}_1$.
            \STATE Update and project: 
            \begin{equation*}
                \Q^{t+1} \gets \proj_{\K} \left( \Q^t - \eta_t \nabla_{\Q} \ell_t(\Q^t) \right).
            \end{equation*}
           \ENDFOR
           
		\end{algorithmic}
\end{algorithm}

Theorem~\ref{thm:convex_relaxation} shows that vanishing loss compared to the optimal ground-truth predictor, at a rate that is independent of the hidden dimension of the system.
\begin{theorem}
\label{thm:convex_relaxation}
    Let $\left \{ \uv_t \right \}_{t = 1}^T \in \mathbb{C}^{d_{\textrm{in}}}$ be any sequence of inputs which satisfy $\norm{\uv_t}_2 \leq 1$ and let $\left \{ \y_t \right \}_{t = 1}^T \in \mathbb{C}^{d_{\text{out}}}$ be the corresponding output coming from some linear dynamical system $(\A, \B, \C)$ as defined per Eq.~\ref{eqn:LDS}.
Let $\mat{P}$ diagonalize $\A$ (note $\mat{P}$ exists w.l.o.g.) and let $\kappa = \norm{\mat{P}} \norm{\mat{P}^{-1}}$.
Assume that $\norm{\B} \norm{\C} \kappa \leq C_{\textrm{domain}}$. Let $\lambda_1, \dots, \lambda_{\dhidden}$ denote the spectrum of $\A$.
If 
    \begin{equation*}
        \max_{j \in [\dhidden]} \abs{\arg(\lambda_j)} \leq 1/(32 \log_2(2 T^3/d_{\text{out}}))^2
    \end{equation*}
    then the predictions $\hat{\y}_1, \dots, \hat{\y}_T$ from Algorithm~\ref{alg:preconditioned_convex_relaxation} where the preconditioning coefficients $\mathbf{c}_{0:n}$ are chosen to be the coefficients of the $n$-th monic Chebyshev polynomial satisfy
    \begin{equation*}
         \frac{1}{T}\sum_{t=1}^T \norm{\hat{\y}_t - \y_t}_1 \leq \tilde{O} \left( \frac{\norm{\B}\norm{\C} \kappa \sqrt{d_{\text{out}}}}{T^{2/13}} \right),
    \end{equation*}
    where $\tilde{O}(\cdot)$ hides polylogarithmic factors in $T$.
\end{theorem}

The proof of Theorem~\ref{thm:convex_relaxation} is in Appendix~\ref{appendix:convex_relaxation}. For a simple baseline comparison, the regret achieved (via the same proof technique) by the vanilla regression algorithm without preconditioning is $O\left(C_{\textrm{domain}} \sqrt{d_{\text{out}}} T^{5/2} \right)$ which is not sublinear in $T$.
\subsection{Universal Sequence Preconditioning Applied to Spectral Filtering}
\label{sec:usp_sf}

Our second main result is the application of Universal Sequence Preconditioning to the spectral filtering algorithm \cite{hazan2017learning}.
Our results are more general and apply to any choice of polynomial, not just Chebyshev.
In addition to applying USP to spectral filtering, we also propose a novel spectral filtering basis.
Both changes to the vanilla spectral filtering algorithm are necessary to extend its sublinear regret bounds to the case of underlying systems with asymmetric hidden transition matrices.
First we define the spectral domain 
\begin{equation*}
    \complex_{\beta} = \left \{ z \in \complex \mid \abs{z} \leq 1, \abs{\arg(z)} \leq \beta \right \}.
\end{equation*}
Given horizon $T$ and $\alpha \in \complex_{\beta}$ let
 \begin{equation}
 \label{eqn:mu_alpha}
     \tilde{\mu}_T(\alpha) \defeq (1-\alpha^2) \begin{bmatrix}
         1 & \alpha & \dots & \alpha^{T-1}
     \end{bmatrix}^{\top},
 \end{equation}
 and 
 \begin{equation}
 \label{eqn:Z_def}
     \mat{Z}_T \defeq \int_{\alpha \in \complex_{\beta}} \tilde{\mu}_T(\alpha) \tilde{\mu}_T(\overline{\alpha})^{\top}  d \alpha,
 \end{equation} 
 where $\overline{\alpha} \in \complex$ denotes the complex conjugate.
The novel spectral filters are the eigenvectors of $\mat{Z}_{T-n-1}$, which we denote as $\phi_1, \dots, \phi_{T-n-1}$.
Note that in the standard spectral filtering literature, the spectral filtering matrix is an integral over the real line and does not involve the complex conjugate.
Our new matrix has an entirely different structure and although it looks quite similar, it surprisingly upends the proof techniques to ensure exponential spectral decay, a critical property for the method. Future work examines this matrix more thoroughly, but in this paper we simply provide a standard bound on its eigenvalues.
\begin{algorithm}[ht]
		\caption{\label{alg:new_sf} Universal Sequence Preconditioning for Spectral Filtering }
		\begin{algorithmic}[1]
			\STATE Input: initial $\Q^{1}_{1:n} , \M^1_{1:k}$, horizon $T$, convex constraints $$\K = \left \{ (\Q_0, \dots, \Q_{n-1}, \M_1, \dots, \M_k \textrm{ s.t. } \norm{\Q_j} \leq R_Q \textrm{ and } \norm{\M_j} \leq R_M \right \},$$ parameter $n$, coefficients $\mathbf{c}_{1:n}$.
			\STATE Let $p_n^{\mathbf{c}}(x) = \mathbf{c}_0 x^n + \mathbf{c}_1 x^{n-1} + \dots + \mathbf{c}_n$ and $\tilde{p}_n^{\mathbf{c}}(x) = (1-x^2) p_n^{\mathbf{c}}(x)$. Let $\tilde{\mathbf{c}}_0, \dots, \tilde{\mathbf{c}}_{n+2}$ be the coefficients of $\tilde{p}_n^{\mathbf{c}}(x)$.
            \STATE Let $\phi_1,...,\phi_n$ be the top $n$ eigenvectors of $\mat{Z}_{T-n-1}$.
            \STATE Assert $\tilde{\mathbf{c}}_0 = 1$. 
			\FOR {$t=1$ to $T$}
            \STATE Let $\tilde{\uv}_{t-n-1:1}$ be $\uv_{t-n-1:1}$ padded with zeros so it has dimension $T-n-1 \times d_{\textrm{in}}$.
			\STATE Predict $\hat{\y}_t(\Q^t, \M^t) = - \sum_{i = 1}^{n+2} \tilde{\mathbf{c}}_i \y_{t-i} + \sum_{j = 0}^n \Q^t_j \uv_{t-j} + \frac{1}{\sqrt{T}} \sum_{j = 1}^k \M_j^t \phi_j^{\top} \tilde{\uv}_{t-n-1:1} $.
            \STATE  Observe true $\y_t$, define loss $\ell_t(\hat{\y}_t ) = \| \hat{\y}_t(\Q^t, \M^t) - \y_t  \|_1$.
            \STATE Update and project:
			$ (\Q^{t+1},\M^{t+1}) = \proj_{\K}\left( \Q^t,\M^t) - \eta_{t} \nabla \ell_{t}(\Q^t,\M^{t}) \right) $
			\ENDFOR
		\end{algorithmic}
\end{algorithm}

\begin{theorem}
\label{thm:main_regret}
  Let $\left \{ \uv_t \right \}_{t = 1}^T \in \R^{d_{\textrm{in}}}$ be any sequence of inputs which satisfy $\norm{\uv_t}_2 \leq 1$ and let $\left \{ \y_t \right \}_{t = 1}^T$ be the corresponding output coming from some linear dynamical system $(\A, \B, \C)$ as defined per Eq.~\ref{eqn:LDS}.
Let $\mat{P}$ diagonalize $\A$ (note $\mat{P}$ exists w.l.o.g.) and let $\kappa = \norm{\mat{P}} \norm{\mat{P}^{-1}}$.
Suppose the radius parameters of Algorithm~\ref{alg:new_sf} satisfy $R_Q \geq \norm{\C} \norm{\B} \norm{\mathbf{c}}_1$ and $R_M \geq 2 \norm{\mat{C}} \norm{\mat{B}} \kappa \log (T) \left(  \max_{j \in [\dhidden]} \abs{\arg(\lambda_j)} \right)^{4/3} T^{7/6} \left( \max_{\alpha \in \complex_{\beta}} \abs{p_n^{\mathbf{c}}(\alpha)} \right) $.
Further suppose that the eigenvalues of $\mat{A}$ have bounded argument:
    \begin{equation*}
        \max_{j \in [\dhidden]} \abs{\arg(\lambda_j)} \leq T^{-1/3}.
    \end{equation*}
    Then the predictions $\hat{\y}_1, \dots, \hat{\y}_T$ from Algorithm~\ref{alg:new_sf} where the preconditioning coefficients $\mathbf{c}_{0:n}$ are chosen to be the coefficients of the $n$-th monic Chebyshev polynomial satisfy
    \begin{equation*}
        \frac{1}{T} \sum_{t=1}^T \norm{\hat{\y}_t - \y_t}_1 \leq \tilde{O} \left( \frac{\norm{\C} \norm{\B}  \kappa \sqrt{d_{\textrm{out}}}}{T^{1/39}} \right).
    \end{equation*}
\end{theorem}

We remark that the result of Theorem~\ref{thm:main_regret} is rather weak.
Although the complex eigenvalues of $\mat{A}$ are not trivially bounded (trivial would be a bound of $1/T$), they still must be polynomially small in $T$.
Moreover we note that the proof technique for the result does not make use of the critical properties of spectral filtering and relies much more on the power of preconditioning. The proof of Theorem~\ref{thm:main_regret} is in Section~\ref{sec:formal_spectral}.
\section{Proof Overview}
\label{sec:analysis_overview}
In this section we give a high level overview of the proofs for Theorem~\ref{thm:convex_relaxation} and Theorem~\ref{thm:main_regret}.
We start by recalling the intuition for Universal Sequence Preconditioning developed in Section~\ref{sec:intuition} which shows that if $\left \{ \y_t \right \}_{t=1}^T$ evolves as a linear dynamical system parameterized by matrices $(\A, \B, \C)$ with inputs $\left \{ \uv_t \right \}_{t=1}^T$ then by Equation~\ref{eqn:lds_breakdown},
\begin{equation*}
    \y_t = - \underbrace{\sum_{i=1}^n c_i \y_{t-i}}_{\aleph_0} + \underbrace{ \sum_{s = 0}^{n-1}   \sum_{i = 0}^s c_i \C \A^{s-i} \B  \uv_{t-s}}_{\aleph_1} + \underbrace{ \sum_{s = 0}^{t-n-1} \C p_n^{c}(\A) \A^s \B \uv_{t-n-s}}_{\aleph_2}.
\end{equation*}
Recall that $\aleph_0$ is the universal preconditioning component, $\aleph_1$ is the term that can easily be learned by convex relaxation and regression, and $\aleph_2$ is the critical term that contains $p_n^c(\A)$.
Both Theorem~\ref{thm:convex_relaxation} and Theorem~\ref{thm:main_regret} use the standard result from online convex optimization (Theorem 3.1 from \cite{hazan2016introduction}) that online gradient descent over convex domain $\mathcal{K}$ achieves regret $\frac{3}{2} GD \sqrt{T}$ as compared to the best point in $\mathcal{K}$, where $D$ denotes the diameter of $\mathcal{K}$ and $G$ denotes the maximum gradient norm.
\paragraph{Regression: Proof of Theorem~\ref{thm:convex_relaxation}} In the case of regression, the domain is chosen so that $\aleph_2$ may be learned and the proof proceeds by bounding the diameter of such a domain and its corresponding maximum gradient norm to get regret $C n^2 \sqrt{d_{\text{out}}} \norm{c}_1 \sqrt{T}$ for a universal constant $C>0$ which depends on the norms of matrices $\mat{B}$ and $\mat{C}$ from the underlying system.
Then $\aleph_3$ is treated as an un-learnable error term. Let $\lambda(\A)$ denote the set of eigenvalues of $\A$.
By the simple magnitude bound of
\begin{equation*}
    \norm{\sum_{s = 0}^{t-n-1} \C p_n(\A) \A^s \B \uv_{t-n-s}} \leq \max_{\lambda \in \lambda(\A)} \abs{p_n(\lambda)} \cdot T \cdot \norm{\C} \cdot \norm{\B} ,
\end{equation*}
the error of ignoring this term can be very small if $\max_{\lambda(\A)} \abs{p_n(\lambda(\A))} $ is small.
In the proof of Theorem~\ref{thm:convex_relaxation} in Appendix~\ref{appendix:convex_relaxation} we show that the regret for a generic polynomial $p_n^{c}$ defined by coefficients $c_{0:n}$ is
\begin{equation*}
     \sum_{t = 1}^T \norm{\y_t - \hat{\y}_t}_1  \leq \underbrace{C n^2 \sqrt{d_{\textrm{out}}} \norm{c}_1 \sqrt{T} }_{\text{Regret from learning $\aleph_2$}} + \underbrace{C \max_{\lambda \in \mathcal{D}} \abs{p_n^{c}(\lambda)} T^2}_{\text{Unlearnable Error Term}},
\end{equation*}
 where $\mathcal{D}$ is the region where $\A$ is allowed to have eigenvalues (see Theorem~\ref{thm:convex_relaxation_general}).
Therefore, to get sublinear regret, we must choose a polynomial which has bounded $\ell_1$ norm of its coefficients, while also exhibits very small infinity norm on the domain of $\A$'s eigenvalues.
\paragraph{Spectral Filtering: Proof of Theorem~\ref{thm:main_regret}}
 
 In the case of spectral filtering, the domain is chosen so that both $\aleph_2$ and $\aleph_3$ may be learned.
Because spectral filtering learns $\aleph_3$, it is able to accumulate less error and hence achieves a better regret bound of $O(T^{-3/13})$ as compared to regression's $O(T^{-2/13})$.
At a high level, the proof proceeds by exploiting the fact that $p_n(\A)$ shrinks the size of the learnable domain.
However this is not enough, in order to extend the result to systems where $\A$ may have complex eigenvalues, the spectral filters must be eigenvalues of a new matrix, defined in Eq.~\ref{eqn:Z_def}, whose domain of integration includes the possibly complex eigenvalues of $\A$.
To get the dimension-independent regret bounds enjoyed by spectral filtering in this new setting where complex eigenvalues may occur, the exponential decay of $\mat{Z}_T$ from Eq.~\ref{eqn:Z_def} must be established.
This is nontrivial and requires several novel techniques inspired by \cite{beckermann2017singular}. The details are in Appendix~\ref{sec:lipschitz_bound}.
Theorem~\ref{thm:main_regret} gives the main guarantee for the spectral filtering algorithm, which states that Algorithm~\ref{alg:new_sf} instantiated with some choice of polynomial $p_n^{c}(\cdot)$ achieves regret
\begin{equation*}
\tilde{O} \left( \left( n \norm{c}_1 +  T^{7/6} \max_{\alpha \in \complex_{\beta}} \abs{p_n^{c}(\alpha)} \right) (n+k) \sqrt{d_{\textrm{out}}} \sqrt{T} \right).
\end{equation*}

Both this theorem, as well as our new guarantee for convex regression, leads us to the following question: \textbf{Is there a universal choice of polynomial $p_n(x)$, where $n$ is independent of hidden dimension, which guarantees sublinear regret?
} \\ 

\subsection{Using the Chebyshev Polynomial over the Complex Plane}

For the real line, the answer to this question is known to be positive using the Chebyshev polynomials of the first kind.
In general, the $n^{\textrm{th}}$ (monic) Chebyshev polynomial $M_n(x)$ satisfies $\max_{x \in [-1,1]} \abs{M_n(x)} \leq 2^{-(n-1)} $.
However, we are interested in a more general question over the complex plane.
Since we care about linear dynamical systems that evolve according to a general asymetric matrix,  we need to extending our analysis to $\complex_{\beta}$.
This is a nontrivial extension since, in general, functions that are bounded on the real line can grow exponentially on the complex plane.
Indeed, $2^{n-1} M_n(x) = \cos ( n \arccos(x) )$ and while $\cos(x)$ is bounded within $[-1,1]$ for any $x \in \R$, over the complex numbers we have 
$\cos(z) = \frac{1}{2}(e^{iz} + e^{-iz}) , $ 
which is unbounded.
Thus, we analyze the Chebyshev polynomial on the complex plane and provide the following bound.
\begin{lemma}
\label{lemma:cheby_bound}
    Let $z \in \complex$ be some complex number with magnitude $\abs{\alpha} \leq 1$.
Let $M_n(\cdot)$ denote the $n$-th monic Chebyshev polynomial. If $\abs{\arg(z)} \leq 1/64n^2$, then $\abs{M_n(z)} \leq 1/2^{n-2}$.
\end{lemma}
We provide the proof in Appendix~\ref{appendix:chebyshev}. We also must analyze the magnitude of the coefficients of the Chebyshev polynomial, which can grow exponentially with $n$.
We provide the following result. 
\begin{lemma}
\label{lemma:cheby_coeffs_bound}
    Let $M_n(\cdot)$ have coefficients $c_0, \dots, c_n$.
Then $\max_{k = 0, \dots, n} \abs{c_k} \leq 2^{0.3n}$. 
\end{lemma}
The proof of Lemma~\ref{lemma:cheby_coeffs_bound} is in Appendix~\ref{appendix:chebyshev}.
Together, these two lemmas are the fundamental building block for universal sequence  preconditioning and for obtaining our new regret bounds.
\section{Experimental Evaluation}
\label{sec:experiments}
We empirically validate that convolutional preconditioning with Chebyshev or Legendre coefficients yields significant online regret improvements across various learning algorithms and data types.
Below we summarize our data generation, algorithm variants, hyperparameter tuning, and evaluation metrics.
\subsection{Synthetic Data Generation}
\label{sec:data}
We generate \(N=200\) sequences of length \(T=2000\) via three mechanisms: (i) a noisy linear dynamical system, (ii) a noisy nonlinear dynamical system, and (iii) a noisy deep RNN.
Inputs \(\uv_{1:T}\sim\mathcal{N}(0,I)\).

\paragraph{Linear Dynamical System.}
Sample \((\A,\B,\C)\) with \(\A\in\R^{300\times300}\) having eigenvalues \(\{z_j\}\) drawn uniformly in the complex plane subject to \(\Im(z_j)\le\tau_{\rm thresh}\) and \(L\le|z_j|\le U\), and \(\B,\C\in\R^{300}\).
Then
\[
\x_t = \A\,\x_{t-1} + \B\,\uv_t,\quad
\y_t = \C\,\x_t + \epsilon_t,\;\epsilon_t\sim\mathcal{N}(0,\sigma^2I).
\]

\paragraph{Nonlinear Dynamical System.}
Similarly sample \((\A_1,\B_1,\C)\) and \((\A_2,\B_2)\) with \(\A_i\in\R^{10\times10}\), \(\B_i,\C\in\R^{10}\).
Then
\[
\x_t^{(0)}=\A_1 \x_{t-1}+\B_1 \uv_t,\;
\x_t^{(1)}=\sigma\bigl(\x_t^{(0)}\bigr),\;
\x_t=A_2 \x_t^{(1)}+B_2 \uv_t,\;
\y_t=\C \x_t+\epsilon_t.
\]

\paragraph{Deep RNN.}
We randomly initialize a sparse 10-layer stack of LSTMs with hidden dimension $100$ and ReLU nonlinear activations.
Given $\uv_{1:T}$ we use this network to generate $\y_{1:T}$.

\subsection{Algorithms and Preconditioning Variants}
\label{sec:algos}
We evaluate the following methods: (1)  \textbf{Regression} (Alg.~\ref{alg:preconditioned_convex_relaxation}) , (2)  \textbf{Spectral Filtering} (Alg.~\ref{alg:new_sf}) , (3)  \textbf{DNN Predictor}: \(n\)-layer LSTM with dims \([d_1,\dots,d_n]\), ReLU.
Each method is applied with one of:
\begin{enumerate}
  \item \emph{Baseline:} no preconditioning
  \item \emph{Chebyshev:} $\mat{c}_{0:n}$ are the coefficients for the \(n\)th-Chebyshev polynomial.
Note that when $n=2$ we have $\mat{c}_0 = 1$ and $\mat{c}_1 = -1$ and therefore this is the method of \emph{differencing} discussed in the introduction.
\item \emph{Legendre:} $\mat{c}_{0:n}$ are the coefficients for the \(n\)th-Legendre polynomial
  \item \emph{Learned:} $\mat{c}_{0:n}$ is a parameter learned jointly with the model parameters
\end{enumerate}
We test polynomial degrees \(n\in\{2,5,10,20\}\).
This choice of degrees shows a rough picture of the impact of $n$.
\paragraph{Hyperparameter Tuning.}
To ensure fair comparison, for each algorithm and conditioning $\mat{c}$ variant we perform a grid search over learning rates \(\eta\in\{10^{-3},10^{-2},10^{-1}\}\), selecting the one minimizing average regret across the \(N\) sequences.
In the case of the learned coefficients, we sweep over the 9 pairs of learning rates \((\eta_{\textrm{model}}, \eta_{\textrm{coefficients}}) \in \{10^{-3},10^{-2},10^{-1}\} \times \{10^{-3},10^{-2},10^{-1}\} \).
\subsection{Results}
\label{sec:results}
Tables~\ref{tab:full-results-linear}--\ref{tab:full-results-dnn} report the mean \(\pm\) std of the absolute error over the final 200 predictions, averaged across 200 runs.
In the linear and nonlinear cases we train a 2-layer DNN (dims \((64,128)\));
for RNN-generated data we match the 10-layer (100-dim) generator.

\emph{Key observations:}
\begin{itemize}
  \item Preconditioning drastically reduces baseline errors for all algorithms and data types.
\item Chebyshev and Legendre yield nearly identical gains.
  \item For Chebyshev and Legendre, once the degree is higher than $5-10$ the performance degrades since $\norm{\mat{c}}_1$ gets very large (see our Lemma~\ref{lemma:cheby_coeffs_bound} which shows that these coefficients grow exponentially fast).
\item Improvements decay as the complex threshold \(\tau_{\rm thresh}\) increases, consistent with our theoretical results which must bound \(\Im(z_j)\).
\item Learned coefficients excel with regression and spectral filtering but destabilize the DNN on nonlinear and RNN‐generated data.
\end{itemize}

\begin{table}[ht]
\centering
\renewcommand{\arraystretch}{1.63}
\setlength{\tabcolsep}{3pt}

\begin{adjustbox}{max width=\textwidth}
\begin{tabular}{>{\raggedright\arraybackslash}p{2.5cm}
|c
|ccc
|ccc
|cccc}
\toprule
\textbf{Setting} &
\textbf{Baseline} &
& \textbf{Chebyshev} &  &
& \textbf{Legendre} & &
& \textbf{Learned} & &  \\
 &
&
\textbf{Deg. 2} & \textbf{Deg.
5} & \textbf{Deg. 10} &
\textbf{Deg. 2} & \textbf{Deg. 5} & \textbf{Deg. 10} &
\textbf{Deg. 2} & \textbf{Deg. 5} & \textbf{Deg.
10} & \textbf{Deg. 20} \\
\hline
Regression  & 
 & 
 & 
 & 
 & 
& 
& 
& 
 & 
& 
& 
\\
$\tau_{\rm thresh}= 0.01$ & 
\textbf{0.74 \small{$\pm$ 0.28} } & 
0.25 \small{$\pm$ 0.09 } & 
\textbf{0.15 \small{$\pm$ 0.07} } & 
0.77 \small{$\pm$ 0.31 } & 
0.36 \small{$\pm$ 0.13 } & 
\textbf{0.14 \small{$\pm$ 0.06} } & 
0.64 \small{$\pm$ 0.26 } & 
0.52 \small{$\pm$ 0.19}  & 
0.27 \small{$\pm$ 0.11 } & 
\textbf{0.17 \small{$\pm$ 0.07} } & 
0.24 \small{$\pm$ 0.09 } \\
$\tau_{\rm thresh}= 0.1$ & 
\textbf{1.92 \small{$\pm$ 0.81}}  & 
0.84 \small{$\pm$ 0.27 } & 
\textbf{0.66 \small{$\pm$ 0.18}}  & 
1.90 
\small{$\pm$ 0.67}  & 
1.10 \small{$\pm$ 0.40}  & 
\textbf{0.63 \small{$\pm$ 0.17} } & 
1.66 \small{$\pm$ 0.58}  & 
1.34 \small{$\pm$ 0.43 } & 
0.57 \small{$\pm$ 0.14 } & 
\textbf{0.55 \small{$\pm$ 0.14}}  & 
0.56 \small{$\pm$ 0.14} s \\
$\tau_{\rm thresh}= 0.9$ & 
\textbf{2.47 \small{$\pm$ 0.89}}  & 
\textbf{1.59 \small{$\pm$ 0.56}}  & 
2.18 \small{$\pm$ 0.79 } & 
2.68 \small{$\pm$ 0.48 } & 
\textbf{1.64 \small{$\pm$ 0.58}}  & 
1.94 \small{$\pm$ 0.70 } & 
2.63 \small{$\pm$ 0.45}  & 
1.73 \small{$\pm$ 0.59 } & 
0.83 \small{$\pm$ 0.25}  & 
0.68 \small{$\pm$ 0.27 } & 
\textbf{0.63 \small{$\pm$ 0.26}}  \\
\hline
Spectral Filtering &  
& 
& 
 & 
& 
& 
& 
& 
& 
& 
& 
\\
$\tau_{\rm thresh}= 0.01$ & 
\textbf{5.94 \small{$\pm$ 3.37} } & 
1.72 \small{$\pm$ 0.95 } & 
\textbf{0.69 \small{$\pm$ 0.38} } & 
3.25 \small{$\pm$ 1.79  }& 
2.78 \small{$\pm$ 1.56 } & 
\textbf{0.66\small{$\pm$ 0.36}  }& 
2.74 \small{$\pm$ 1.51  }& 
1.99 \small{$\pm$ 0.94  }& 
\textbf{0.54\small{$\pm$ 0.29}  }& 
0.55 \small{$\pm$ 0.25  }& 
0.61 \small{$\pm$ 0.25  }\\
$\tau_{\rm thresh}= 0.1$ & 
\textbf{0.89 \small{$\pm$ 0.34} } & 
0.42 \small{$\pm$ 0.11  }& 
\textbf{0.34 \small{$\pm$ 0.07}  }& 
0.86 \small{$\pm$ 0.28  }& 
0.54 \small{$\pm$ 0.17  }& 
\textbf{0.33 \small{$\pm$ 0.07}  
}& 
0.76 \small{$\pm$ 0.24  }& 
0.69 \small{$\pm$ 0.28  }& 
\textbf{0.31 \small{$\pm$ 0.06}  }& 
0.37 \small{$\pm$ 0.08  }& 
0.45 \small{$\pm$ 0.28  }\\
$\tau_{\rm thresh}= 0.9$  & 
\textbf{10.17 \small{$\pm$ 8.80} } & 
\textbf{9.87 \small{$\pm$ 8.90} } & 
12.66 \small{$\pm$ 8.18 } & 
32.90 \small{$\pm$ 22.02  }& 
\textbf{9.42 \small{$\pm$ 8.39} } & 
11.53 \small{$\pm$ 7.46  }& 
28.83 \small{$\pm$ 19.24 } & 
7.93 \small{$\pm$ 4.42 } & 
6.31 \small{$\pm$ 4.19 } & 
\textbf{5.73 \small{$\pm$ 3.80} } & 
5.86 \small{$\pm$ 4.02 } \\
\hline
2-layer DNN & 
& 
& 
& 
& 
& 
& 
& 
& 
& 
& 
\\
$\tau_{\rm thresh}= 0.01$ 
 & 
\textbf{4.49 \small{$\pm$ 2.02}  }& 
\textbf{2.31 \small{$\pm$ 1.18}  }& 
2.62 \small{$\pm$ 1.52  }& 
10.36\small{$\pm$ 6.05  }& 
2.79 \small{$\pm$ 1.34 } & 
\textbf{2.35 \small{$\pm$ 1.35}  }& 
8.92 \small{$\pm$ 5.20 } & 
2.89 \small{$\pm$ 1.24  }& 
1.56 \small{$\pm$ 0.64  }& 
0.79 \small{$\pm$ 0.25  }& 
\textbf{0.40 \small{$\pm$ 0.16}  }\\
$\tau_{\rm thresh}= 0.1$  & 
\textbf{9.41 \small{$\pm$ 7.34}  }& 
2.66 \small{$\pm$ 1.76  }& 
\textbf{1.52 \small{$\pm$ 0.72} } & 
4.65 \small{$\pm$ 3.03  }& 
4.24\small{$\pm$ 3.06  }& 
\textbf{1.44 \small{$\pm$ 0.71}  }& 
4.03 \small{$\pm$ 2.64  }& 
6.54 \small{$\pm$ 4.32  }& 
3.22 
\small{$\pm$ 1.93  }& 
1.59 \small{$\pm$ 0.95  }& 
\textbf{0.80\small{$\pm$ 0.48}  }\\
$\tau_{\rm thresh}= 0.9$  & 
\textbf{2.45 \small{$\pm$ 1.31}  }& 
\textbf{2.24 \small{$\pm$ 1.34} } & 
3.49 \small{$\pm$ 2.27 } & 
11.48 \small{$\pm$ 8.15  }& 
\textbf{2.17 \small{$\pm$ 1.25} } & 
3.10 \small{$\pm$ 2.00  }& 
9.97 \small{$\pm$ 7.05  }& 
1.29 \small{$\pm$ 0.66  }& 
0.71 \small{$\pm$ 0.34  }& 
0.43 \small{$\pm$ 0.18  }& 
\textbf{0.24 \small{$\pm$ 0.11}  }\\
\bottomrule
\end{tabular}
\end{adjustbox}
\caption{Linear dynamical system data (detailed in Sec.~\ref{sec:data}) across varying complex threshold $\tau_{\rm thres}$.}
\label{tab:full-results-linear}
\end{table}

\begin{table}[ht]
\centering
\renewcommand{\arraystretch}{1.3}
\setlength{\tabcolsep}{4pt}

\begin{adjustbox}{max width=\textwidth}
\begin{tabular}{>{\raggedright\arraybackslash}p{2.5cm}
|c
|ccc
|ccc
|cccc}
\toprule
\textbf{Setting} &
\textbf{Baseline} &
& \textbf{Chebyshev} &  &
& \textbf{Legendre} & &
& \textbf{Learned} & &  \\
 &
&
\textbf{Deg.
2} & \textbf{Deg. 5} & \textbf{Deg. 10} &
\textbf{Deg. 2} & \textbf{Deg. 5} & \textbf{Deg. 10} &
\textbf{Deg. 2} & \textbf{Deg.
5} & \textbf{Deg. 10} & \textbf{Deg. 20} \\
\hline
Spectral Filtering & 
 & 
 & 
& 
 & 
& 
 & 
& 
& 
& 
 & 
\\
 $\tau_{\rm thresh}= 0.01$& 
\textbf{153.8 \small{$\pm$ 15.7} } & 
43.7\small{$\pm$ 19.4 }& 
0.92\small{$\pm$ 0.31 }& 
\textbf{0.26\small{$\pm$ 0.15} }& 
78.4\small{$\pm$ 34.7 }& 
2.82\small{$\pm$ 1.17 }& 
\textbf{0.34\small{$\pm$ 0.19} }& 
1.04\small{$\pm$ 0.29} & 
0.07\small{$\pm$ 0.03 }& 
\textbf{0.05\small{$\pm$ 0.03} }& 
0.07\small{$\pm$ 0.03 }\\
  $\tau_{\rm thresh}= 0.01$ & 
\textbf{124.1\small{$\pm$ 68.4}} & 
33.4\small{$\pm$ 17.3 }& 
2.02\small{$\pm$ 1.23 }& 
\textbf{0.36\small{$\pm$ 0.31}} & 
56.7\small{$\pm$ 30.5 }& 
3.81\small{$\pm$ 1.67 }& 
\textbf{0.35\small{$\pm$ 0.26} }& 
2.85\small{$\pm$ 1.13 }& 
0.04\small{$\pm$ 0.01 }& 
\textbf{0.02\small{$\pm$ 0.01} }& 
0.07\small{$\pm$ 
0.02} \\
 $\tau_{\rm thresh}= 0.9$ & 
\textbf{165.5\small{$\pm$ 84.5}} & 
43.41\small{$\pm$ 16.39 }& 
\textbf{1.27\small{$\pm$ 0.37} }& 
1.90\small{$\pm$ 0.80 }& 
76.9\small{$\pm$ 29.1 }& 
3.32\small{$\pm$ 1.09} & 
\textbf{1.64\small{$\pm$ 0.69}} & 
1.79\small{$\pm$ 0.61 }& 
\textbf{0.10\small{$\pm$ 0.04} }& 
0.12\small{$\pm$ 0.05} & 
0.17\small{$\pm$ 0.07 }\\
\hline
2-layer DNN & 
& 
 & 
& 
& 
 & 
& 
 & 
 & 
& 
 & 
 \\
 $\tau_{\rm thresh}= 0.01$ & 
\textbf{10.46\small{$\pm$ 5.10} }& 
3.45\small{$\pm$ 1.19 }& 
\textbf{0.19\small{$\pm$ 0.16}} & 
0.43\small{$\pm$ 0.22} & 
3.55\small{$\pm$ 1.56} & 
\textbf{0.18\small{$\pm$ 0.14} }& 
0.37\small{$\pm$ 0.19 }& 
45.40\small{$\pm$ 11.35 }& 
25.34\small{$\pm$ 9.06 }& 
12.73\small{$\pm$ 4.60} & 
\textbf{ 6.43\small{$\pm$ 2.32} } \\
 $\tau_{\rm thresh}= 0.1$ 
& 
\textbf{4.20\small{$\pm$ 1.10} }& 
3.38\small{$\pm$ 0.79 }& 
\textbf{0.14\small{$\pm$ 0.05} }& 
0.41\small{$\pm$ 0.16 }& 
3.42\small{$\pm$ 0.79 }& 
\textbf{0.25\small{$\pm$ 0.14}}& 
0.36\small{$\pm$ 0.13 }& 
67.70\small{$\pm$ 28.40} & 
31.99\small{$\pm$ 11.78} & 
15.95\small{$\pm$ 5.87 }& 
\textbf{8.01\small{$\pm$ 2.95}} \\
 $\tau_{\rm thresh}= 0.9$ & 
\textbf{6.72\small{$\pm$ 2.35} }& 
2.45\small{$\pm$ 1.12 }& 
\textbf{0.08\small{$\pm$ 0.03} }& 
0.28\small{$\pm$ 0.15} & 
3.35\small{$\pm$ 0.98 }& 
\textbf{ 0.20\small{$\pm$ 0.07 }}& 
0.24\small{$\pm$ 0.13} & 
58.65\small{$\pm$ 17.09 }& 
28.74\small{$\pm$ 8.07 }& 
14.39\small{$\pm$ 4.06 }& 
\textbf{ 7.26\small{$\pm$ 2.05} }\\
\bottomrule
\end{tabular}
\end{adjustbox}
\caption{Nonlinear data (detailed in Sec.~\ref{sec:data}) across varying complex threshold $\tau_{\rm thres}$.}
\label{tab:full-results-nonlinear}
\end{table}

\begin{table}[ht!]
\centering
\renewcommand{\arraystretch}{1.3}
\setlength{\tabcolsep}{4pt}

\begin{adjustbox}{max width=\textwidth}
\begin{tabular}{>{\raggedright\arraybackslash}p{2.5cm}
|c
|ccc
|ccc
|cccc}
\toprule
\textbf{Setting} &
\textbf{Baseline} &
& \textbf{Chebyshev} &  &
& \textbf{Legendre} & &
& \textbf{Learned} & &  \\
 &
&
\textbf{Deg.
2} & \textbf{Deg. 5} & \textbf{Deg. 10} &
\textbf{Deg. 2} & \textbf{Deg. 5} & \textbf{Deg. 10} &
\textbf{Deg. 2} & \textbf{Deg.
5} & \textbf{Deg. 10} & \textbf{Deg. 20} \\
\hline
 10-layer DNN & 
\textbf{0.54 \small{$\pm$ 0.23} }& 
0.29 \small{$\pm$ 0.10 }& 
\textbf{0.08\small{$\pm$ 0.03} }& 
0.13 \small{$\pm$ 0.05} & 
 0.37 \small{$\pm$ 0.14 }& 
\textbf{ 0.09\small{$\pm$ 0.03} } & 
0.12\small{$\pm$ 0.04 }& 
1.49\small{$\pm$ 0.93}& 
2.13\small{$\pm$ 1.05 }& 
1.04\small{$\pm$ 0.51}  &
\textbf{ \textbf{ 0.5 \small{$\pm$ 0.24} }} \\
\bottomrule
\end{tabular}
\end{adjustbox}
\caption{Performance (average absolute error of the last 200 predictions) of a 10-layer DNN (detailed in Sec.~\ref{sec:algos}) on data generated from the same model (detailed in Sec.~\ref{sec:data}).
\vspace{-5mm}}
\label{tab:full-results-dnn}
\end{table}

\subsection{ETTh1 Dataset}
\label{sec:etth1}
To evaluate whether our proposed preconditioning approach generalizes to real-world time series, we conduct experiments on the well-established \textbf{ETTh1} dataset from the Electricity Transformer Temperature (ETT) 
benchmark~\cite{zhou2021informer}. The ETTh1 dataset consists of continuous hourly measurements of load and oil temperature collected from electricity transformers and has been used in several recent works~\cite{zhou2021informer, wu2021autoformer, nie2022time, gu2022efficiently, gupta2022simplifying, zeng2023transformer, nguyen2024preconditioning}.
We study the effect of preconditioning on a $10$-layer LSTM with hidden dimension $100$ per layer using the Adam optimizer.
We set the horizon to be $T=5000$ and we sweep over a broader range of learning rates $\eta \in \{10^{-j} \}_{j = 0, 1, 2, 3, 4, 5}$.
As before we consider (i) no preconditioning (baseline), (ii) fixed Chebyshev coefficients, (iii) fixed Legendre coefficients, and (iv) coefficients learned jointly with model parameters.
As seen in Figure~\ref{fig:etth1}, preconditioning with Chebyshev and Legendre for degree $5$ the best performance after only the first $1000$ iterations, while the performance of jointly learning the coefficients is worse at this stage.
The performance of all three preconditioning methods are roughly on par with each other by $2500$ iterations and by the full horizon $T = 5000$, jointly learning the coefficients results in the best average prediction error.
\begin{figure}[H]
    \centering
    \begin{subfigure}[b]{0.32\textwidth}
        \includegraphics[width=\linewidth]{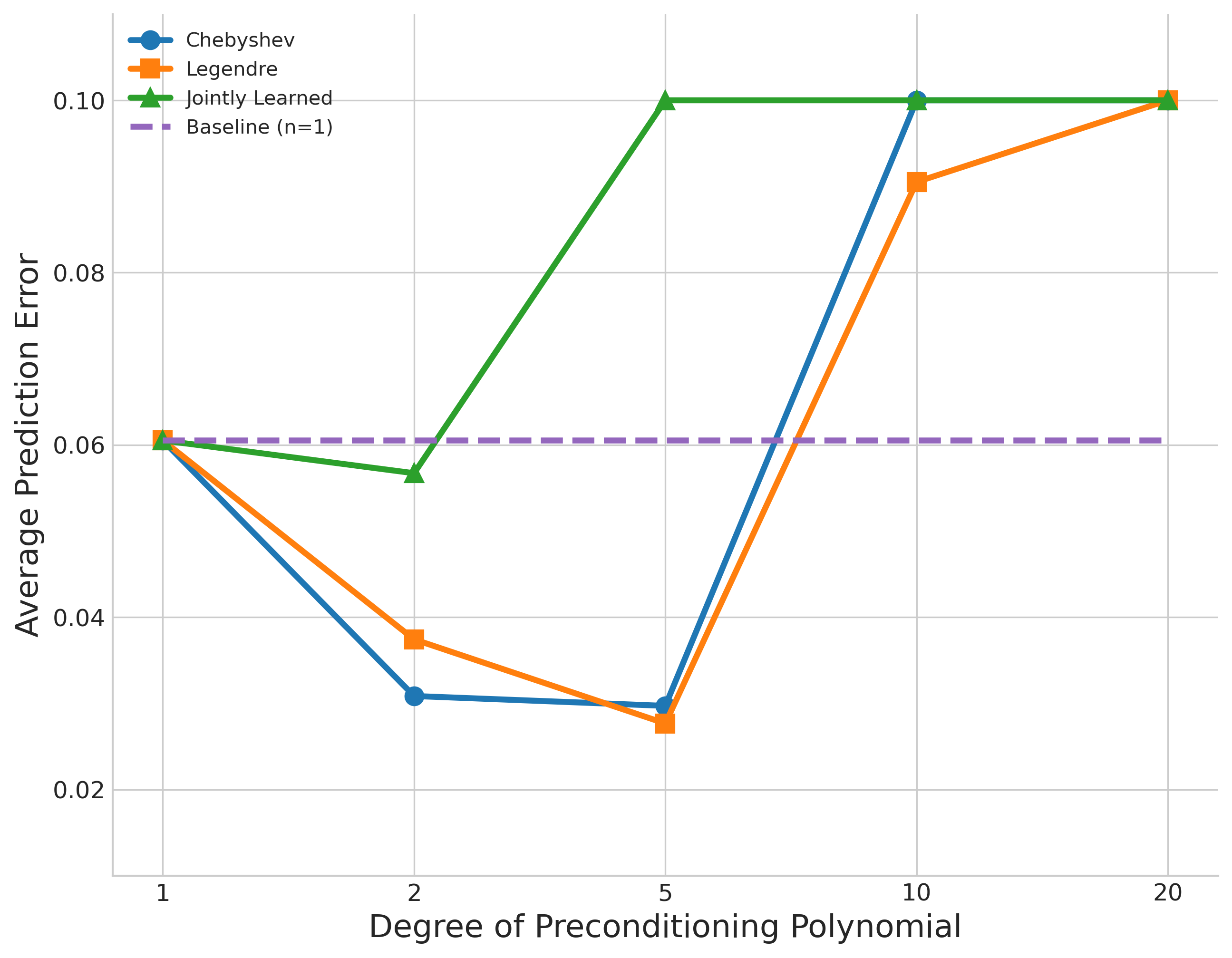}
        \caption{Early Stage: $T=1000$ }
    \end{subfigure}
    \hfill
    \begin{subfigure}[b]{0.32\textwidth}
        \includegraphics[width=\linewidth]{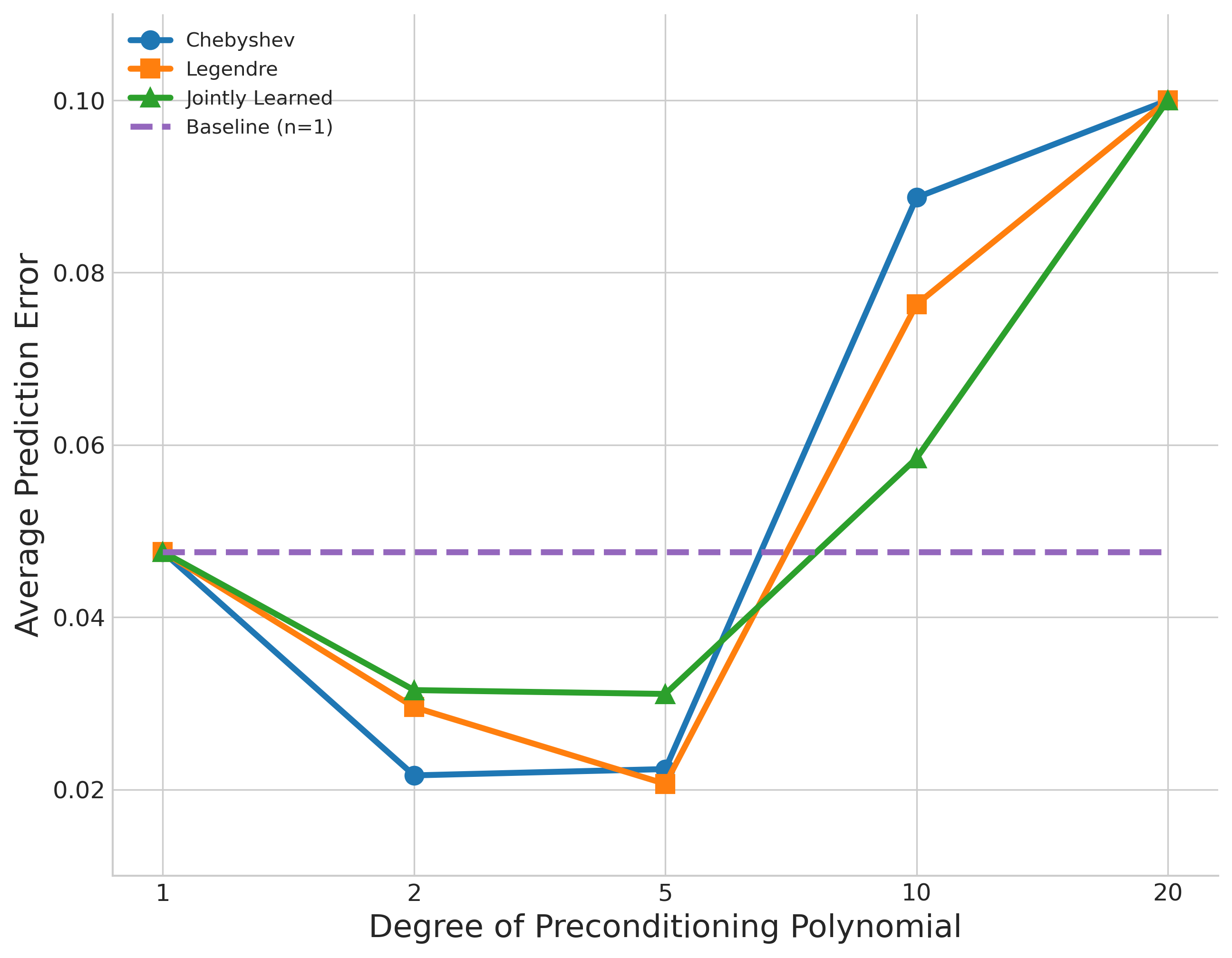}
        \caption{Middle Stage: $T=2500$ }
    \end{subfigure}
    \hfill
    \begin{subfigure}[b]{0.32\textwidth}
        \includegraphics[width=\linewidth]{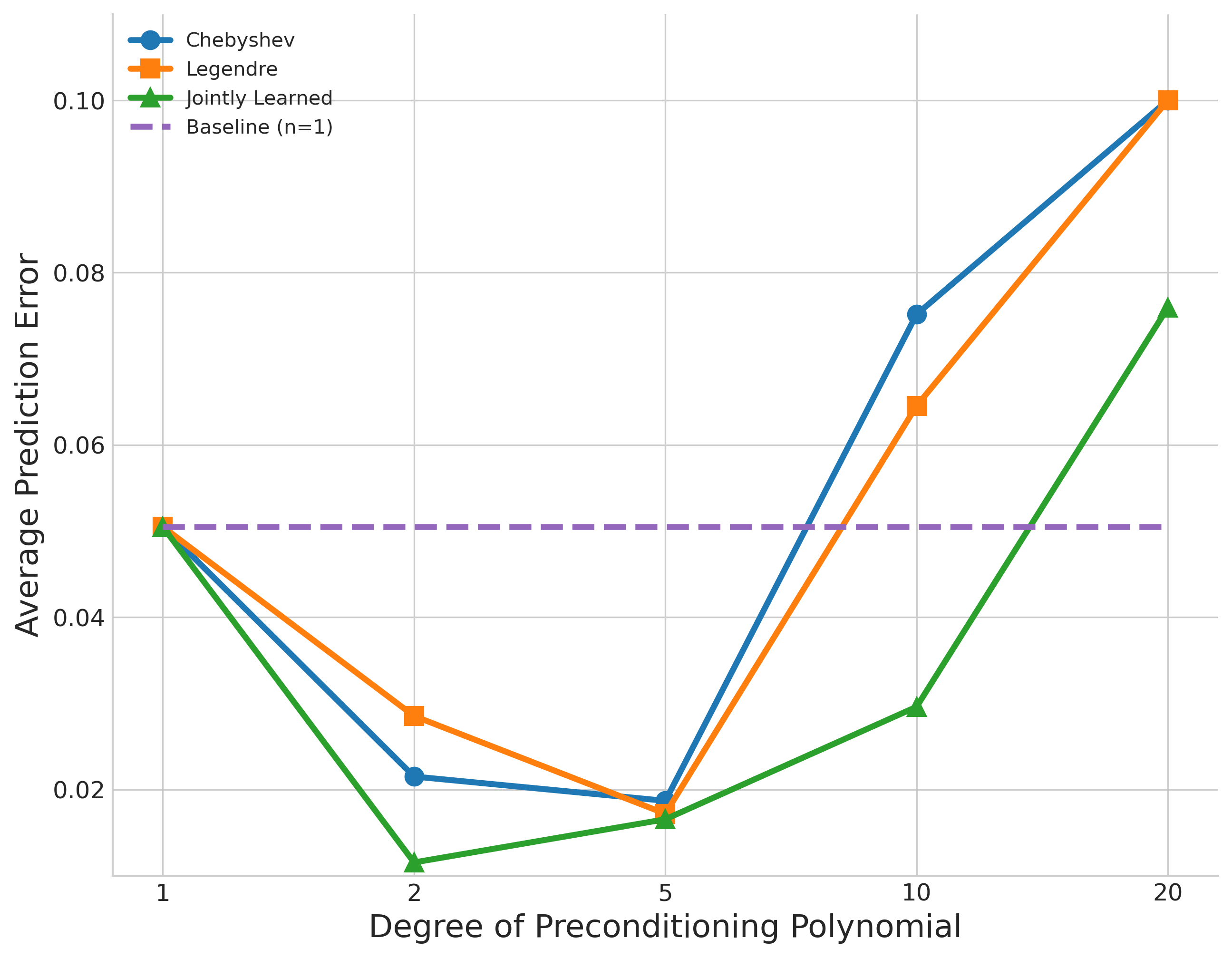}
        \caption{Final Stage: $T=5000$}
    \end{subfigure}

    \vspace{0.5cm}
    
\caption{Absolute prediction error on final $200$ predictions averaged over 10 independent runs for 10-layer LSTM with layer dimension $100$ using Adam optimizer and sweeping over learning rates for each run.}
    \label{fig:etth1}
\end{figure}
\section{Discussion}
There are many settings in machine learning where universal, rather than learned, rules have proven very efficient.  For example, physical laws of motion can be learned directly from observation data. However, Newton's laws of motion succinctly crystallize very general phenomenon, and have proven very useful for large scale physics simulation engines. Similarly, in the theory of mathematical optimization, adaptive gradient methods have revolutionized deep learning. Their derivation as a consequence of regularization in  online regret minimization is particularly simple \cite{duchi2011adaptive}, and thousands of research papers have not dramatically improved the initial basic ideas. These optimizers are, at the very least, a great way to initialize learned optimizers \cite{wichrowska2017learned}.

By analogy, our thesis in this paper is that universal preconditioning based on the solid theory of dynamical systems can be applicable to many domains or,  at the very least, an initialization for other learning methods.

\section{Acknowledgements}
Elad Hazan gratefully acknowledges funding from the Office of
Naval Research and Open Philanthropy. 
\newpage 
\bibliographystyle{plain}
\bibliography{main}
\newpage
\newpage
\tableofcontents
\newpage
\appendix
\section{Related work} \label{sec:related_work}
\paragraph{Preconditioning for sequence prediction.} 
Preconditioning in the context of time series analysis has roots in the classical work of Box and Jenkins \cite{box1976time}. In this foundational text they propose differencing as a method for making the time series stationary, and thus amenable to statistical learning techniques such as ARMA (auto-regressive moving average) \cite{anava2013online}. The differencing operator can be applied numerous times, and for different lags, giving rise to the ARIMA family of forecasting models.   

Identifying the order of an ARIMA model, and in particular the types of differencing needed to make a series stationary, is a hard problem. This is a special case of the problem we consider: differencing corresponds to certain coefficients of preconditioning the time series, whereas we consider arbitrary coefficients. 

\paragraph{Background on control of linear dynamical systems.} 
Linear dynamical systems are the most fundamental and basic model in control theory, and have been studied for more than a century. For a thorough introduction, see the texts \cite{bertsekas2007dynamic,kemin,hazan2022introduction}. 

A rigorous proof that the Cayley-Hamilton theorem implies that $d_h$ learned closed-loop components are sufficient to learn any LDS is given in  \cite{agarwal2023spectral,hazan2018spectral}. 

The seminal work of Kalman on state-space representation and filtering \cite{kalman1960new} allows one to learn any LDS with hidden-dimension parameters under stochastic and generative assumptions. Closed-loop auto-regressive learning subsumes Kalman filtering in the presence of adversarial noise, see e.g. \cite{kozdoba2019line}. \cite{ghai2020no} provide a method to learn a marginally stable LDS in the presence of bounded adversarial noise and asymmetric matrices, however their regret bound depends on the hidden dimension of the system. More recently, \cite{bakshi2023new} use tensor and moment-based methods to learn a LDS with stochastic noise without dependence on the system's condition number. However, their algorithmic complexity still scales polynomially with the hidden dimension. 

In this work we consider regret in the context of {\it learning} linear dynamical systems. This is related to, but different from, {\it control} of the systems. We survey regret minimization for control next. 

\paragraph{Regret for classical control models. }
The first works addressing control in the machine learning community assume either no perturbation in the dynamics at all, or i.i.d. Gaussian perturbations. 
Much of this work has considered obtaining low regret in the online LQR setting \citep{abbasi2011regret,dean2018regret,mania2019certainty,cohen2019learning} where a fully-observed linear dynamic system is driven by i.i.d. Gaussian noise via $x_{t+1} = \A x_t + \B \uv_t + w_t$, and the learner incurs constant quadratic state and input cost $\ell(x,u) = \frac{1}{2}x^\top \Q x + \frac{1}{2}u^\top R u$. The optimal policy for this setting is well-approximated by a \emph{state feedback controller} $\uv_t = K x_t$, where $K$ is the solution to the Discrete Algebraic Riccati Equation (DARE), and thus regret amounts to competing with this controller. Recent algorithms \cite{mania2019certainty,cohen2019learning} attain $\sqrt{T}$ regret for this setting, with polynomial runtime and polynomial regret dependence on relevant problem parameters. 
A parallel line of work by  \cite{cohen2018online} establishes $\sqrt{T}$ regret in a variant of online LQR where the system is known to the learner, noise is stochastic, but an adversary selects quadratic loss functions $\ell_t$ at each time $t$. Again, the regret is measured with respect to a best-in-hindsight state feedback controller.

Provable control in the Gaussian noise setting via the policy gradient method was studied in \cite{fazel2018global}. Other relevant work from the machine learning literature includes tracking adversarial targets \citep{abbasi2014tracking}.

\paragraph{The non-stochastic control problem.}
The most accepted and influential control model stemming from the machine learning community was established in 
\cite{agarwal2019online}, who obtain $\sqrt{T}$-regret in the more general and challenging setting where both the Lipschitz loss function and the perturbations are adversarially chosen. The key insight behind this result is combining an improper controller parametrization known as disturbance-based control with online convex optimization with memory due to \cite{anava2015online}. Follow-up work by \cite{agarwal2019logarithmic} achieves logarithmic pseudo-regret for strongly convex, adversarially selected losses and well-conditioned stochastic noise. Further extensions were made for linear control with partial observation \cite{simchowitz2020improper}, system identification with adversarial noise \cite{chen2021black}, and many more settings surveyed in \cite{hazan2022introduction}. 

\paragraph{Online learning and online convex optimization.}
We make extensive use of techniques from the field of online learning and regret minimization in games \citep{cesa2006prediction, hazan2016introduction}. Following previous work from machine learning, we consider regret minimization in sequence prediction, where the underlying sequence follows a linear dynamical system.   

\paragraph{Spectral filtering for learning linear dynamical systems.}
The spectral filtering technique was put forth in \cite{hazan2017learning} for learning symmetric linear dynamical systems. In \cite{hazan2018spectral}, the technique was extended for more general dynamical systems using closed-loop regression; however, this required hidden-dimension parameters and polynomial dependence on the approximation guarantee. Spectral filtering techniques were recently used in non-linear sequence prediction, notably in the context of large language models, albeit with no theoretical guarantees \cite{agarwal2023spectral}. As convolutional models, these methods are attractive for sequence prediction due to faster generation and inference as compared to attention-based models \cite{agarwal2024futurefill}.

While several methods exist that can learn in the presence of asymmetric transition matrices \citep{kalman1960new, bakshi2023new, ghai2020no}, their performance depends on hidden dimension. On the other hand, spectral filtering methods \citep{hazan2017learning} achieve regret which is independent of hidden dimension, even for marginally stable systems. However, these spectral filtering methods were limited to systems with symmetric transition matrices. In contrast, real-world dynamical systems often have asymmetric transition matrices with large hidden dimension, necessitating a more general approach. In this paper, we provide such an approach by extending the theory of spectral filtering to handle asymmetric systems, as long as the complex component of their eigenvalues is bounded.

In this paper we dramatically improve the spectral filtering technique and broaden its applicability in two major aspects: First, for general asymmetric linear dynamical systems we remove the dependence on the hidden dimension. Second, we improve the dependence of the number of learned parameters from polynomial to logarithmic.    
\section{Memory Capacity of Linear Dynamical Systems} \label{subsection:memory}

The hidden dimension $d_h$, which is the dimension of the transition matrix $\A$, plays a significant role in the expressive power of LDS.
One of the most important features of the hidden dimension is that an LDS can memorize and recall inputs from up to ${d_{\textrm{hidden}}}$ iterations in the past.
This can be seen with the system where $\B,\C$ are identity, and $\A$ is given by the permutation matrix
$$ \A_{d_{\textrm{hidden}}}^{\mathrm{perm}} =
\begin{bmatrix}
0 & 0 & \cdots & 0 & 1 \\
1 & 0 & \cdots & 0 & 0 \\
0 & 1 & \cdots & 0 & 0 \\
\vdots & \vdots & \ddots & \vdots & \vdots \\
0 & 0 & \cdots & 0 & 0 \\
0 & 0 & \cdots & 1 & 0 \\
\end{bmatrix} , 
$$
which implements the memory operator $\y_t = \uv_{t - d_h}$.
Observe that any method which uses fewer than $d_{\textrm{hidden}}$ parameters will fail to implement this memory operator and therefore, for general linear dynamical systems, $d_{\textrm{hidden}}$ parameters are {\it necessary}.
Seemingly, this contradicts our promised results, which allows for learning a general LDS without hidden dimension dependence.
The explanation is in the spectrum of the system. Notice that the eigenvalues of the permutation matrix $\A$ above are the $d_{\textrm{hidden}}$ roots of unity given by 
$$ \lambda_1,...,\lambda_{d_{\textrm{hidden}}} \in \left\{ e^{2 \pi i \frac{k}{d}} , k = 1,2,...,d_{\textrm{hidden}} \right\}.
$$ Note that these eigenvalues have complex component as large as $1-1/d_{\textrm{hidden}}$.
Although in general a LDS can express signals with $d_{\textrm{hidden}}$ memory, and thus intuitively might require $d_{\textrm{hidden}}$ parameters, there are notable special cases that allow for efficient learning, i.e. learning the LDS with far fewer parameters.
A notable case is that of spectral filtering, which allows efficient learning of a {\it symmetric LDS} with poly-logarithmic (in the desired accuracy $\eps$) number of parameters.
The significance of a symmetric transition matrix $\A$ is that it can be diagonalized over the real numbers.
The natural question that arises is {\bf which asymmetric matrices can be learned by spectral filtering efficiently, and which characterization of their spectrum allows for efficient learning?
}

The answer we offer is surprisingly broad.  For a LDS with transition matrix $\A$, let $\lambda_1,...,\lambda_d$ be its complex eigenvalues.
We show that we can learn up to $\eps$ accuracy any LDS for which the largest eigenvalue has imaginary part bounded by $\frac{1}{\mathrm{poly} \log \frac{1}{\eps}} $.
We note that the spectral radius can be arbitrarily close to, or even equal to, one.
The only restriction is on the complex part, which is mildly constrained as a logarithmic function of $\eps$.
As per the permutation matrix example, this dependence is necessary and nearly tight - if the imaginary component of the eigenvalues of $\A$ becomes large, any learning method requires parameterization that depends on the hidden dimension of the system.
\section{Online Version of Preconditioning}
\label{appendix:online_version}
\begin{algorithm}[H]
		\caption{\label{alg:seq-precond} Universal Sequence Preconditioning 
(Online Version) }
		\begin{algorithmic}[1]
			\STATE Input: sequence prediction model $f_\theta$ with initial parameter $\theta^0$; loss function $\ell(\cdot)$; preconditioning coefficients $\mat{c}_{1:n}$.
            \FOR {$t = 1$ to $T$}
            \STATE Receive $\uv_t$
            \STATE Predict
            \begin{equation*}
                \hat{\y}_t = - \sum_{j = 1}^n \mat{c}_j \y_{t-j}  + f_{\theta^t}(\uv_{1:t}, \y_{1:t-1})
            \end{equation*}
            \STATE Observe true output $\y_t$ and suffer loss $\ell_t(\hat{\y}_t, \y_t)$.
            \STATE Update via projected gradient descent
            \begin{equation*}
                \theta^{t+1} \gets \mathrm{proj}_{\K}\left( \theta^t - \eta_t \nabla_{\theta} \ell_t(\hat{\y}_t, \y_t)\right)
            \end{equation*}
           \ENDFOR
		\end{algorithmic}
\end{algorithm}
\newtheorem*{retheorem}{Theorem}
\section{Proof of Convolutional Preconditioned Regression Performance Theorem~\ref{thm:convex_relaxation}}
\label{appendix:convex_relaxation}
We will prove Theorem~\ref{thm:convex_relaxation} first by proving the general result of Algorithm~\ref{alg:preconditioned_convex_relaxation} for any choice of preconditioning coefficients $\mat{c}_{0:n}$. Then we will apply the Chebyshev coefficients to the result to get Theorem~\ref{thm:convex_relaxation}.
For convenience, we restate Theorem~\ref{thm:convex_relaxation} here. 
\begin{theorem}[General Form of Theorem~\ref{thm:convex_relaxation}]
\label{thm:convex_relaxation_general}
    Let $\left \{ \uv_t \right \}_{t = 1}^T \in \R^{d_{\textrm{in}}}$ be any sequence of inputs which satisfy $\norm{\uv_t}_2 \leq 1$ and let $\left \{ \y_t \right \}_{t = 1}^T$ be the corresponding output coming from some linear dynamical system $(\A, \B, \C)$ as defined per Eq.~\ref{eqn:LDS}. Let $\mat{P}$ diagonalize $\A$ (note $\mat{P}$ exists w.l.o.g.) and let $\kappa = \norm{\mat{P}} \norm{\mat{P}^{-1}}$. Assume that $\norm{\B} \norm{\C} \kappa \leq C_{\textrm{domain}}$. Assume that $\norm{\B} \norm{\C} \kappa \leq C_{\textrm{domain}}$. Then the predictions $\hat{\y}_1, \dots, \hat{\y}_T$ from Algorithm~\ref{alg:preconditioned_convex_relaxation} satisfy
   \begin{equation*}
     \sum_{t = 1}^T \norm{\hat{\y}_t - \y_t}_1 \leq C_{\textrm{domain}} \left( \frac{3}{2}  n^2 \sqrt{d_{\textrm{out}}} \norm{\mat{c}}_1 \sqrt{T} +   \max_{\lambda \in \lambda(\A)} \abs{p_n^{\mat{c}} (\lambda)} T^2 \right).
\end{equation*}
\end{theorem}

\begin{proof}[Proof of Theorem~\ref{thm:convex_relaxation_general}]
For the remainder of the proof we will denote 
\begin{equation*}
    \hat{\y}_t(\Q) = - \sum_{i = 1}^n \mat{c}_i \y_{t-i} + \sum_{j = 0}^n \Q_j \uv_{t-j},
\end{equation*}
so that Algorithm~\ref{alg:preconditioned_convex_relaxation} outputs $\hat{\y}_t = \hat{\y}_t(\Q^t)$. Recall we define the domain 
\begin{equation*}
    \K = \left \{ \left( \Q_0, \dots, \Q_{n-1} \right) \textrm{ s.t. } \norm{\Q_j} \leq C_{\textrm{domain}} n \norm{\mat{c}}_1 \right \}. 
\end{equation*}
For convenience, we will use the shorthand $\Q$ to refer to $\left( \Q_0, \dots, \Q_{n-1}\right)$. First we prove that the regret of Algorithm~\ref{alg:preconditioned_convex_relaxation} as compared to the best $\Q^* \in \K$ is hindsight is small. Then we prove that the best $\Q^* \in \K$ in hindsight achieves small prediction error. 
Let $$G = \max_{t \in [T]} \norm{\nabla_{\Q} \ell_t(\Q^t)},$$ and let $$D = \max_{\Q^1, \Q^2 \in \K} \norm{\Q^1 - \Q^2 }.$$ By Theorem 3.1 from \cite{hazan2016introduction}, 
    \begin{equation*}
        \sum_{t = 1}^T \ell_t(\Q^t) -  \min_{\Q^* \in \K}  \sum_{t = 1}^T \ell_t(\Q^*) \leq \frac{3}{2} GD \sqrt{T}.
    \end{equation*}
Therefore it remains to bound $D$ and $G$. First we bound $D$. By definition of $\K$, 
\begin{align*}
D & = \max_{\Q^1, \Q^2 \in \K} \norm{\Q^1 - \Q^2 } \\
& =  \max_{(\Q^1_0, \dots, \Q^1_{n-1}), (\Q^2_0, \dots, \Q^2_{n-1}) \in \K} \norm{(\Q^1_0, \dots, \Q^1_{n-1}) - (\Q^2_0, \dots, \Q^2_{n-1}) }\\
    & \leq \sum_{j = 0}^{n-1} \norm{\Q^1_j - \Q^2_j} \\
    & \leq 2 C_{\textrm{domain}} n^2 \norm{\mat{c}}_1.
\end{align*}
Next we bound the gradient. Recall that
\begin{align*}
    \ell_t(\Q) & = \norm{\hat{\y}_t(\Q) - \y_t}_1 \\
    & = \norm{- \sum_{i = 1}^n \mat{c}_i \y_{t-i} + \sum_{j = 0}^n \Q_j \uv_{t-j} - \y_t}_1.
\end{align*}Note that, in general, $\nabla_{\A} \norm{\A \mat{x} - \mat{b}}_1 = \mathrm{sign}(\A \mat{x} - \mat{b}) \mat{x}^{\top}$. Since $\norm{\mat{x} \mat{y}^{\top}}_F \leq \norm{\mat{x}}_2 \norm{\mat{y}}_2$ we have
\begin{equation*}
    \norm{\nabla_{\A} f(\A)}_F \leq \sqrt{d} \norm{\mat{x}}_2,
\end{equation*}
where $d$ is the dimension of $\mat{b}$.
Using this and the assumption that for any $t \in [T]$, $\norm{\uv_t}_2 \leq 1$, we have
\begin{align*}
   \norm{ \nabla_{\Q_i} \ell_t(\Q) }_F & \leq \sqrt{d_{\textrm{out}}} \norm{\uv_{t-i}}_2 \leq \sqrt{d_{\textrm{out}}}. 
\end{align*}
Therefore, 
\begin{align*}
   G = \max_{t \in [T]} \norm{\nabla_{\Q}\ell_t(\Q^t)}_F \leq n \sqrt{d_{\textrm{out}}}.
\end{align*}
Thus we have a final regret bound
\begin{align*}
    \sum_{t = 1}^T \ell_t(\Q^t) - \min_{\Q^* \in \K} \ell_t(\Q^*) \leq \frac{3}{2} C_{\textrm{domain}} n^2 \sqrt{d_{\textrm{out}}} \norm{\mat{c}}_1 \sqrt{T}.
\end{align*}

Next we show that if $(\uv_{1:T}, \y_{1:T})$ is a linear dynamical system parameterized by $(\A, \B, \C)$, then for any $t \in [T]$, 
\begin{equation*}
   \min_{\Q^* \in \K} \ell_t(\Q^*) \leq \norm{\C} \norm{\B} \cdot \max_{\lambda \in \lambda(\A)} \abs{p_n^{\mat{c}} (\lambda)} \cdot T,
\end{equation*}
where $p_n^{\mat{c}}$ denotes the polynomial
\begin{equation*}
    p_n^{\mat{c}}(x) \defeq \sum_{i = 0}^n \mat{c}_i x^{n-i},
\end{equation*}
and $\lambda(\A)$ denotes the set of eigenvalues of $\A$. 
Indeed, if $(\uv_{1:T}, \y_{1:T})$ is a linear dynamical system parameterized by $(\A, \B, \C)$ then
\begin{equation*}
 \y_t = \sum_{s = 1}^{t} \C \A^{t-s} \B \uv_{s}.
\end{equation*}
With some linear algebra we get that convolving $\y_{t:t-n}$ with coefficients $\mat{c}_{0:n}$ results in
\begin{equation*}
\sum_{i=0}^n \mat{c}_i \y_{t-i} = \sum_{s = 0}^{n-1} \sum_{i = 0}^s \mat{c}_i \C \A^{s-i} \B \uv_{t-s} + \sum_{s = 0}^{t-n-1} \C p_n^{\mat{c}}(\A) \A^s \B \uv_{t-n-s}, 
\end{equation*}
or equivalently (since $\mat{c}_0 = 1$ due to the assertion in Algorithm~\ref{alg:preconditioned_convex_relaxation}),
\begin{equation*}
    \y_t = - \sum_{i=1}^n \mat{c}_i \y_{t-i} + \sum_{s = 0}^{n-1} \sum_{i = 0}^s \mat{c}_i \C \A^{s-i} \B \uv_{t-s} + \sum_{s = 0}^{t-n-1} \C p_n^{\mat{c}}(\A) \A^s \B \uv_{t-n-s}.
\end{equation*}
Set $\hat{\Q}_s = \sum_{i = 0}^s \mat{c}_i \C \A^{s-i} \B $ and set $\hat{\Q} =  \left( \hat{\Q}_0, \dots, \hat{\Q}_{n-1} \right)$. 
Since we assumed $C_{\textrm{domain}} \geq \norm{\C} \norm{\B}$, 
\begin{equation*}
    \norm{\hat{\mat{Q}}_i} \leq \sum_{j = 0}^i \abs{\mat{c}_j} \norm{\C \A^{i-j} \B} \leq \norm{\C} \norm{\B} \sum_{j = 0}^i \abs{\mat{c}_j} \leq C_{\textrm{domain}} n \norm{\mat{c}}_1.
\end{equation*}
Therefore $\hat{\Q}  \in \K$. Moreover,
\begin{align*}
   \hat{\y}_t(\hat{\Q}) - \y_t  & = \left( - \sum_{i = 1}^n \mat{c}_i \y_{t-i} + \sum_{j=0}^{n-1} \hat{\Q}_j \uv_{t-j} \right)   \\
   & \qquad - \left( - \sum_{i=1}^n \mat{c}_i \y_{t-i} + \sum_{s = 0}^{n-1} \sum_{i = 0}^s \mat{c}_i \C \A^{s-i} \B \uv_{t-s} + \sum_{s = 0}^{t-n-1} \C p_n^{\mat{c}}(\A) \A^s \B \uv_{t-n-s} \right) \\
   & = \sum_{s = 0}^{t-n-1} \C p_n^{\mat{c}}(\A) \A^s \B \uv_{t-n-s}.
\end{align*}
Therefore, $\norm{   \hat{\y}_t(\hat{\Q}) - \y_t}_1 = \norm{\sum_{s = 0}^{t-n-1} \C p_n^{\mat{c}}(\A) \A^s \B \uv_{t-n-s}}_1$. Let $\A$ be diagonalized by some $\mat{P}$ so that
\begin{equation*}
    \A = \mat{P} \mat{D} \mat{P}^{-1},
\end{equation*}
where $\mat{D}$ is the diagonalization of $\A$ and let $\kappa = \norm{\mat{P}} \norm{\mat{P}^{-1}}$, note that we can assume this w.l.o.g. since the set of dagonalizable matrices over the complex numbers is dense and therefore if $\A$ is not diagonalizable we may apply an aribtrarily small perturbation to it. Then since $\max_{j \in [\dhidden]} \abs{\lambda_j(\A)} \leq 1$,
\begin{align*}
    \norm{\C p_n^{\mat{c}}(\A) \A^j \B} & = \norm{\mat{C} \mat{P} p_n^{\mat{c}}(\mat{D}) \mat{D}^j \mat{P}^{-1} \B}  \\
    & \leq \max_{k \in [\dhidden]} \abs{p_n^{\mat{c}}(\mat{D}_{kk})} \cdot \norm{\C} \norm{\mat{P}} \norm{\mat{P}^{-1}} \norm{\mat{B}} \\
    & \leq \max_{\lambda \in \lambda(\A)} \abs{p_n^{\mat{c}}(\lambda)} \cdot \norm{\C}\norm{\mat{B}} \kappa.
\end{align*}
Thus,
\begin{equation*}
    \norm{  \hat{\y}_t(\hat{\Q}) - \y_t}_1 \leq \norm{\C} \norm{\mat{B}} \kappa \cdot \max_{\lambda \in \lambda(\A)} \abs{p_n^{\mat{c}} (\lambda)} \cdot T,
\end{equation*}
and so (recalling that we showed $\hat{\Q} \in \K$),
\begin{align*}
    \min_{\Q^* \in \K} \sum_{t = 1}^T \ell_t(\Q^*) & \leq \sum_{t = 1}^T \ell_t(\hat{\Q}) \\
    & = \sum_{t = 1}^T \norm{\hat{\y}_t(\hat{\Q}) - \y_t}_1 \\
    & \leq \norm{\C} \norm{\B} \kappa \cdot \max_{\lambda \in \lambda(\A)} \abs{p_n^{\mat{c}} (\lambda)} \cdot T^2. 
\end{align*}
Since $C_{\textrm{domain}} \geq \norm{\C} \norm{\B} \kappa $ we conclude,
\begin{equation*}
     \sum_{t = 1}^T \ell_t(\Q^t)  \leq C_{\textrm{domain}} \left( \frac{3}{2}  n^2 \sqrt{d_{\textrm{out}}} \norm{\mat{c}}_1 \sqrt{T} +   \max_{\lambda \in \lambda(\A)} \abs{p_n^{\mat{c}} (\lambda)} T^2 \right).
\end{equation*}
\hfill $\qedsymbol$ \end{proof}

Next we choose $\mat{c}$ to be the coefficients of the $n$-th monic Chebyshev polynomial to get the original theorem, Theorem~\ref{thm:convex_relaxation}, restated here for convenience.

\begin{retheorem}[Restatement of Theorem~\ref{thm:convex_relaxation}]
 Let $\left \{ \uv_t \right \}_{t = 1}^T \in \R^{d_{\textrm{in}}}$ be any sequence of inputs which satisfy $\norm{\uv_t}_2 \leq 1$ and let $\left \{ \y_t \right \}_{t = 1}^T$ be the corresponding output coming from some linear dynamical system $(\A, \B, \C)$ as defined per Eq.~\ref{eqn:LDS}.  Let $\mat{P}$ diagonalize $\A$ (note $\mat{P}$ exists w.l.o.g.) and let $\kappa = \norm{\mat{P}} \norm{\mat{P}^{-1}}$. Assume that $\norm{\B} \norm{\C} \kappa \leq C_{\textrm{domain}}$. Let $\lambda_1, \dots, \lambda_{\dhidden}$ denote the spectrum of $\A$. If 
    \begin{equation*}
        \max_{j \in [\dhidden]} \abs{\mathrm{Arg}(\lambda_j)} \leq \left( 64 \log_2 \left( \frac{8T^{3/2} }{3 \sqrt{d_{\textrm{out}}}}  \right) \right)^{-2},
    \end{equation*}
    then the predictions $\hat{\y}_1, \dots, \hat{\y}_T$ from Algorithm~\ref{alg:preconditioned_convex_relaxation} where the preconditioning coefficients $\mat{c}_{0:n}$ are chosen to be the coefficients of the $n$-th monic Chebyshev polynomial satisfy
    \begin{equation*}
         \frac{1}{T}\sum_{t=1}^T \norm{\hat{\y}_t - \y_t}_1 \leq \frac{9 C_{\textrm{domain}} \sqrt{d_{\textrm{out}}} \log_2^2(3 T)}{T^{2/13}}.
    \end{equation*}
\end{retheorem}

\begin{proof}
From Theorem~\ref{thm:convex_relaxation_general} we have 
\begin{equation*}
     \sum_{t = 1}^T \ell_t(\Q^t)  \leq C_{\textrm{domain}} \left( \frac{3}{2}  n^2 \sqrt{d_{\textrm{out}}} \norm{\mat{c}}_1 \sqrt{T} +   \max_{\lambda \in \lambda(\A)} \abs{p_n^{\mat{c}} (\lambda)} T^2 \right).
\end{equation*}
By Lemma~\ref{lemma:cheby_bound}, if for any eigenvalue $\lambda$ of $\A$, $\abs{\arg(\lambda)} \leq 1/64n^2$ then
\begin{equation*}
    \max_{\lambda \in \lambda(\A)} \abs{p_n^{\mat{c}}(\lambda)} \leq \frac{1}{2^{n-2}}.
\end{equation*}
Moreover, by Lemma~\ref{lemma:cheby_coeffs_bound}, $\norm{\mat{c}}_1 \leq 2^{0.3n}$. Thus the Chebyshev-conditioned predictor class satisfies
\begin{equation*}
     \sum_{t = 1}^T \ell_t(\Q^t)  \leq C_{\textrm{domain}} \left( \frac{3}{2}  n^2 \sqrt{d_{\textrm{out}}} 2^{0.3n} \sqrt{T} +  2^{-(n-2)} T^2\right).
\end{equation*}
Picking 
\begin{equation*}
    n = \frac{10}{13}\log_2 \left( \frac{8}{3 \sqrt{d_{\textrm{out}}}} T^{3/2} \right),
\end{equation*}
we get
\begin{align*}
 \sum_{t = 1}^T \ell_t(\Q^t)  & \leq 2\left(\frac{3}{2} \left(\frac{10}{13}\log_2 \left( \frac{8}{3 \sqrt{d_{\textrm{out}}}} T^{3/2} \right) \right)^2 
 \sqrt{d_{\textrm{out}}} \right)^{10/13} 4^{3/13}  T^{11/13} \\
 & \leq 9 C_{\textrm{domain}} \sqrt{d_{\textrm{out}}} \log_2(3 T)^2 T^{11/13}.
\end{align*}
Dividing both sides by $T$ we get the stated result. 
\hfill $\qedsymbol$ \end{proof}
 
\section{Proof of Theorem~\ref{thm:main_regret}} \label{sec:formal_spectral}
\subsection{Preliminaries and Notation}
\label{sec:preliminaries}
We analyze the output sequence $\left \{ \y_t \right \}_{t = 1}^T$ generated by a linear dynamical system $(\A, \B, \C)$ with inputs $\left \{ \uv_t \right \}_{t = 1}^T \in \R^{d_{\textrm{in}}}$ satisfying $\norm{\uv_t}_2 \leq 1$.
We assume that the dynamics matrix $\A$ is diagonalizable. Let $\mat{P}$ be the matrix that diagonalizes $\A$ (which exists w.l.o.g. as diagonalizable matrices are dense in $\mathbb{C}^{d \times d}$), and let $\kappa = \norm{\mat{P}} \norm{\mat{P}^{-1}}$ denote the condition number of the eigenbasis.

Recall the spectral domain $\complex_{\beta} = \left \{ z \in \complex \mid \abs{z} \leq 1, \abs{\arg(z)} \leq \beta \right \}$ defined in Section~\ref{sec:usp_sf}.
For a given polynomial $p_n^{\mathbf{c}}(x)$ with coefficients $\mathbf{c}$, let $B_n = \max_{\alpha \in \complex_{\beta}}  \abs{p_n^{\mathbf{c}}(\alpha)}$.

Recall the definitions of $\tilde{p}_n^{\mathbf{c}}(\alpha)$, $\tilde{\mu}(\alpha)$, and the spectral covariance matrix $\mat{Z}$ from Section~\ref{sec:usp_sf} (Eqs.~\ref{eqn:mu_alpha}--\ref{eqn:Z_def}).

Let $\phi_1, \dots, \phi_T$ be the eigenvectors of $\mat{Z}$.
Let $\uv_{t:1}$ be the concatenated inputs up to time $t$ which are padded to create a length $T$ vector,
\begin{equation}
    \uv_{t:1} \defeq \begin{bmatrix}
        & \uv_t 
        & \uv_{t-1}
        & \dots 
        & \uv_1
        & 0 
        & \dots
        & 0
    \end{bmatrix}^{\top}.
\end{equation}

\subsection{General Form of Main Result}
First, we prove a more general form that holds for any choice of $\mathbf{c}_{0:n}$ and resulting polynomial $p_n^{\mathbf{c}}$.
\begin{theorem}
\label{thm:main_regret_general}
Let the system and notation be defined as in Section~\ref{sec:preliminaries}.
If the radius parameters of Algorithm~\ref{alg:new_sf} are set to:
    \begin{align*}
        & R_Q = \norm{\C} \norm{\B} \norm{\mathbf{c}}_1 \\
        & R_M = 2 \norm{\mat{C}} \norm{\mat{B}} \kappa \log (T) \beta^{4/3} T^{7/6} B_n 
    \end{align*}
    then the predictions $\hat{\y}_1, \dots, \hat{\y}_T$ from Algorithm~\ref{alg:new_sf} satisfy
      \begin{align*}
    \sum_{t=1}^T \ell_t(\Q^t, \M^t) \leq 18 \kappa \norm{\mat{C}} \norm{\mat{B}} \log (T) (n+k)  \left(T^3 \beta^{4/3} B_n k + n  \norm{\mathbf{c}}_1 T^{1/2}  \right).
\end{align*}
\end{theorem}

Critically, we observe that the guaranteed regret bound does not depend on the hidden dimension of the dynamics matrix $\A$.
While the general version of Algorithm~\ref{alg:new_sf} is interesting in its own right, we show that by choosing the coefficients of the algorithm to be the Chebyshev coefficients we obtain sublinear absolute error.
\begin{retheorem}[Detailed Version of Theorem~\ref{thm:main_regret}]
  Let the assumptions of Theorem~\ref{thm:main_regret_general} hold.
Furthermore, suppose that for any eigenvalue, $\lambda_j$, of $\mat{A}$
    \begin{equation*}
        \max_{j \in [\dhidden]} \abs{\arg(\lambda_j)} \leq T^{-1/4} \cdot T^{-13p/4},
    \end{equation*}
  then the predictions $\hat{\y}_1, \dots, \hat{\y}_T$ from Algorithm~\ref{alg:new_sf} where the preconditioning coefficients $\mathbf{c}_{0:n}$ are chosen to be the coefficients of the $n$-th monic Chebyshev polynomial satisfy
    \begin{equation*}
        \frac{1}{T} \sum_{t=1}^T \norm{\hat{\y}_t - \y_t}_1 \leq \tilde{O} \left( \frac{\kappa \norm{\mat{C}} \norm{\mat{B}}}{T^p}, \right),
    \end{equation*}
    where $\tilde{O}(\cdot)$ hides polylogarithmic factors in $T$.
\end{retheorem}
\begin{proof}
We introduce parameter $\beta$ to denote the maximum argument of the eigenvalues of $\mat{A}$ so that the spectrum of $\mat{A}$ lies in $\complex_{\beta}$.
From Theorem~\ref{thm:convex_relaxation_general} we have 
\begin{equation*}
     \sum_{t = 1}^T \ell_t(\Q^t, \M^t)  \leq 18 \kappa \norm{\mat{C}} \norm{\mat{B}} \log (T) (n+k)  \left(T^3 \beta^{4/3} B_n k + n  \norm{\mathbf{c}}_1 T^{1/2}  \right).
\end{equation*}
By Lemma~\ref{lemma:cheby_bound}, if for any eigenvalue $\lambda$ of $\A$, $\abs{\arg(\lambda)} \leq 1/64n^2$ then
\begin{equation}
\label{eqn:eig_bnd}
    \max_{\lambda \in \lambda(\A)} \abs{p_n^{\mathbf{c}}(\lambda)} \leq \frac{1}{2^{n-2}} = 4\cdot 2^{-n}.
\end{equation}
We will choose 
\begin{equation}
\label{eqn:n_choice}
    n = \frac{10}{13} \log_2 \left( T^{-1/2} T^3 \beta^{4/3}\right),
\end{equation}
and so if $\beta< 1/T^{1/4}$ then $\beta < 1/64n^2$, meaning that Eq.~\ref{eqn:eig_bnd} holds.
Moreover, by Lemma~\ref{lemma:cheby_coeffs_bound}, $\norm{\mathbf{c}}_1 \leq 2^{0.3n}$. Thus the Chebyshev-conditioned predictor class satisfies
\begin{equation*}
     \sum_{t = 1}^T \ell_t(\Q^t)  \leq 60 \kappa \norm{\mat{C}} \norm{\mat{B}} \log (T) (n+k)  \left(T^3 \beta^{4/3} 2^{-n} k + n  2^{0.3n} T^{1/2}  \right)
\end{equation*}
Then for $n$ as chosen above in Eq.~\ref{eqn:n_choice},
\begin{align*}
T^3 \beta^{4/3} 2^{-n} + 2^{0.3n} T^{1/2} = 2 \left( T^{14} \beta^{4} \right)^{1/13}.
\end{align*}
Assuming that $\beta<1$ we have that $n < 3 \log_2(T)$
and therefore
\begin{align*}
 \sum_{t = 1}^T \ell_t(\Q^t)  & \leq 120 \kappa \norm{\mat{C}} \norm{\mat{B}} \log (T) n^2 k^2   \left(T^3 \beta^{4/3} 2^{-n}  +   2^{0.3n} T^{1/2}  \right) \tag{Factoring out $n$ and $k$ and using $n+k \leq 2n k $} \\
 & \leq 240 \kappa \norm{\mat{C}} \norm{\mat{B}} \log (T) \left( 3 \log_2\left( T\right) \right)^2  k^2  \left( T^{14} \beta^{4} \right)^{1/13}  \tag{Plugging in chosen value for $n$.} \\
 & \leq 720 \kappa \norm{\mat{C}} \norm{\mat{B}}  \log (T)^3 k^2 \left( T^{14} \beta^{4} \right)^{1/13}.
\end{align*}
Therefore if $\beta = T^{-1/4} \cdot T^{-13p/4}$ then the final accumulated error is 
\begin{align*}
 \sum_{t = 1}^T \ell_t(\Q^t) \leq 720 \kappa \norm{\mat{C}} \norm{\mat{B}}  \log (T)^3 k^2  T^{1-p},
 \end{align*}
 which is sublinear as long as $p > 0$.
\hfill $\qedsymbol$ \end{proof}

\paragraph{Remark on the loss function. }  The reader may notice we use the $\ell_1$, or absolute loss, rather than Euclidean or other loss functions.
All norms are equivalent up to the (output) dimension, and thus learning to predict as well as the best linear dynamical system in hindsight is meaningful in any norm.
However, we make this technical choice since it greatly simplifies the regret bounds, and in particular the bound on the gradient norms, which is technically involved.
We conjecture that sublinear regret bounds are attainable in other norms as well, and leave it for future work.
We prove Theorem~\ref{thm:main_regret} on an (equivalent) algorithm, where we rescale the parameter $\M_j$ by $\sqrt{T}$ and the input $\langle \phi_j, \uv_{(t-n-i):1} \rangle$ by $1/\sqrt{T}$.
We account for this rescaling by increasing the size of the domain for $\M$ by $\sqrt{T}$.
The proof of Theorem~\ref{thm:main_regret} proceeds in two parts. The first is to show that any linear dynamical signal is well approximated by a predictor of the form in Algorithm~\ref{alg:new_sf} $\hat{\y}(\Q, \M)$ where $(\Q, \M) \in \K$.
\begin{lemma}[Approximation Lemma]
\label{lemma:approximation}
Suppose $\y_{1:T}$ evolves as a linear dynamical system characterized by $(\A, \B, \C)$ as in Eq.~\ref{eqn:LDS} satisfying the assumptions in Section~\ref{sec:preliminaries}.
Consider domain
    \begin{align*}
    \K =  \{ (\Q_1, \dots, \Q_n, \M_1, \dots, \M_k) \textrm{ s.t.
} & \norm{\Q_i} \leq R_Q, \textrm{ and } \norm{\M_i} \leq R_M \}.
\end{align*}
   If 
    \begin{align*}
        & R_Q \geq \norm{\C} \norm{\B} \norm{\mathbf{c}}_1, \\
        & R_M \geq 2 \norm{\mat{C}} \norm{\mat{B}} \kappa \log (T) \beta^{4/3} T^{7/6} B_n, 
    \end{align*}
    then there exists $(\hat{\Q}_1, \dots, \hat{\Q}_n, \hat{\M}_1, \dots, \hat{\M}_k) \in \K$ such that for prediction (as in Algorithm~\ref{alg:new_sf})
    \begin{equation*}
        \hat{\y}_t(\hat{\Q}, \hat{\M}) = - \sum_{i = 1}^n c_i \y_{t-i} + \sum_{j = 0}^n \hat{\Q}^t_j \uv_{t-j} + \frac{1}{\sqrt{T}} \sum_{j = 1}^k \hat{\M}_j^t \phi_j^{\top} \tilde{\uv}_{t-n-1:1},
    \end{equation*}
it holds that
    \begin{equation*}
        \norm{\hat{\y}_t - \y_t}_1 \leq 6 \kappa \norm{\mat{C}}  \norm{\mat{B}} \log(T) T^2 \beta^{4/3} B_n.
\end{equation*}
\end{lemma}

\begin{proof}[Proof of Lemma~\ref{lemma:approximation}]
Suppose $\y_{1:T}$ evolves as a linear dynamical system characterized by $(\A, \B, \C)$.
Then given inputs $\uv_{1:t}$,
\begin{equation*}
 \y_t = \sum_{s = 1}^{t} \C \A^{t-s} \B \uv_{s}.
\end{equation*}
With some linear algebra we get that convolving $\y_{t:t-n}$ with coefficients $\mathbf{c}_{0:n}$ results in
\begin{equation}
\label{eqn:y_t}
    \y_t = - \sum_{i=1}^n \mathbf{c}_i \y_{t-i} + \sum_{s = 0}^{n-1} \sum_{i = 0}^s \mathbf{c}_i \C \A^{s-i} \B \uv_{t-s} + \sum_{s = 0}^{t-n-1} \C p_n^{\mathbf{c}}(\A) \A^s \B \uv_{t-n-s}.
\end{equation}
Set $\hat{\Q}_s = \sum_{i = 0}^s \mathbf{c}_i \C \A^{s-i} \B $ and set $\hat{\Q} =  \left( \hat{\Q}_0, \dots, \hat{\Q}_{n-1} \right)$.
Then
\begin{align*} \norm{\hat{\Q}_i} \leq  \sum_{j = 0}^i \abs{\mathbf{c}_j} \norm{\C \A^{i-j} \B} \leq  \norm{\C} \norm{\B}  \sum_{j = 0}^i \abs{\mathbf{c}_j} \leq  \norm{\C} \norm{\B} \norm{\mathbf{c}}_1.
\end{align*}

Next we turn our attention to the spectral filtering parameters.
Using the notation $\tilde{p}_n^{\mathbf{c}}(\alpha)$ and $\tilde{\mu}(\alpha)$ from Section~\ref{sec:preliminaries}, we define the combined vector $\mu_{\tilde{p}_n}(\alpha) \defeq p_n^{\mathbf{c}}(\alpha) \tilde{\mu}(\alpha)$.
Let $\mat{D}$ be the diagonalization of $\A$. 
Then,
\begin{align*}
    \sum_{s = 0}^{t-n-1} \C \tilde{p}_n^{\mathbf{c}}(\A) \A^s \B \uv_{t-n-s}   & = \sum_{s = 0}^{t-n-1}  \mat{C} \mat{P} \tilde{p}_n^{\mathbf{c}}(\mat{D}) \mat{D}^s \mat{P}^{-1} \B  \uv_{t-n-s} \\
   & = \sum_{s = 0}^{t-n-1}  \mat{C} \mat{P} \left(  \sum_{m = 1}^{\dhidden} \tilde{p}_n^{\mathbf{c}}(\alpha_m) \alpha_m^s \mat{e}_m \mat{e}_m^{\top} \right) \mat{P}^{-1} \B  \uv_{t-n-s} \\
   & = \sum_{m = 1}^{\dhidden}  \sum_{s = 0}^{t-n-1}  \mat{C} \mat{P} \mat{e}_m  \mat{e}_m^{\top} \mat{P}^{-1} \B   \tilde{p}_n^{\mathbf{c}}(\alpha_m) \alpha_m^s  \uv_{t-n-s} \\
    & = \sum_{m = 1}^{\dhidden}  \mat{C} \mat{P} \mat{e}_m  
\mat{e}_m^{\top} \mat{P}^{-1} \B  \mu_{\tilde{p}_n}(\alpha_m)^{\top} \tilde{\uv}_{t-n-1:1} \\
      & = \sum_{m = 1}^{\dhidden}  \mat{C} \mat{P} \mat{e}_m  \mat{e}_m^{\top} \mat{P}^{-1} \B  \mu_{\tilde{p}_n}(\alpha_m)^{\top} \left(\sum_{j=1}^{T-n} \phi_j \phi_j^{\top} \right) \tilde{\uv}_{t-n-1:1} \tag{Orthonormality of the filters} \\
        & = \sum_{j = 1}^{T-n} \left( \sum_{m = 1}^{\dhidden}  \mat{C} \mat{P} \mat{e}_m  \mat{e}_m^{\top} \mat{P}^{-1} \B  \mu_{\tilde{p}_n}(\alpha_m)^{\top} \phi_j \right) \phi_j^{\top}  \tilde{\uv}_{t-n-1:1}.
\end{align*}
Therefore defining
    \begin{equation*}
        \hat{\M}_j \defeq T^{1/2} \sum_{m = 1}^{\dhidden}  \mat{C} \mat{P} \mat{e}_m  \mat{e}_m^{\top} \mat{P}^{-1} \B  \mu_{\tilde{p}_n}(\alpha_m)^{\top} \phi_j,
    \end{equation*}
    we have
  \begin{align*}
   \sum_{s = 0}^{t-n-1} \C \tilde{p}_n^{\mathbf{c}}(\A) \A^s \B \uv_{t-n-s}  
        & = \sum_{j = 1}^{T-n} \hat{\mat{M}}_j \frac{ \phi_j^{\top}  \tilde{\uv}_{t-n-1:1} }{\sqrt{T}}.
\end{align*}  
Next we bound the norm of $\hat{\M}_j$.Let $\mat{S}$ be the diagonal matrix with entries $\mat{S}_{mm} = \mu_{\tilde{p}_n}(\alpha_m)^{\top} \phi_j$.
Note that $\hat{\M}_j = \mat{C} \mat{P} \mat{S} \mat{P^{-1}} \mat{B}$. For short, let $C_{\kappa} = \norm{\mat{C}} \norm{\mat{B}} \kappa$.
Recalling that $\mu_{\tilde{p}_n}(\alpha) = p_n^{\mathbf{c}}(\alpha) \mu (\alpha)$,
\begin{equation}
\label{eqn:Mbound}
    \norm{\hat{\M}_j}  = \norm{T^{1/2} \mat{C} \mat{P} \mat{S} \mat{P^{-1}} \mat{B}}  \leq C_{\kappa} T^{1/2} \max_{m \in [\dhidden]}  \abs{p_n^{\mathbf{c}}(\alpha_m)}\cdot \max_{m \in [\dhidden]} \abs{\tilde{\mu}(\alpha_m)^{\top} \phi_j}.
\end{equation}
By Lemma~\ref{lemma:spectral_filtering_property}, 
\begin{equation*}
     \max_{\alpha \in \complex_{\beta}} \abs{\tilde{\mu}(\alpha)^{\top} \phi_j} \leq 2 \log (T)  \beta^{4/3} T^{2/3}.
\end{equation*}
Therefore,
\begin{equation*}
    \norm{\hat{\M}_j} \leq 2 C \log (T) \beta^{4/3} T^{7/6} \max_{\alpha \in \complex_{\beta}}  \abs{p_n^{\mathbf{c}}(\alpha)}.
\end{equation*}
Therefore $(\hat{\Q}, \hat{\M}) \in \K$ for the chosen radius of $\K$ (i.e. $R_Q$ and $R_M$).
Finally, we bound the truncation error:
\begin{equation*}
    \hat{\y}_t(\hat{\Q}, \hat{\M}) - \y_t  = \sum_{j = k+1}^{T-n} \hat{\mat{M}}_j \frac{ \phi_j^{\top}  \tilde{\uv}_{t-n-1:1} }{\sqrt{T}}.
\end{equation*}

\begin{align*}
    \norm{ \hat{\y}_t(\hat{\Q}, \hat{\M}) -  \y_t}_2 
     & \leq \sum_{j = k+1}^{T-n} \norm{ \hat{\M}_j } \norm{\phi_j^{\top}}_1 \norm{ \uv_{(t-n-1):1}}_{\infty} T^{-1/2} \\
     & \leq \sum_{j = k+1}^{T-n} \norm{ \hat{\M}_j } \tag{$\norm{\phi_j}_2 = 1 \implies \norm{\phi_j}_1 \leq \sqrt{T} $} \\
      & \leq C_{\kappa} T^{1/2} B_n  \cdot \sum_{j = k+1}^{T-n} \max_{\alpha \in \mathbb{C}_{\beta}} \abs{\tilde{\mu}(\alpha)^{\top} \phi_j} \tag{Bound on $\norm{\hat{\mat{M}}_j}$ from Eq.~\ref{eqn:Mbound}} \\
      & \leq C_{\kappa} T^{1/2} {B_n} \cdot \left(6 \log(T) \beta^{4/3} T^{3/2} \right) \tag{Lemma~\ref{lemma:spectral_filtering_property}} \\
  & = 6 \norm{\mat{C}} \norm{\mat{B}} \kappa \log(T) T^2 \beta^{4/3} B_n.
\tag{Plugging in value for $C_{\kappa}$.}
\end{align*}
\hfill $\qedsymbol$ \end{proof}

The next result provides the regret of Online Gradient Descent when compared to the best $(\Q^*, \M^*) \in \K$.
\begin{lemma}[Online Gradient Descent]
\label{lemma:ogd}
Recall the domain $\K$ in Algorithm~\ref{alg:new_sf} be
\begin{align*}
    \K =  \{ (\Q_1, \dots, \Q_n, \M_1, \dots, \M_k) \textrm{ s.t.
} & \norm{\Q_i} \leq R_Q \textrm{ and } \norm{\M_i} \leq R_M \}.
\end{align*}

The iterates output by Algorithm~\ref{alg:new_sf} satisfy 
\begin{equation*}
        \sum_{t = 1}^T \ell_t(\Q^t, \M^t) -  \min_{(\Q^*, \M^*) \in \K}  \sum_{t = 1}^T \ell_t(\Q^*, \M^*) \leq \frac{3}{2} (n R_Q + k R_M) (n+k) \sqrt{d_{\textrm{out}}} \sqrt{T}.
\end{equation*}
\end{lemma}

\begin{proof}[Proof of Lemma~\ref{lemma:ogd}]
Let $G = \max_{t \in [T]} \norm{\nabla_{\Q,\M} \ell_t(\Q^t, \M^t)}$ and let $$D = \max_{(\Q^1, \M^1), (\Q^2, \M^2) \in \K} \norm{(\Q^1, \M^1) - (\Q^2, \M^2)}.$$ By Theorem 3.1 from \cite{hazan2016introduction}, 
    \begin{equation*}
        \sum_{t = 1}^T \ell_t(\Q^t, \M^t) -  \min_{(\Q^*, \M^*) \in \K}  \sum_{t = 1}^T \ell_t(\Q^*, \M^*) \leq \frac{3}{2} GD \sqrt{T}.
\end{equation*}
    Therefore it remains to bound $G$ and $D$.
By definition of $\K$ we have
    \begin{equation*}
        D \leq n R_Q + k R_M.
\end{equation*}
    For $G$ we compute the subgradient at any $\Q_i$ and $\M_i$.
Note that, in general, $\nabla_{\A} \norm{\A \mat{x} - \mat{b}}_1 = \mathrm{sign}(\A \mat{x} - \mat{b}) \mat{x}^{\top}$.
Since $\norm{\mat{x} \mat{y}^{\top}}_F \leq \norm{\mat{x}}_2 \norm{\mat{y}}_2$ we have
\begin{equation*}
    \norm{\nabla_{\A} f(\A)}_F \leq \sqrt{d} \norm{\mat{x}}_2,
\end{equation*}
where $d$ is the dimension of $\mat{b}$.
Using this and the assumption that for any $t \in [T]$, $\norm{\uv_t}_2 \leq 1$, we have
\begin{align*}
     \norm{ \nabla_{\M_j} \ell_t(\Q,\M) } & \leq \sqrt{d_{\textrm{out}}} \norm{\phi_j^{\top} \tilde{\uv}_{t-n-1:1} T^{-1/2}} \\
     & \leq \sqrt{d_{\textrm{out}}}  \norm{\phi_j}_1 \norm{\tilde{\uv}_{t-n-1:1}}_{\infty} T^{-1/2} \\
     & \leq \sqrt{d_{\textrm{out}}} .
\tag{$\norm{\tilde{\uv}_{t-n-1:1}}_{\infty} \leq 1$ and $\norm{\phi_j}_1 \leq \sqrt{T}$ since $\norm{\phi_j}_2 \leq 1$ }
\end{align*}
Next,
\begin{align*}
   \norm{ \nabla_{\Q_i} \ell_t(\Q) }_F & \leq \sqrt{d_{\textrm{out}}} \norm{\uv_{t-i}}_2 \leq \sqrt{d_{\textrm{out}}}.
\end{align*}
Therefore, 
\begin{align*}
   G = \max_{t \in [T]} \norm{\nabla_{(\Q, \M)}\ell_t(\Q^t, \M^t)}_F \leq (n+k) \sqrt{d_{\textrm{out}}}.
\end{align*}
    Therefore, we have 
 \begin{equation*}
        \sum_{t = 1}^T \ell_t(\Q^t, \M^t) -  \min_{(\Q^*, \M^*) \in \K}  \sum_{t = 1}^T \ell_t(\Q^*, \M^*) \leq \frac{3}{2} (n R_Q + k R_M) (n+k) \sqrt{d_{\textrm{out}}} \sqrt{T}.
\end{equation*}

\hfill $\qedsymbol$ \end{proof}

Combining Lemma~\ref{lemma:approximation} and Lemma~\ref{lemma:ogd} proves Theorem~\ref{thm:main_regret}. 
\begin{proof}[Proof of Theorem~\ref{thm:main_regret}]
    By Lemma~\ref{lemma:ogd} the iterates from Algorithm~\ref{alg:new_sf} $(\Q^1, \M^1)$, \dots, $(\Q^T, \M^T)$ satisfy
    \begin{equation*}
     \sum_{t = 1}^T \ell_t(\Q^t, \M^t) -  \min_{(\Q^*, \M^*) \in \K}  \sum_{t = 1}^T \ell_t(\Q^*, \M^*) \leq \frac{3}{2} (n R_Q + k R_M) (n+k) \sqrt{d_{\textrm{out}}} \sqrt{T}.
\end{equation*}
By Lemma~\ref{lemma:approximation}, if $\y_{1:T}$ comes from a linear dynamical system then for large enough radius parameters of $\K$ ($R_Q$ and$R_M$), there exists $(\hat{\Q}, \hat{\M}) \in \K$ such that 
\begin{equation*}
    \sum_{t = 1}^T \ell_t(\hat{\Q}, \hat{\M}) \leq T \left( 6 \kappa \norm{\mat{C}}  \norm{\mat{B}} \log(T) T^2 \beta^{4/3} B_n \right).
\end{equation*}
Let $C_T = 6 \kappa \norm{\mat{C}}  \norm{\mat{B}} \log(T)$ for short.
Since $(\hat{\Q}, \hat{\M}) \in \K$ we must also have that
\begin{equation*}
    \min_{(\Q^*, \M^*) \in \K}  \sum_{t = 1}^T \ell_t(\Q^*, \M^*) \leq C_T T^3 \beta^{4/3} B_n.
\end{equation*}
Therefore, 
\begin{align*}
    \sum_{t=1}^T \ell_t(\Q^t, \M^t) & \leq \min_{(\Q^*, \M^*) \in \K}  \sum_{t = 1}^T \ell_t(\Q^*, \M^*) + \frac{3}{2} (n R_Q + k R_M) (n+k) \sqrt{d_{\textrm{out}}} \sqrt{T} \\
    & \leq C_T T^3 \beta^{4/3} B_n +  \frac{3}{2} (n R_Q + k R_M) (n+k) \sqrt{d_{\textrm{out}}} \sqrt{T}.
\end{align*}
Recall that we set $R_Q = \norm{\C} \norm{\B} \norm{\mathbf{c}}_1$ and $R_M = 2 \norm{\mat{C}} \norm{\mat{B}} \kappa \log (T) \beta^{4/3} T^{7/6} B_n$.
Using $C_T$ we bound $R_M \leq C_T \beta^{4/3} T^{7/6} B_n$.
We conclude,
\begin{align*}
    \sum_{t=1}^T \ell_t(\Q^t, \M^t) & \leq C_T T^3 \beta^{4/3} B_n +  \frac{3}{2} C_T \bigg( n \norm{\mathbf{c}}_1  + k \beta^{4/3} T^{7/6} B_n \bigg) (n+k) \sqrt{d_{\textrm{out}}} \sqrt{T}  \\
    & \leq C_T  \left( T^3 \beta^{4/3} B_n +  \frac{3}{2} n(n+k) \norm{\mathbf{c}}_1 T^{1/2}  + \frac{3}{2} k (n+k) \beta^{4/3} T^{5/3} B_n  \right) \\
    & \leq C_T  \left( 3 T^3 \beta^{4/3} B_n k (n+k) +  \frac{3}{2} n(n+k) \norm{\mathbf{c}}_1 T^{1/2} \right) \tag{Bounding the right-most term by the left-most term times $k(n+k)$} \\
    & \leq 6 \kappa 
\norm{\mat{C}} \norm{\mat{B}} \log (T)  \left( 3 T^3 \beta^{4/3} B_n k (n+k) +  \frac{3}{2} n(n+k) \norm{\mathbf{c}}_1 T^{1/2} \right) \tag{Plugging in $C_T = 5 \kappa \norm{\mat{C}}  \norm{\mat{B}} \log(T)$} \\
    & \leq 18 \kappa \norm{\mat{C}} \norm{\mat{B}} \log (T) (n+k)  \left(T^3 \beta^{4/3} B_n k + n  \norm{\mathbf{c}}_1 T^{1/2}  \right).
\end{align*}
\hfill $\qedsymbol$ \end{proof}

\subsection{Proving Approximation Lemma~\ref{lemma:approximation}}
\label{sec:approximation_lemma_proof}

\begin{lemma}
\label{lemma:spectral_filtering_property}
Let $\tilde{\mu}(\alpha)$ and $\mat{Z}$ be as defined in Section~\ref{sec:preliminaries} and let $\phi_1, \dots, \phi_{T}$ be the eigenvectors of $\mat{Z}_{T}$.
\begin{equation*}
    \max_{\alpha \in \complex_{\beta}} \abs{\tilde{\mu}(\alpha)^{\top} \phi_j} \leq 2 \log (T)  \beta^{4/3} T^{2/3} .
\end{equation*}
Moreover,
\begin{equation*}
\sum_{j=1}^T \abs{\tilde{\mu}(\alpha)^{\top} \phi_j}   \leq 6 \log (T) \beta^{4/3} T^{3/2}.
\end{equation*}
\end{lemma}

\subsubsection{Proof of the Spectral Filtering Property Lemma~\ref{lemma:spectral_filtering_property}}
\label{sec:lipschitz_bound}

In order to prove Lemma~\ref{lemma:spectral_filtering_property} we require two further helper lemmas.
The first is Lemma~\ref{lemma:lipschitz}, which roughly argues that the Lipschitz constant of a function $f: \complex_{\beta} \to \R$ can be bounded by a polynomial of the expectation of the function on $\complex_{\beta}$.
The second is Lemma~\ref{thm:novel_spectral_decay}, which argues that the matrix $\mat{Z}$ from Eq.~\ref{eqn:Z_def} that we use to derive the new spectral filters has small eigenvalues.
\begin{lemma} \label{lemma:lipschitz}
Let $L>0$, $g_{\max}>0$ and $0\le\beta\le 1$.  Define
\[
  \mathbb C_\beta=\bigl\{\,z \in\mathbb C:\;|z|\le1,\;\abs{\arg(z)} \le\beta\bigr\},
\]
and let $\mathcal F$ be the set of non-negative, $L$-Lipschitz functions
$f:\mathbb C_\beta\!\to\!\mathbb R$ such that
\(\max_{z\in\mathbb C_\beta}f(z)=g_{\max}\).
Then
\[
  \boxed{\;
    \min_{f\in\mathcal F}
      \int_{\mathbb C_\beta} f(z)\,dz
    \;\ge\;
    \frac{\beta\,g_{\max}^{3}}{24\,L^{2}}\;}
  .
\]
\end{lemma}

\begin{proof}[Proof of Lemma~\ref{lemma:lipschitz}]
\textbf{1.  The extremal function.} 
Fix $f\in\mathcal{F}$ and choose $z_\star$ with $f(z_\star)=g_{\max}$.
Among admissible functions we consider only those that decrease as fast
as the Lipschitz constraint allows in every radial direction, replacing $f$
by the {\em extremal cone}.
Set
\[
   h(z):=[\,g_{\max}-L|z-z_\star|\,]_+,
   \qquad
   r:=\frac{g_{\max}}{L}.
\]
Because $0\le h\le f$ on $\mathbb C_\beta$, it suffices to lower bound
\(\displaystyle\int_{\mathbb C_\beta}h\).

\medskip
\textbf{2.
Geometry around $z_\star$.}  It can be seen by Euclidean geometry,  the point $z_\star = 0$ minimize the volume of the intersection with $C_\beta$, up to a factor of ${4}$, to make it the extremal function.
It is depicted in Figure \ref{fig:cone-int}. Notice that the area of the intersection is colored in blue, and it is particularly easy to integrate over.
\begin{figure}[ht!] 
  \centering
  \includegraphics[width=.55\textwidth]{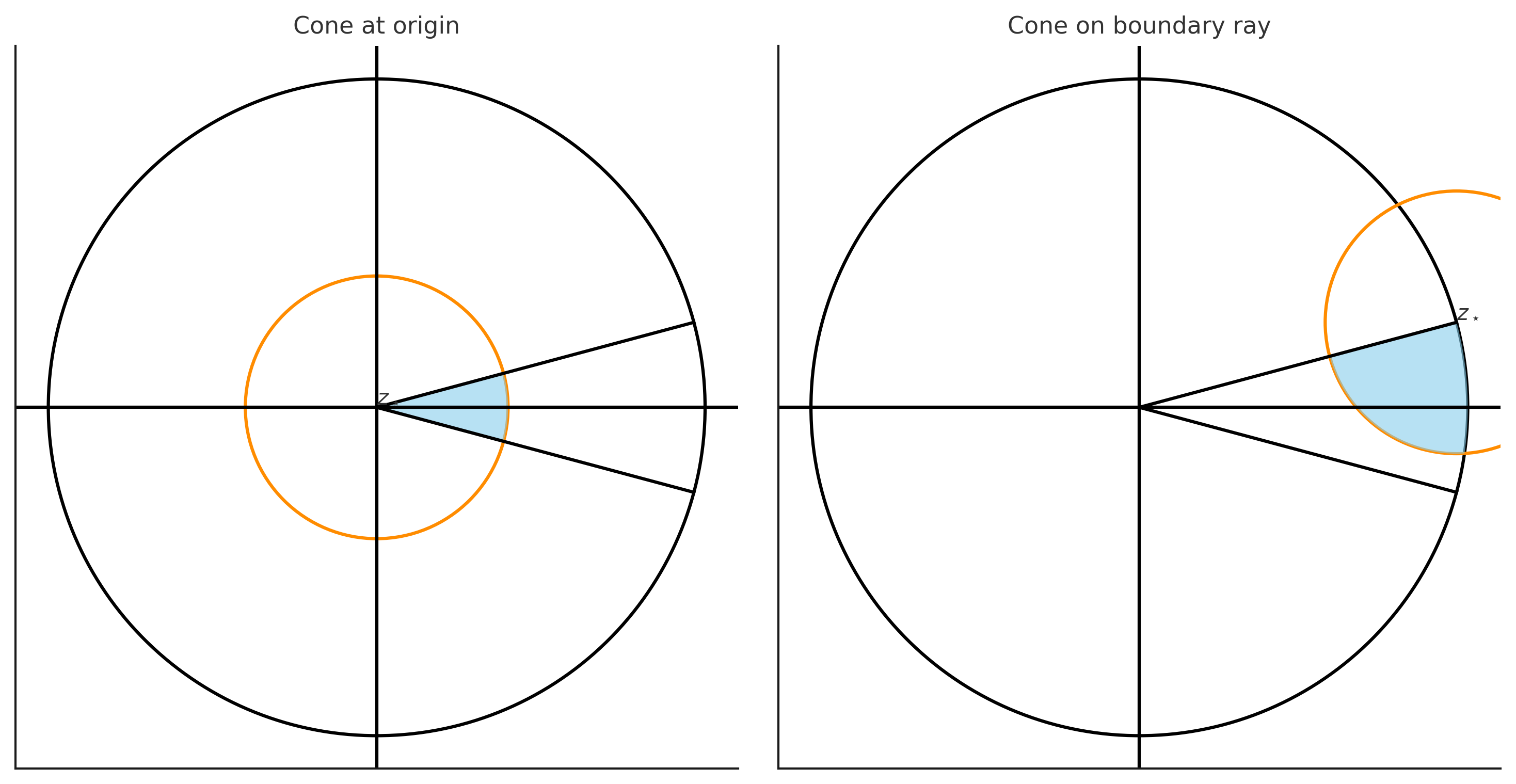}
  \caption{The origin minimizes the intersection area up to factor $\frac{1}{4}$.}
  \label{fig:cone-int}
\end{figure}

\medskip
\textbf{3.
Integrating \(h\).}
\[
\begin{aligned}
  \int_{\mathbb C_\beta} h
  & \geq    \int_{0}^{\beta}\!\!\int_{0}^{r}
       \bigl(g_{\max}-L\rho\bigr)\,\rho
       \,d\rho\,d\varphi
   =  \frac{\beta \,g_{\max}^{3}}{6\,L^{2}}
.
\end{aligned}
\]

Since \(h\in\mathcal F\), and is the extremal function, we obtain the Lemma statement.
\hfill $\qedsymbol$ \end{proof}

\begin{lemma}
\label{thm:novel_spectral_decay}
Recall from Section~\ref{sec:preliminaries}
 \begin{equation*}
     \mat{Z} \defeq \int_{\alpha \in \complex_{\beta}} \tilde{\mu}(\alpha) \tilde{\mu}(\overline{\alpha})^{\top}  d \alpha.
\end{equation*} 
There are at most $2 \log(T)$ eigenvalues with absolute value greater than $\beta$.
\end{lemma}

\begin{proof}[Proof of Lemma~\ref{thm:novel_spectral_decay}]
The $(j,k)$-th entry of $\mat{Z}$ is
\begin{align*}
    \mat{Z}_{jk} &  = \int_{r \in [0,1]} \int_{\theta \in \complex_{\beta}} (1-r^2e^{2i\theta}) (re^{i \theta})^j (1-r^2e^{-2i\theta})(r e^{-i \theta})^k r dr d \theta.
\end{align*}
Evaluating this integral for $j=k$ we get
\begin{align*}
    \mat{Z}_{jj} = \beta \left( \frac{1}{j+1} + \frac{1}{j+3} \right) - \frac{\sin(2 \beta)}{j+2} \leq \frac{3 \beta}{j+1}.
\end{align*}
Therefore,
\begin{align*}
    \mathrm{tr}(\mat{Z}) & \leq 6 \beta \log(T).
\end{align*}
Using the fact that $\mat{Z}$ has nonnegative eigenvalues, we have that the number of eigenvalues larger than $\beta$, $n_{\beta}$, satisfies
\begin{equation*}
    \sum_{i=1}^T \lambda_i \geq \sum_{\lambda_i \geq \beta} \lambda_i \geq \beta n_{\beta}.
\end{equation*}
Since the sum of $\mat{Z}$'s eigenvalues is bounded by its trace we see
\begin{equation*}
    \mathrm{tr}(\mat{Z}) \geq \beta n_{\beta}.
\end{equation*}
Using our bound on the trace of $\mat{Z}$ we have
\begin{equation*}
    n_{\beta} \leq \frac{2 \beta \log(T)}{\beta} = 6 \log (T).
\end{equation*}
\hfill $\qedsymbol$ \end{proof}

With Lemma~\ref{lemma:lipschitz} and Lemma~\ref{thm:novel_spectral_decay} in hand, we are ready to prove Lemma~\ref{lemma:spectral_filtering_property}.
\begin{proof}[Proof of Lemma~\ref{lemma:spectral_filtering_property}]
Let
\begin{equation*}
    f_j(\alpha) \defeq \abs{\phi_j^{\top} \tilde{\mu}(\alpha)}^2,
\end{equation*}
If $f_{j}(\alpha)$ is $L$-Lipschitz, letting $g_{\textrm{max}} = \max_{\alpha \in \complex_{\beta}} f_{j}(\alpha)$ by Lemma~\ref{lemma:lipschitz},
\begin{equation*}
    \int_{\alpha \in \complex_{\beta}} f_{j} (\alpha) d\alpha \geq \frac{\beta g_{\textrm{max}}^3  }{24 L^2},
\end{equation*}
or equivalently,
\begin{equation*}
   \max_{\alpha \in \complex_{\beta}} f_{j}(\alpha)  \leq \left( \frac{24 L^2}{\beta}  \int_{\alpha \in \complex_{\beta}} f_{j}(\alpha) d\alpha \right)^{1/3}.
\end{equation*}
Observe that
\begin{align*}
\int_{\alpha \in \complex_{\beta}} f_{j}(\alpha) d\alpha & = \int_{\alpha \in \complex_{\beta}}  \abs{\phi_j^{\top} \tilde{\mu}(\alpha)}^2 d \alpha \\
& = \phi_j^{\top} \left( \int_{\alpha \in \complex_{\beta}}   \tilde{\mu}(\alpha) \tilde{\mu}(\overline{\alpha})^{\top} d \alpha   \right) \overline{\phi_j}  \\
& = \sigma_j.
\end{align*}
Therefore conclude we have shown,
\begin{equation}
\label{eqn:bound1}
    \max_{\alpha \in \complex_{\beta}} \abs{\tilde{\mu}(\alpha)^{\top} \phi_j} = \max_{\alpha \in \complex_{\beta}}\sqrt{f_j(\alpha)}  \leq \left( \frac{24 L^2}{\beta} \sigma_j \right)^{1/6}.
\end{equation}
The remainder of the proof consists of bounding the Lipschitz constant $L$ and bounding the eigenvalue $\sigma_j$.
To bound the Lipschitz constant of $f_{j}$,
\begin{align*}
  L & \leq \max_{\alpha \in S} \abs{ f'_{j}(\alpha) } \\
  & =  \max_{\alpha \in S} 2 \abs{ \mathrm{Re} \left( \phi_j^{\top} \tilde{\mu}'(\alpha) \cdot \phi_j^{\top} \tilde{\mu}(\alpha) \right)} \\
    & \leq \max_{\alpha \in S} \norm{\phi_j}_2^2 \cdot \norm{\tilde{\mu}(\alpha)}_2 \cdot \norm{\tilde{\mu}'(\alpha)}_2 \\
    & = \max_{\alpha \in S} \norm{\tilde{\mu}(\alpha)}_2 \cdot \norm{\tilde{\mu}'(\alpha)}_2.
\end{align*}
Using Lemma~\ref{lemma:mu_alpha_bnd}, we have $L \leq 12 \beta^4 T^2$. By Lemma~\ref{thm:novel_spectral_decay}, for any $j \in [T]$,
 \begin{equation*}
        \sigma_j \leq 2 \beta \log T.
    \end{equation*}
We conclude,
\begin{equation*}
    \max_{\alpha \in \mathbb{C}_{\beta}} \abs{\tilde{\mu}(\alpha)^{\top} \phi_j} \leq \left( 24 \beta^8 T^4 \log(T) \right)^{1/6} < 2 \log (T)  \beta^{4/3} T^{2/3} .
\end{equation*}
We also bound
\begin{align*}
    \sum_{j=1}^T \abs{\tilde{\mu}(\alpha)^{\top} \phi_j} & \leq \sum_{j=1}^T \left( \frac{24 L^2}{\beta} \sigma_j \right)^{1/6} \\
    & = \left( \frac{24 L^2}{\beta}\right)^{1/6} \sum_{j=1}^T \sigma_j^{1/6} \\
    & = T \left( \frac{24 L^2}{\beta}\right)^{1/6} \sum_{j=1}^T \frac{1}{T} \sigma_j^{1/6} \\
    & \leq T \left( \frac{24 L^2}{\beta}\right)^{1/6}  \left( \sum_{j=1}^T \frac{1}{T} \sigma_j  \right)^{1/6} \tag{$\mathbb{E}[f(X)] \leq f(\mathbb{E}[X])$ for concave $f(\cdot)$ and $f(x) = x^{1/6}$ is concave} \\
    & \leq  T^{5/6} \left( \frac{24 L^2}{\beta}\right)^{1/6}  \left( \mathrm{tr}(\mat{Z} ) \right)^{1/6} \\
    & \leq  T^{5/6} \left( \frac{24 L^2}{\beta}\right)^{1/6}   \left( 
6 \beta \log (T) \right)^{1/6}  \tag{$\mathrm{tr}(\mat{Z}) \leq 6 \beta \log (T)$ by proof of Lemma~\ref{thm:novel_spectral_decay}} \\
    & \leq \left(24 T^5 \cdot 12^2 \beta^8 T^4  \cdot 6 \log T \right)^{1/6} \tag{Plugging in $L \leq 12 \beta^4 T^2$ from above.} \\
    & \leq 6 \beta^{4/3} T^{3/2} \log (T).
\end{align*}
\hfill $\qedsymbol$ \end{proof}

\newpage
\appendix
\section{Proofs of Technical Lemmas}

\begin{lemma}
\label{lemma:mu_alpha_bnd}
For $\beta \ll 1$ and $T$ large enough, the $\ell_2$ norms are bounded by:
\begin{equation*}
\|\tilde{\mu}(\alpha)\|_2 \le \max\left(1, \ 4 \beta^2 \sqrt{T} \right), \quad \|\tilde{\mu}'(\alpha)\|_2 \le \max\left(1, \ 3 \beta^2 T^{3/2} \right).
\end{equation*}
\end{lemma}

\begin{proof}[Proof of Lemma~\ref{lemma:mu_alpha_bnd}]
\textbf{1. Bound on $\|\tilde{\mu}(\alpha)\|_2$.}
The squared magnitude of the entries depends on $|1-\alpha^2|$. Using the triangle inequality and the small angle approximation $|\sin(x)| \le |x|$:
\begin{equation*}
    |1-\alpha^2| = |1 - r^2 e^{2i\theta}|
\end{equation*}
The magnitude is maximized when $r=1$, pushing the value to the boundary of the unit circle.
\begin{equation*}
    |1 - e^{2i\theta}| = 2 |\sin(\theta)| \le 2|\theta| \le 2 \beta
\end{equation*}
Thus, we have the critical inequality for the kernel term:
\begin{equation*}
    |1-\alpha^2| \le 2 \beta \quad (\text{for } r \approx 1)
\end{equation*}
The squared $\ell_2$ norm is:
\begin{equation*}
    \|\tilde{\mu}(\alpha)\|_2^2 = \sum_{j=0}^{T-1} |\tilde{\mu}_j(\alpha)|^2 = |1-\alpha^2|^4 \sum_{j=0}^{T-1} |\alpha|^{2j}
\end{equation*}
We analyze the maximum over the domain $\mathcal{D}$. The maximum of a modulus of a holomorphic function (multiplied by polynomial terms) generally occurs at the boundaries or critical points (here, $z=0$).
\paragraph{Case A: The Origin ($r=0$).} At $\alpha = 0$, the zeroth term is $(1-0)^2 \cdot 1 = 1$. All other terms are 0, so $\|\tilde{\mu}(0)\|_2 = 1$.
\paragraph{Case B: The Boundary ($r=1$).} At $r=1$, the geometric sum is simply $\sum_{j=0}^{T-1} 1 = T$. Using our preliminary bound $|1-\alpha^2| \le 2 \beta$:
\begin{equation*}
    \|\tilde{\mu}(\alpha)\|_2^2 \le(2\beta)^4 \cdot T = 16 \beta^4 T.
\end{equation*}
Taking the square root, $\|\tilde{\mu}(\alpha)\|_2 \le 4 \beta^2 \sqrt{T}$. Combining the cases we have, 
\begin{equation*}
    \|\tilde{\mu}(\alpha)\|_2 \le \max(1,4 \beta^2 \sqrt{T})
\end{equation*}
\textbf{2. Bound on $\|\tilde{\mu}'(\alpha)\|_2$.}
First, we compute the derivative using the product rule:
\begin{align*}
    \tilde{\mu}'_j(\alpha) &= \frac{d}{d\alpha} \left[ (1-\alpha^2)^2 \alpha^j \right] \\
    &= 2(1-\alpha^2)(-2\alpha)\alpha^j + (1-\alpha^2)^2 (j \alpha^{j-1}) \\
    &= \alpha^{j-1} (1-\alpha^2) \left[ -4\alpha^2 + j(1-\alpha^2) \right]
\end{align*}
We seek to maximize the squared norm $\|\tilde{\mu}'(\alpha)\|_2^2 = \sum_{j=0}^{T-1} |\tilde{\mu}'_j(\alpha)|^2$.
\paragraph{Case A: The Origin ($r=0$).} For $j=1$, the term is $\alpha^0 (1-0) [-0 + 1(1)] = 1$. For $j \neq 1$, the term vanishes due to the $\alpha$ factor. Thus $\|\tilde{\mu}'(0)\|_2 = 1$.
\paragraph{Case B: The Boundary ($r=1$).} For large $T$, the sum is dominated by large $j$. For large $j$, the term $j(1-\alpha^2)$ dominates the constant $-4\alpha^2$.
\begin{equation*}
    |\tilde{\mu}'_j(\alpha)| \approx |\alpha|^{j-1} |1-\alpha^2| \cdot j |1-\alpha^2| = j |1-\alpha^2|^2
\end{equation*}
Summing the squares:
\begin{equation*}
    \|\tilde{\mu}'(\alpha)\|_2^2 \approx \sum_{j=0}^{T-1} \left( j |1-\alpha^2|^2 \right)^2 = |1-\alpha^2|^4 \sum_{j=0}^{T-1} j^2
\end{equation*}
Using the summation formula $\sum_{j=0}^{n} j^2 \approx \frac{n^3}{3}$ and the bound $|1-\alpha^2| \le 2 \beta$:
\begin{equation*}
    \|\tilde{\mu}'(\alpha)\|_2^2 \le (2\beta)^4 \cdot \frac{T^3}{3} = 16 \beta^4 \frac{T^3}{3}
\end{equation*}
Taking the square root:
\begin{equation*}
    \|\tilde{\mu}'(\alpha)\|_2 \le \frac{4}{\sqrt{3}} \beta^2 T^{3/2}
\end{equation*}
\hfill $\blacksquare$
\end{proof}
\section{Chebyshev Polynomials Evaluated in the Complex Plane}
\label{appendix:chebyshev}
In this section we let $T_n$ denote the $n$-th Chebyshev polynomial and let $M_n$ denote the monic form. 

\begin{proof}[Proof of Lemma~\ref{lemma:cheby_bound}]
We use that $M_n(z) = T_n(z)/2^{n-1}$ and
\begin{equation}
T_n(z) =  \cos \left( n \arccos(z)\right).
\end{equation}
If $\abs{\mathrm{Im}(z)} \leq 1/64n^2$ then by Lemma~\ref{lemma:arccos}, $\arccos(z) \leq 1/n$. Therefore $n \arccos(z) \leq 1$ and so by Fact~\ref{fact:cos_bound}, 
\begin{equation}
\label{eqn:tbound}
    T_n(z) = \cos( n \arccos(z)) \leq 2
\end{equation}
Now we turn to the derivative $M_n'(z)$. It's a fact that
\begin{equation}
    M_n'(z) = \frac{n}{2^{n-1}} U_{n-1}(z),
\end{equation}
where $U_{n-1}$ is the Chebyshev polynomial of the second kind. We next use the fact that
\begin{equation*}
    U_{n-1}(z) = \begin{cases} 2 \sum_{\substack{j \geq 0 \\ j \textrm{ even } }}^n T_j(z), & n \textrm{ even}, \\
    2 \sum_{\substack{j \geq 0 \\ j \textrm{ odd } }}^n T_j(z), & n \textrm{ odd} .
    \end{cases}
\end{equation*}
By Eq~\eqref{eqn:tbound}, $\abs{T_j(z)} \leq 2$ for any $j$ and therefore
\begin{equation}
    \abs{U_{n-1}(z)} \leq n.
\end{equation}
Therefore, 
\begin{equation*}
    \abs{M_n'(z)} \leq \frac{n^2}{2^{n-1}}.
\end{equation*}
\hfill $\qedsymbol$ \end{proof}

\begin{fact}
\label{fact:cos_bound}
    Let $z \in \complex$. Then $\abs{\cos(z)} \leq 2$ whenever $\abs{\mathrm{Im}} \leq 1$.
\end{fact}
\begin{proof}[Proof of Fact~\ref{fact:cos_bound}]
   \begin{align*}
       \abs{ \cos(x + i y) } & = \left( \cos^2 x \cosh^2 y + \sin^2 x \sinh^2 y \right)^{1/2} \tag{Uses standard complex cosine identity.} \\
       & = \left( \cos^2 x  +  \cos^2 x \left(\cosh^2 y - 1 \right) + \sin^2 x \sinh^2 y \right)^{1/2}  \\
        & = \left( \cos^2 x  +  \cos^2 x \sinh^2 y + \sin^2 x \sinh^2 y \right)^{1/2} \tag{$\cosh^2 y - \sinh^2 y = 1$} \\
         & = \left( \cos^2 x  +   \sinh^2 y \right)^{1/2} \tag{$\cos^2 x + \sin^2 x = 1$} \\
         & \leq \left( 1  +   \sinh^2 y \right)^{1/2} \tag{$\sinh^2 y \leq 2$ when $\abs{y} \leq 1$.} \\
          & \leq 2.
   \end{align*}
\hfill $\qedsymbol$ \end{proof}

\begin{lemma}
\label{lemma:arccos}
    Let $z \in \complex$ with $\abs{z} \leq 1$. Then $\abs{ \mathrm{Im} \left( \arccos(z) \right) } \leq 1/n$ whenever $\abs{\arg(z)} \leq 1/64n^2$.
\end{lemma}
\begin{proof}[Proof of Lemma~\ref{lemma:arccos}]
    Let $re^{i \theta} = z$ and assume $\abs{\theta} \leq 1/64n^2$. We use the Taylor series for $\arccos(\cdot)$,
    \begin{align*}
        \arccos(re^{i \theta}) & = \frac{\pi}{2} - \sum_{k = 0}^{\infty} a_k (re^{i \theta})^{2k + 1} \tag{For $a_k = \frac{(2k)!}{4^k (k!)^2 (2k + 1)}$} \\
        & = \frac{\pi}{2} - \sum_{k = 0}^{\infty} a_k r^{2k + 1} e^{i (2k + 1) \theta} \tag{De Moivre's Theorem}\\
        & = \frac{\pi}{2} - \sum_{k = 0}^{\infty} a_k r^{2k + 1}  \cos ( (2k + 1) \theta)  - i \sum_{k = 0}^{\infty} a_k r^{2k + 1} \sin( (2k + 1) \theta). \tag{$e^{i \theta} = \cos \theta + i \sin \theta$} 
    \end{align*}
    Therefore,
    \begin{equation*}
         \mathrm{Im} \left( \arccos(re^{i \theta}) \right)  = \sum_{k = 0}^{\infty} a_k r^{2k + 1} \sin( (2k + 1) \theta).
    \end{equation*}
    Then
    \begin{align*}
         \abs{ \mathrm{Im} \left( \arccos(re^{i \theta}) \right)  } & \leq \sum_{k = 0}^{\infty} a_k  \abs{r}^{2k+1} \abs{ \sin( (2k + 1) \theta)}  \\
         & \leq \sum_{k = 0}^{\infty} a_k \abs{r}^{2k+1} \min (1, (2k + 1) \abs{ \theta} ) \tag{$\abs{\sin(x)} \leq \min (\abs{x}, 1)$} \\
           & \leq  \sum_{k = 0}^{\infty} a_k  \min (1, (2k + 1) \abs{\arg(z) }  ) \tag{$\abs{r} \leq 1$}  \\
         & \leq \sum_{k = 0}^{K} a_k (2k+1) \abs{\arg(z)} + \sum_{k = K+1}^{\infty} a_k  . \tag{For any arbitrary $K \geq 0$}
    \end{align*}
    Now we bound $a_k$.
\begin{align*}
    a_k (2k + 1)& = \frac{(2k)!}{4^k (k!)^2} \\
    & \leq \frac{\sqrt{2 \pi (2k)} (2k/e)^{2k} \left( 1 + \frac{1}{2k} \right)}{4^k \left( \sqrt{2 \pi k} (k/e)^k \right)^2 } \tag{Stirling's Formula }
    \\
    & = \left( 1 + \frac{1}{2k} \right)/\sqrt{ \pi k}   \\
    & \leq 1/\sqrt{k}.
\end{align*}
Therefore we also have that $a_k \leq 1/k^{3/2}$. Using this (and noting that $a_0 = 1$) we see,
 \begin{align*}
         \abs{ \mathrm{Im} \left( \arccos(z) \right)  } & \leq \abs{\arg(z)} \left( 1 + \sum_{k =1}^K \frac{1}{\sqrt{k}} \right) + \sum_{k = K}^{\infty} \frac{1}{k^{3/2}} \\
         & \leq 4 \left( \abs{\arg(z)} \sqrt{K} + \frac{1}{\sqrt{K}} \right) \\
         & \leq \frac{8}{\sqrt{K}} \tag{For $\abs{\arg(z)}\leq 1/K$} \\
         & \leq \frac{1}{n} \tag{For $K \geq 64 n^2$.}
         \end{align*}
         Therefore, for $\abs{\arg(z)} \leq 1/64n^2$, we have that
         \begin{equation*}
             \abs{\mathrm{Im} ( \arccos (z) ) } \leq 1/n.
         \end{equation*}
\hfill $\qedsymbol$ \end{proof}

\begin{proof}[Proof of Lemma~\ref{lemma:cheby_coeffs_bound}]
We bound the coefficients of the Chebyshev polynomial. From Chapter 22 of \cite{abramowitz1948handbook},
\begin{equation}
\label{eqn:coefficients_cheby}
    T_n(x) = \frac{n}{2} \sum_{m=0}^{\lfloor n/2 \rfloor} (-1)^m \frac{(n-m-1)!}{m!(n-2m)!} (2x)^{n-2m}.
\end{equation}
Therefore
\begin{equation*}
    M_n(x) = \frac{1}{2^{n-1}} T_n(x) = n \sum_{m=0}^{\lfloor n/2 \rfloor} (-1)^m \frac{(n-m-1)!}{m!(n-2m)!} 2^{-2m} x^{n-2m}.
\end{equation*}
Let $c_m = \frac{(n-m-1)!}{m!(n-2m)!} 2^{-2m}$. Then
\begin{align*}
    \max_{m = 0, \dots, n} c_m & \leq \max_{m = 0, \dots, n}  {n-m \choose m} 4^{-m} \\
    & \leq \max_{m = 0, \dots, n} \left( \frac{(n-m)e}{4m}\right)^m \tag{${n \choose k }\leq (ne/k)^k$} \\
    & \leq \max_{c \in [0,1]} \left( \frac{(1-c)e}{4c}\right)^{cn} \tag{Letting $m = cn$} \\
    & \leq 2^{0.3n}. \tag{$\max_{c \in [0,1]} ((1-c)e/4c)^c \leq 2^{0.3}$}
\end{align*}
\hfill $\qedsymbol$ \end{proof}
\section{Plots from Experiments}
The details of the experiments are in Section~\ref{sec:experiments}, please refer to that for specifics.
\subsection{Experiments with Linear Dynamical System Data}
As our theory shows, for Chebyshev polynomials (the same can be shown for Legendre polynomials), the coefficients of the polynomial grow exponentially and therefore the performance gains vanish after the degree is too high. However, learning the optimal coefficients is able to sidestep this issue. 

\begin{figure}[H]
    \centering

    \begin{subfigure}[b]{0.32\textwidth}
        \includegraphics[width=\linewidth]{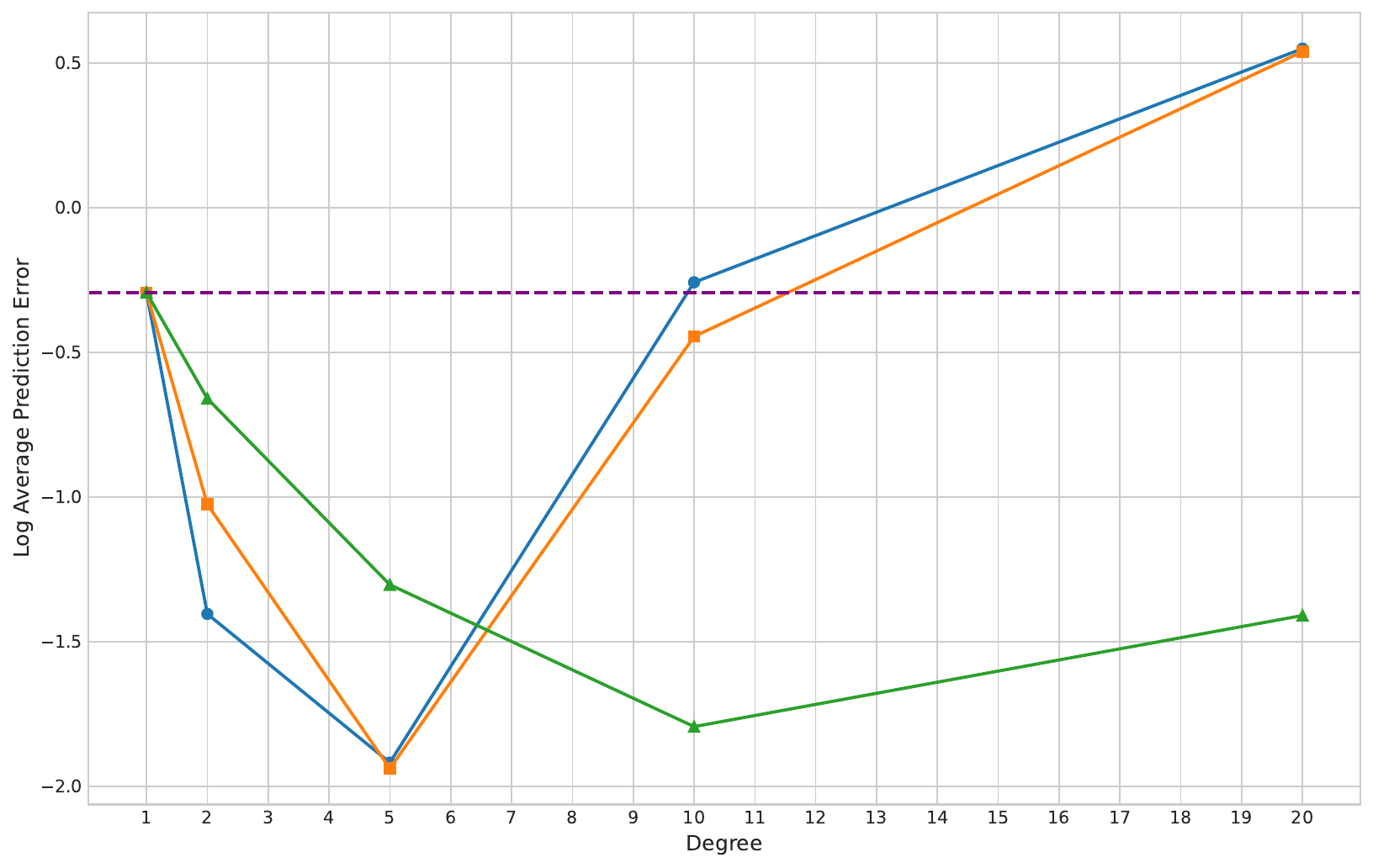}
        \caption{Regression, $\tau_{\textrm{thresh}}$=0.01}
    \end{subfigure}
    \hfill
    \begin{subfigure}[b]{0.32\textwidth}
        \includegraphics[width=\linewidth]{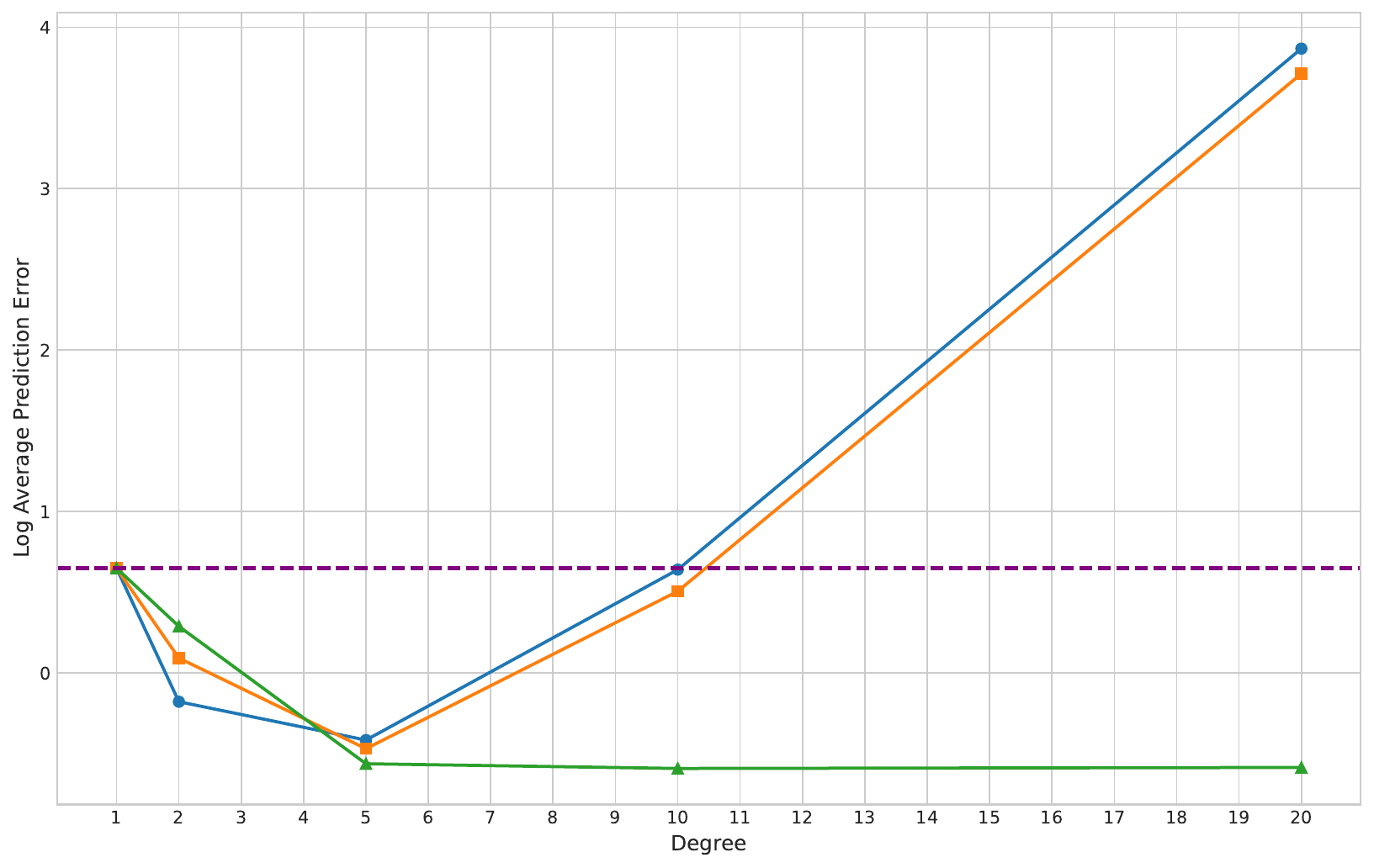}
        \caption{Regression, $\tau_{\textrm{thresh}}$=0.1}
    \end{subfigure}
    \hfill
    \begin{subfigure}[b]{0.32\textwidth}
        \includegraphics[width=\linewidth]{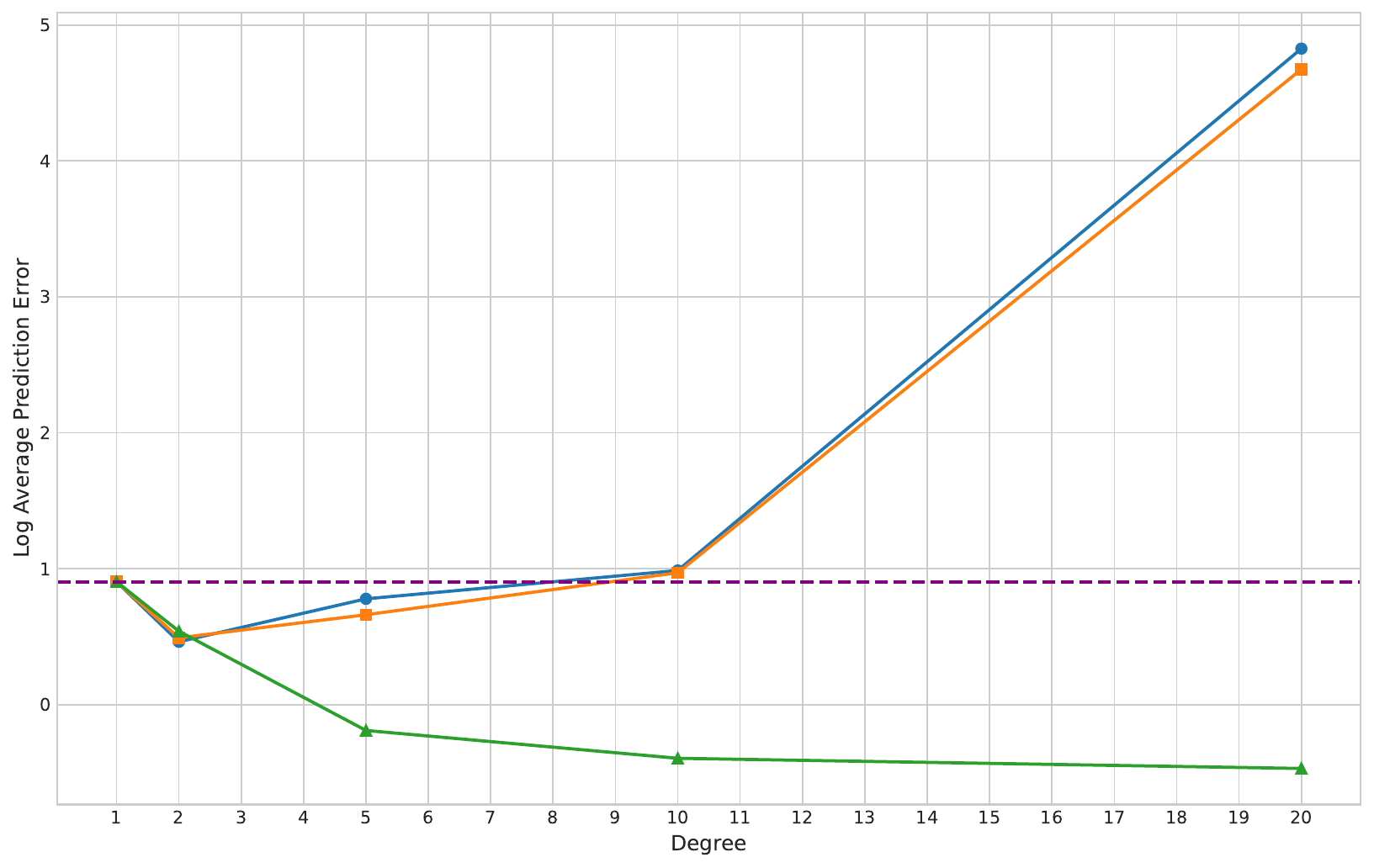}
        \caption{Regression, $\tau_{\textrm{thresh}}=0.9$}
    \end{subfigure}

    \vspace{0.5cm}

    \begin{subfigure}[b]{0.32\textwidth}
        \includegraphics[width=\linewidth]{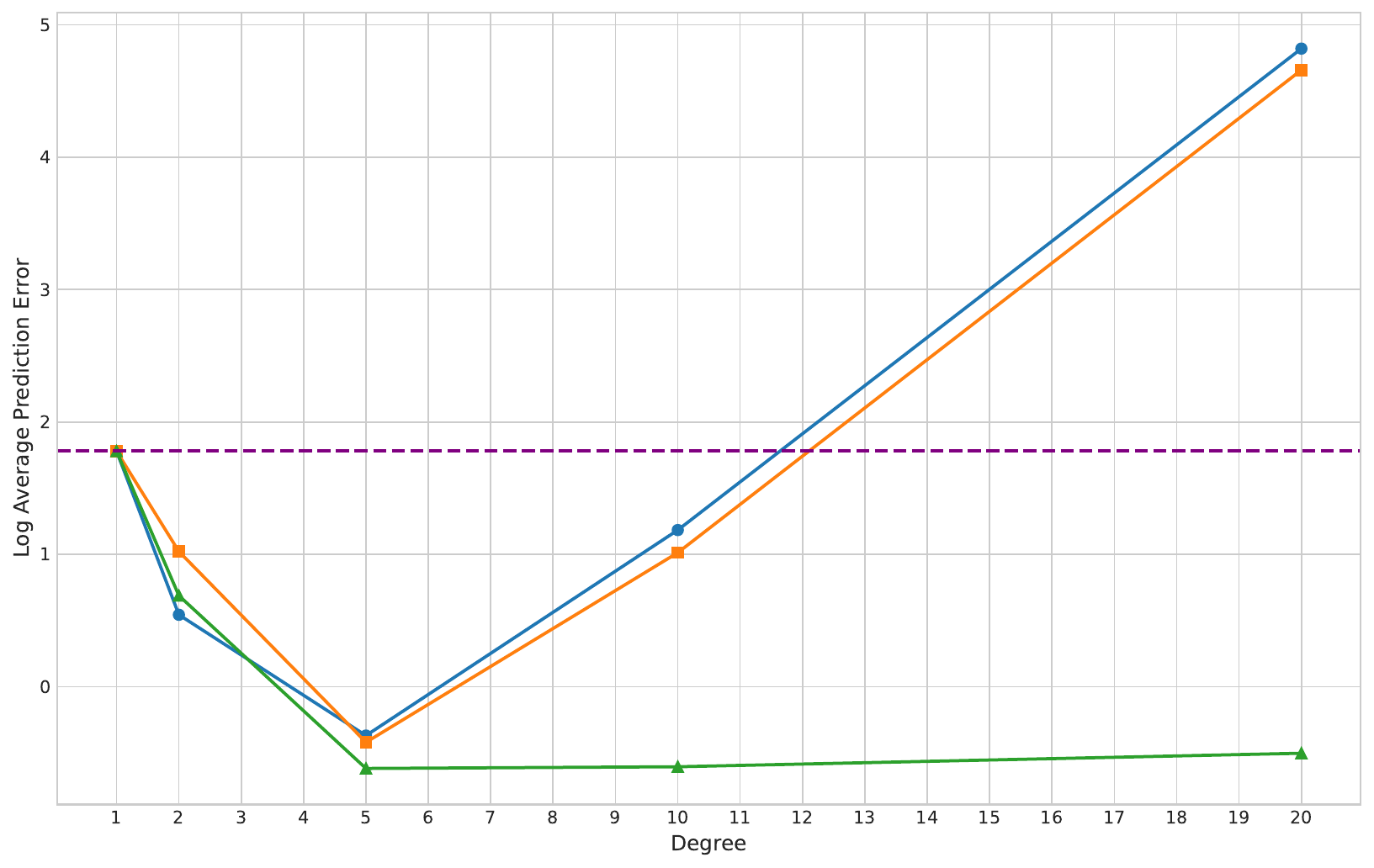}
        \caption{Spectral Filtering, $\tau_{\textrm{thresh}}=0.01$}
    \end{subfigure}
    \hfill
    \begin{subfigure}[b]{0.32\textwidth}
        \includegraphics[width=\linewidth]{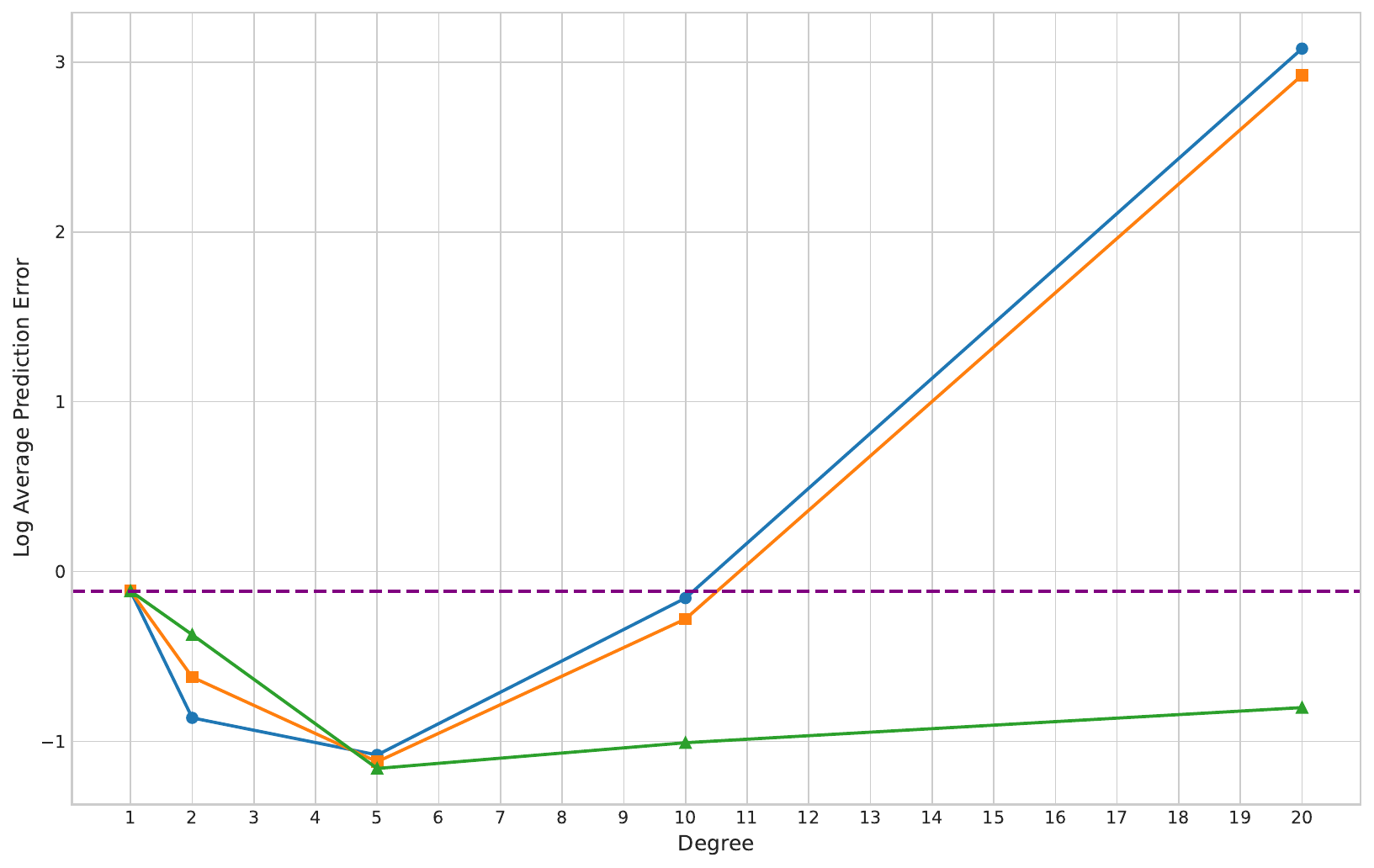}
        \caption{Spectral Filtering, $\tau_{\textrm{thresh}}=0.1$}
    \end{subfigure}
    \hfill
    \begin{subfigure}[b]{0.32\textwidth}
        \includegraphics[width=\linewidth]{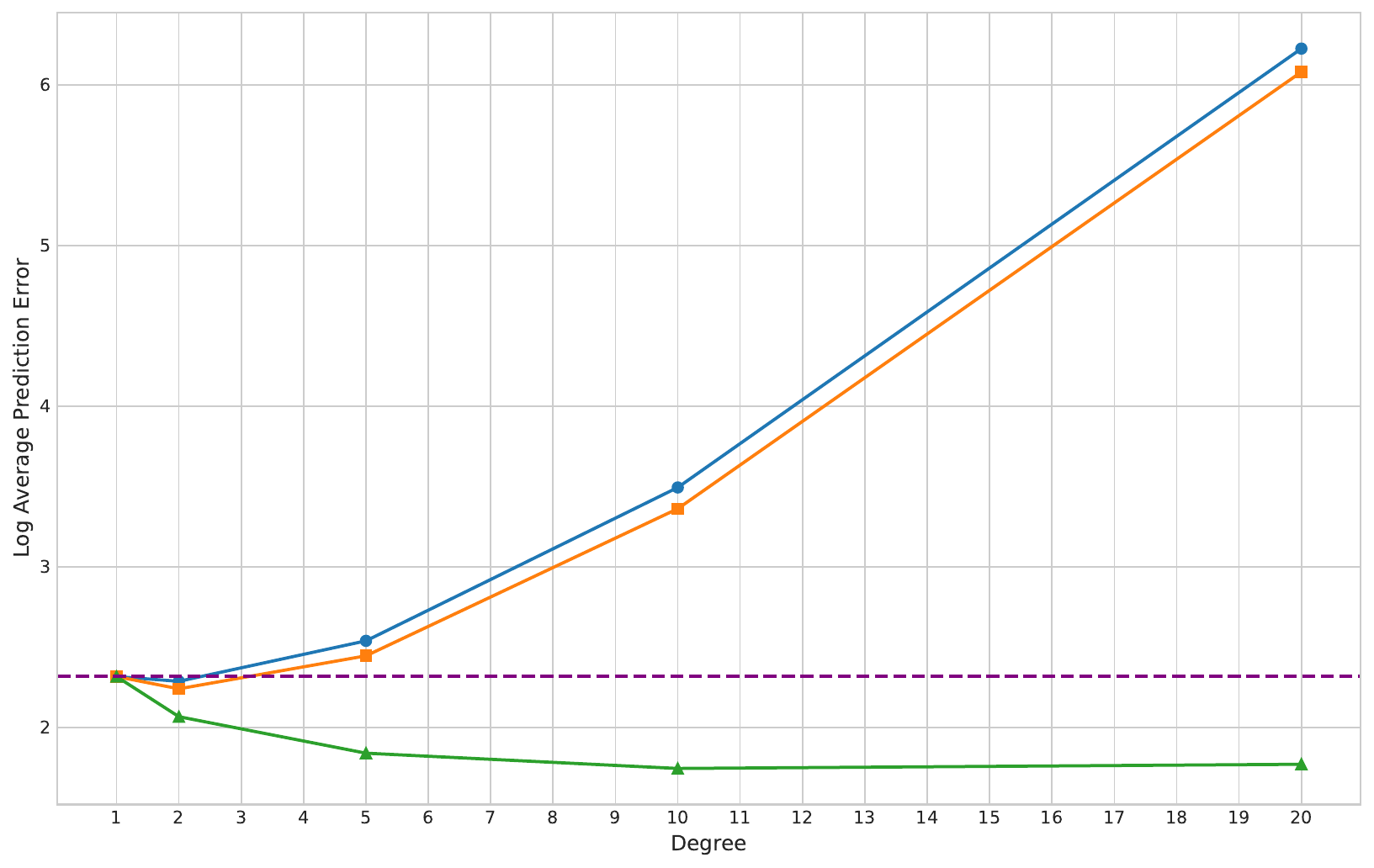}
        \caption{Spectral Filtering, $\tau_{\textrm{thresh}}=0.9$}
    \end{subfigure}

    \vspace{0.5cm}

    \begin{subfigure}[b]{0.32\textwidth}
        \includegraphics[width=\linewidth]{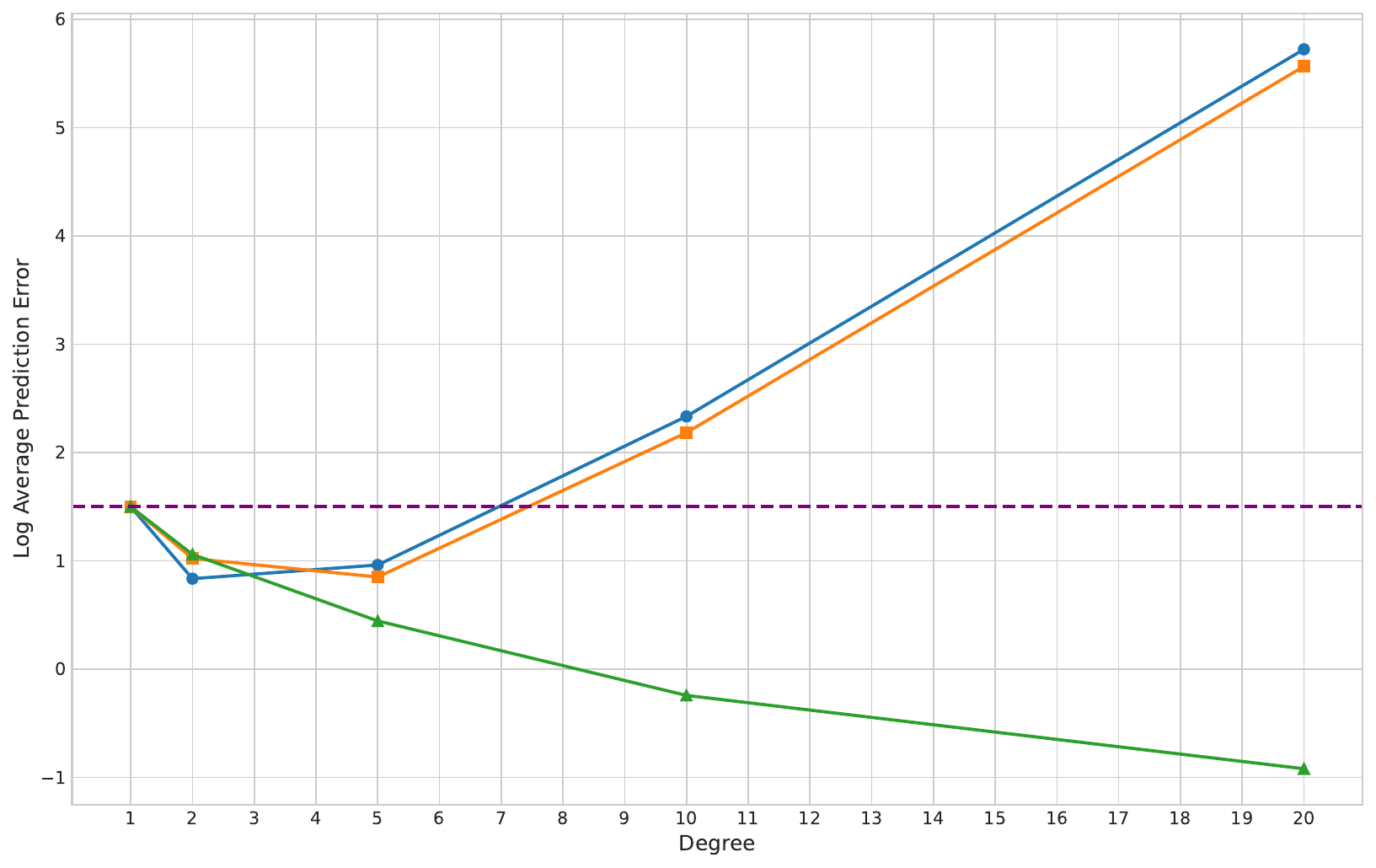}
        \caption{DNN, $\tau_{\textrm{thresh}}=0.01$}
    \end{subfigure}
    \hfill
    \begin{subfigure}[b]{0.32\textwidth}
        \includegraphics[width=\linewidth]{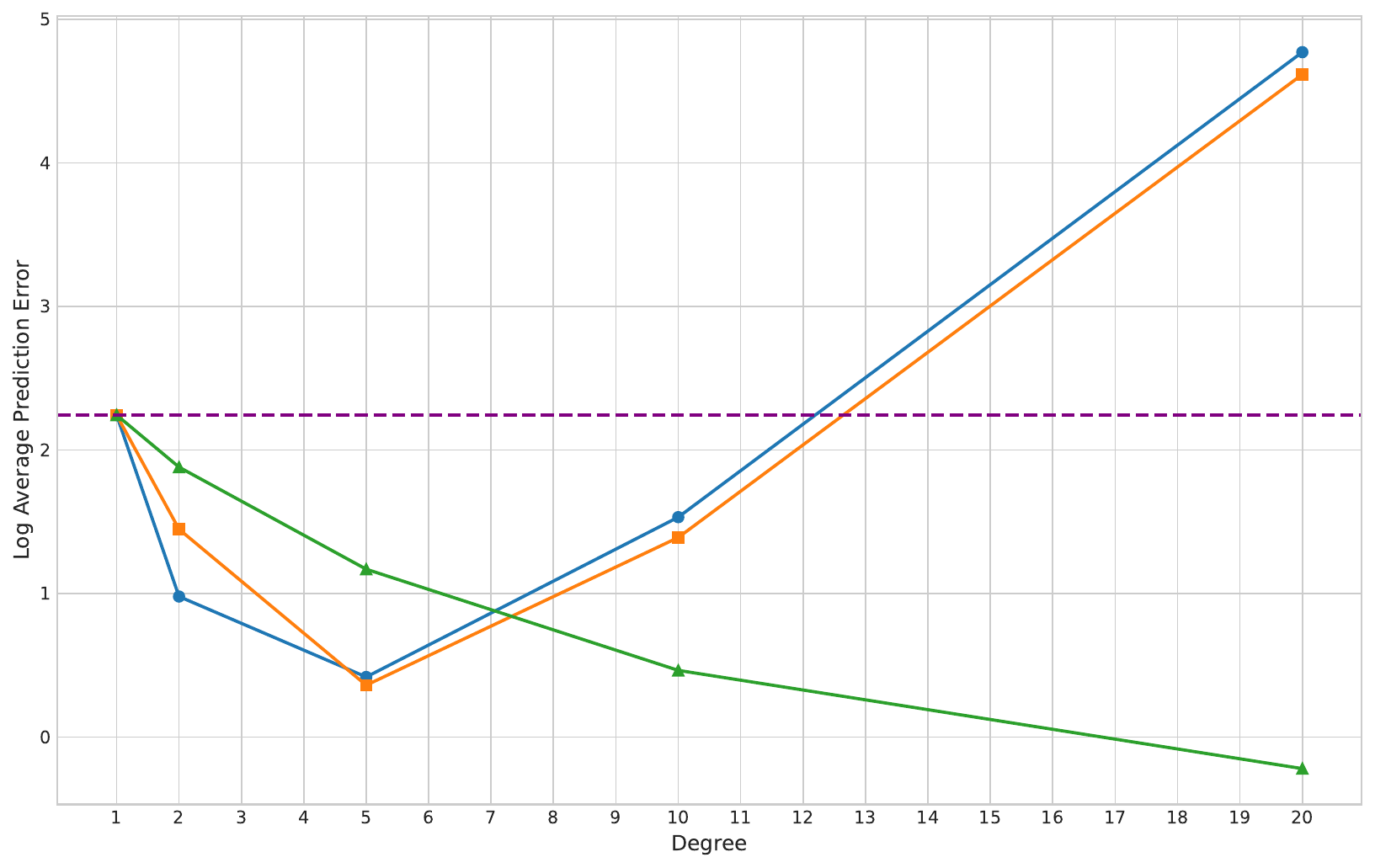}
        \caption{DNN, $\tau_{\textrm{thresh}}=0.1$}
    \end{subfigure}
    \hfill
    \begin{subfigure}[b]{0.32\textwidth}
        \includegraphics[width=\linewidth]{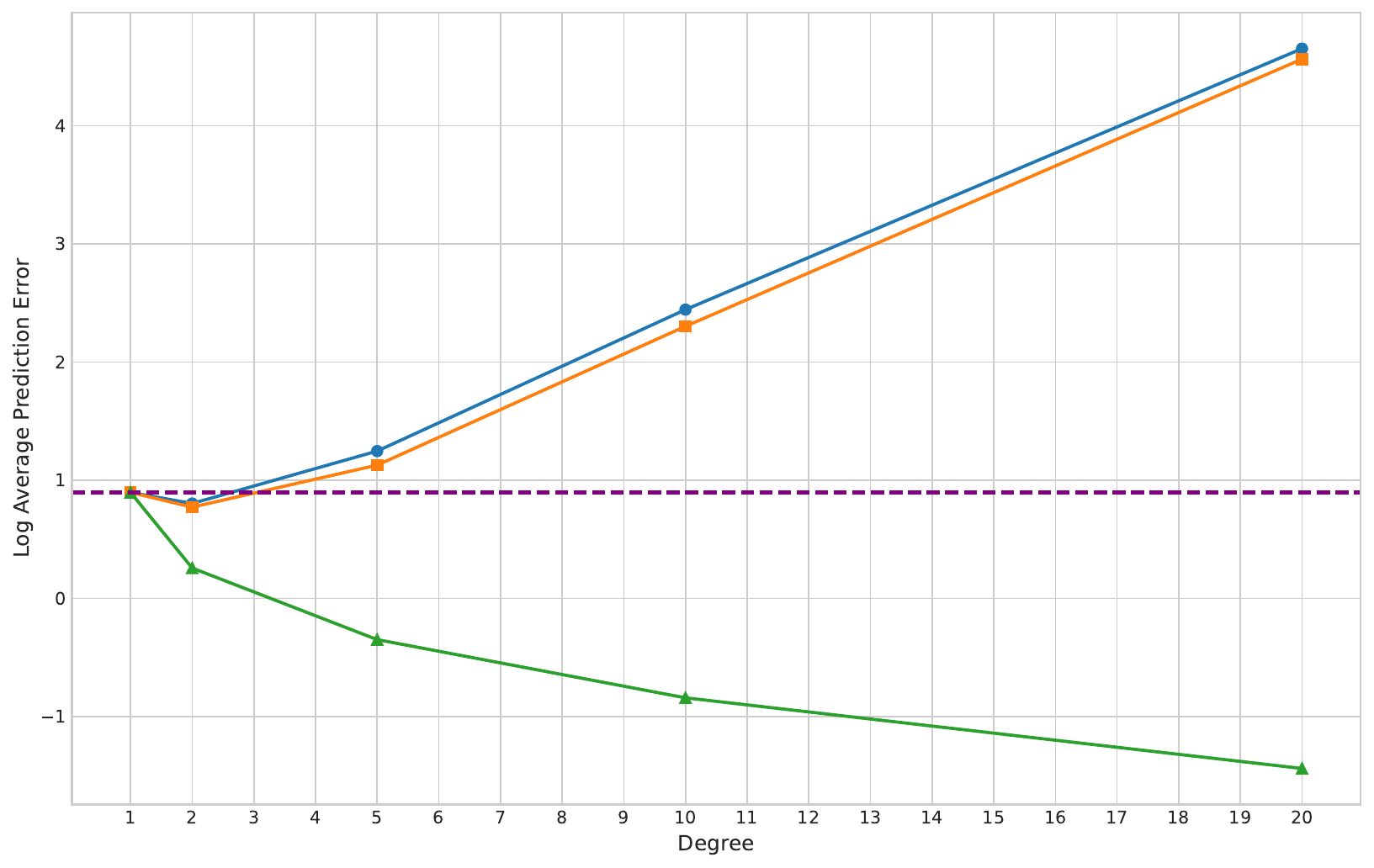}
        \caption{DNN, $\tau_{\textrm{thresh}}=0.9$}
    \end{subfigure}

    \vspace{0.5cm}

    \includegraphics[width=0.12\textwidth]{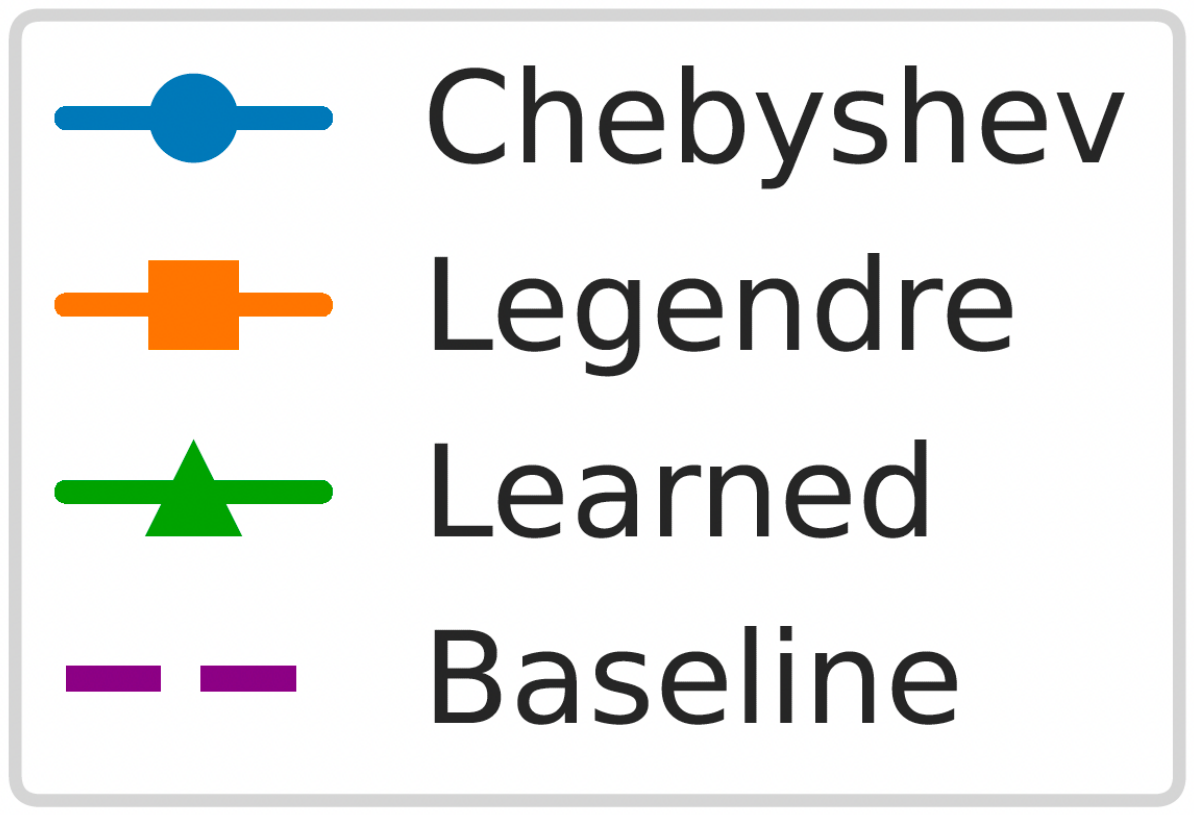}

    \caption{Absolute prediction error averaged over 200 independent runs with data generated from a linear dynamical system with varying complex threshold.}
\end{figure}

\subsection{Nonlinear Data}
\begin{figure}[H]
    \centering

    \begin{subfigure}[b]{0.32\textwidth}
        \includegraphics[width=\linewidth]{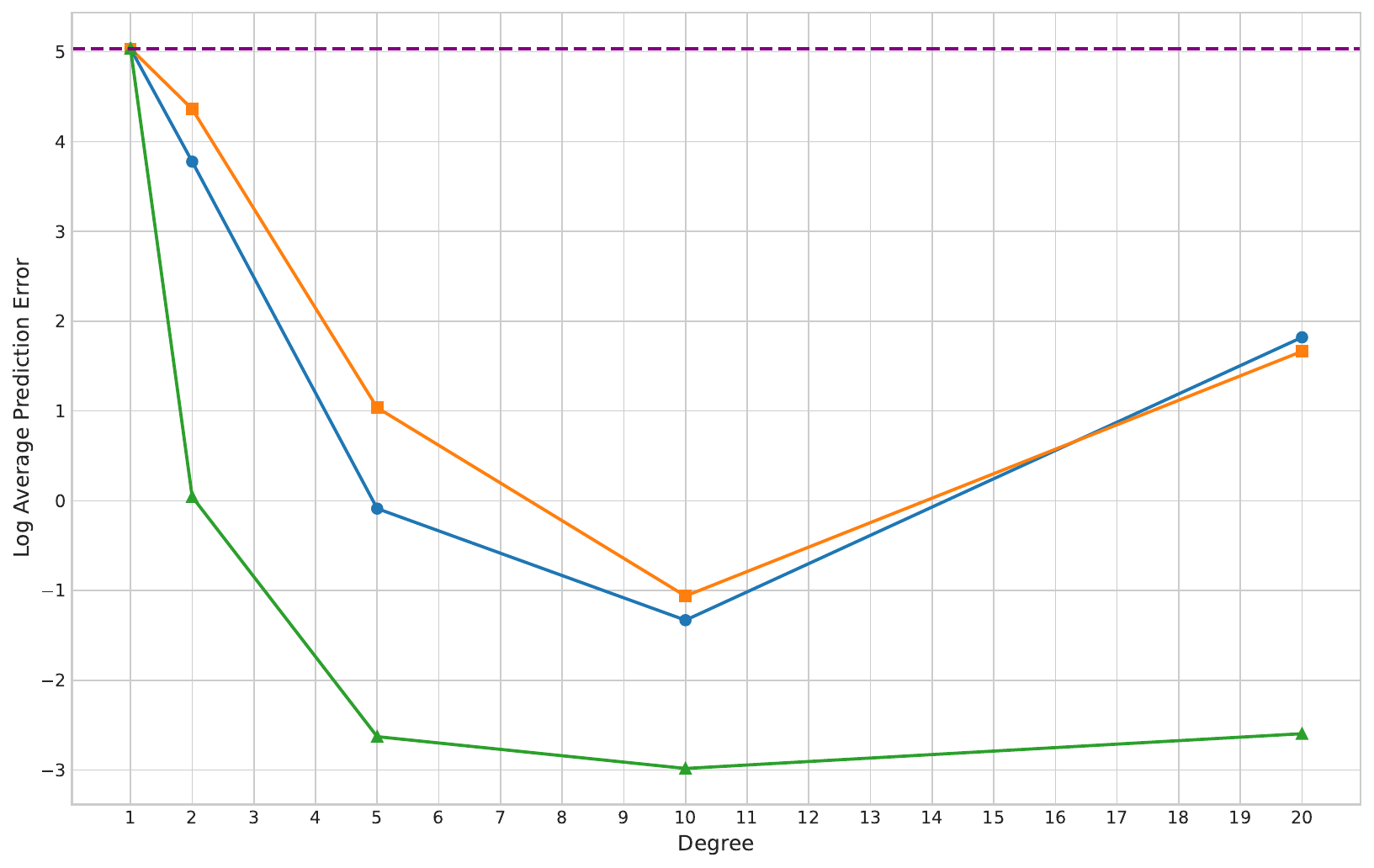}
        \caption{Spectral Filtering, $\tau_{\textrm{thresh}}=0.01$}
    \end{subfigure}
    \hfill
    \begin{subfigure}[b]{0.32\textwidth}
        \includegraphics[width=\linewidth]{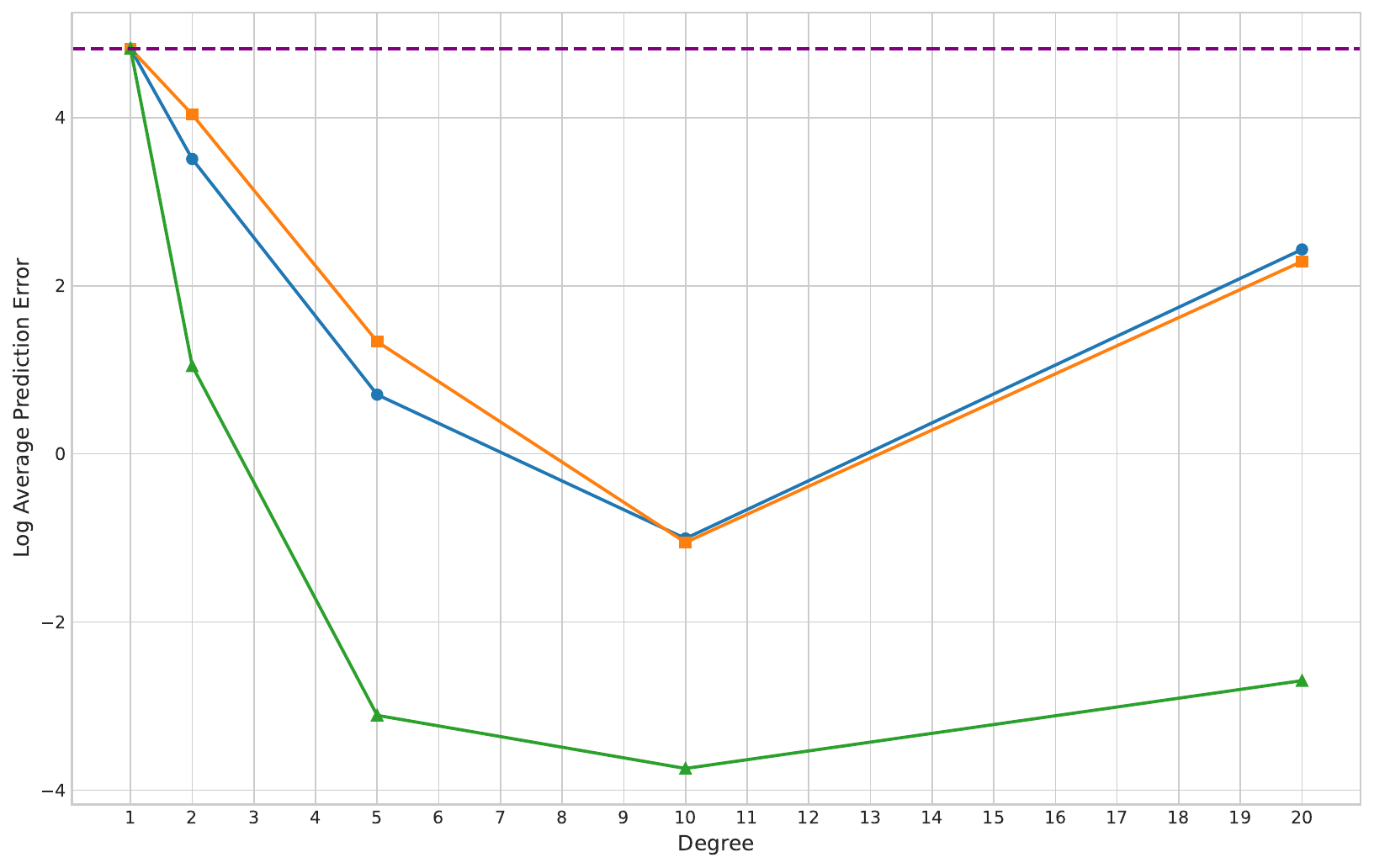}
        \caption{Spectral Filtering, $\tau_{\textrm{thresh}}=0.1$}
    \end{subfigure}
    \hfill
    \begin{subfigure}[b]{0.32\textwidth}
        \includegraphics[width=\linewidth]{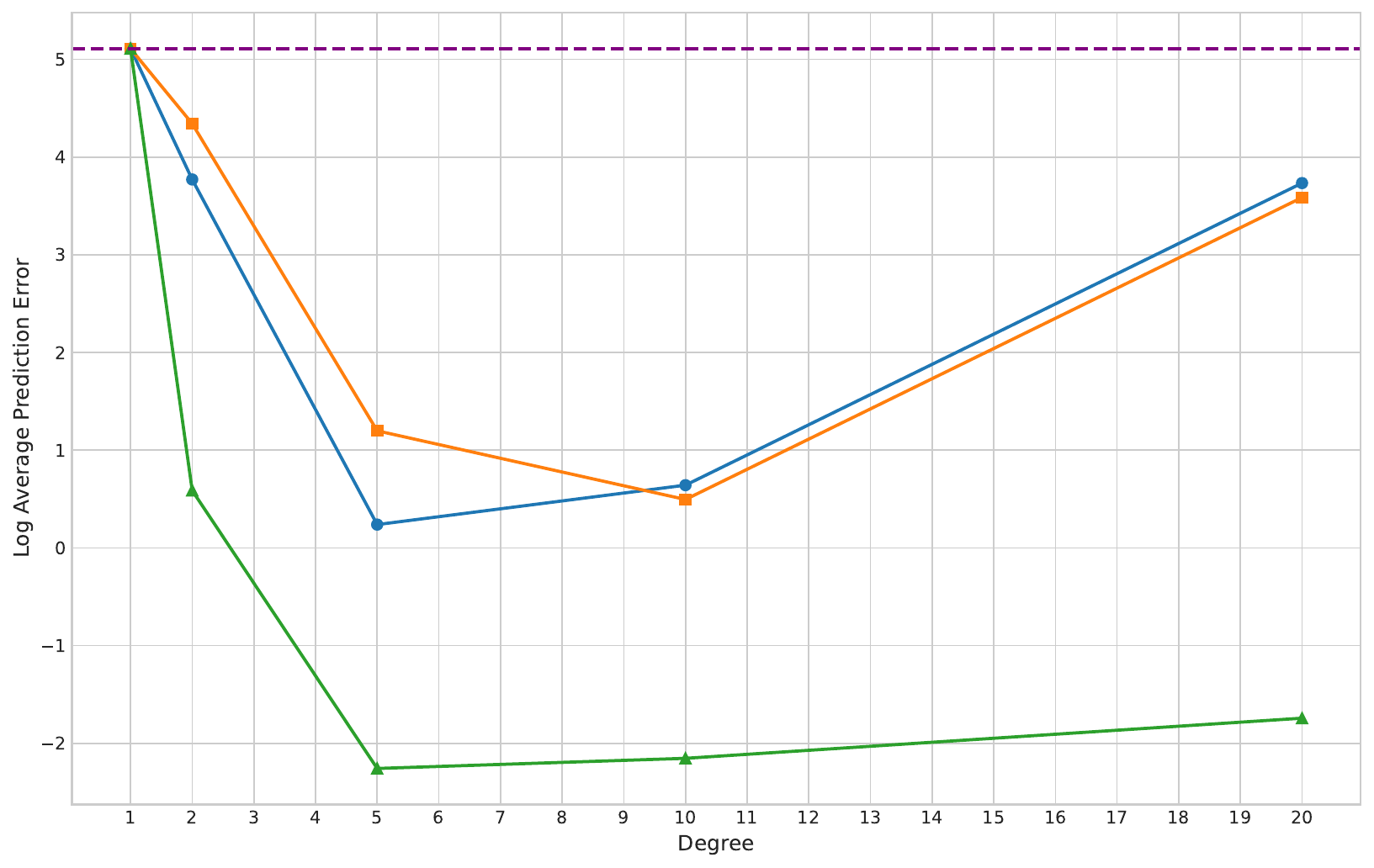}
        \caption{Spectral Filtering, $\tau_{\textrm{thresh}}=0.9$}
    \end{subfigure}

    \vspace{0.5cm}

    \begin{subfigure}[b]{0.32\textwidth}
        \includegraphics[width=\linewidth]{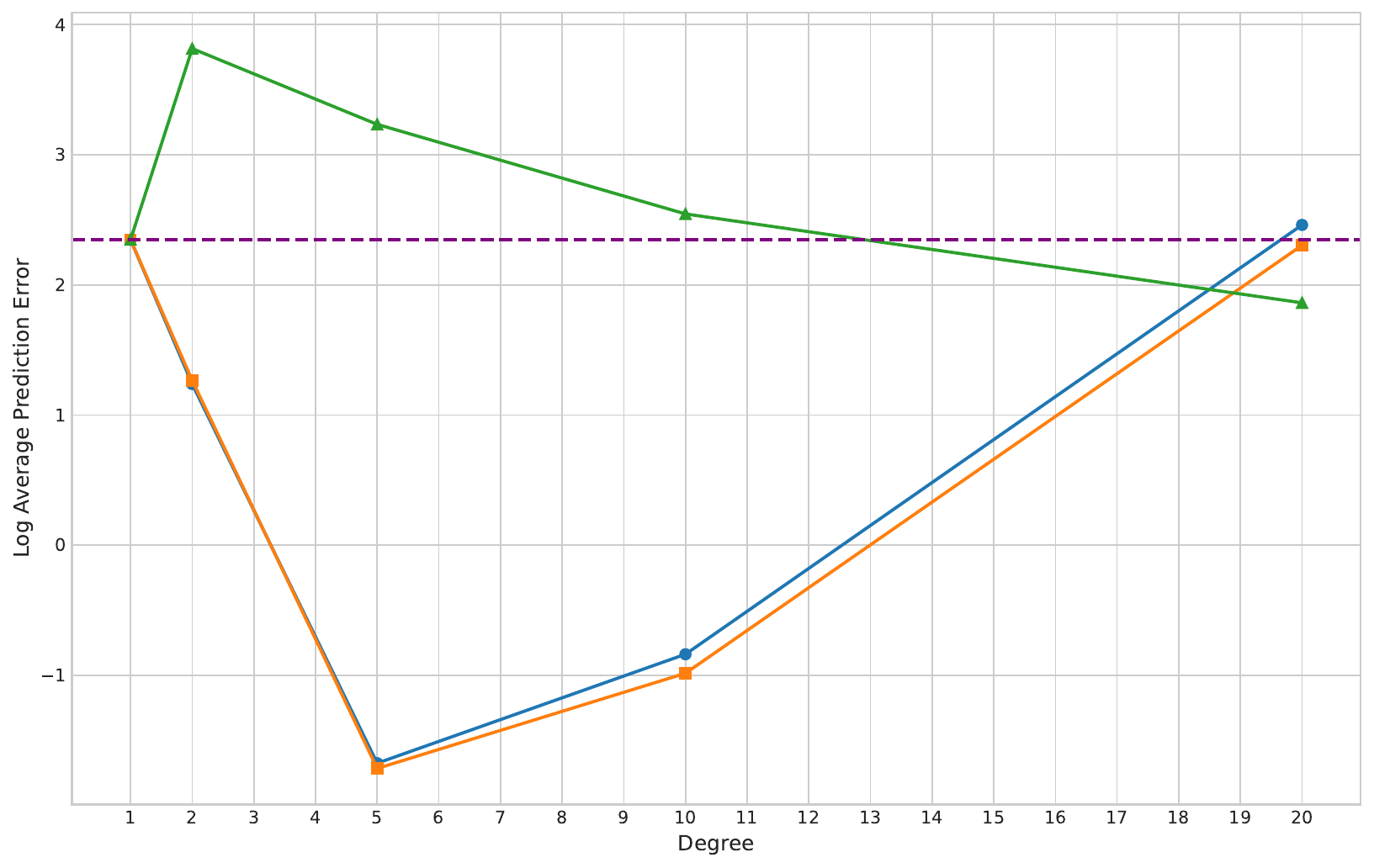}
        \caption{DNN, $\tau_{\textrm{thresh}}=0.01$}
    \end{subfigure}
    \hfill
    \begin{subfigure}[b]{0.32\textwidth}
        \includegraphics[width=\linewidth]{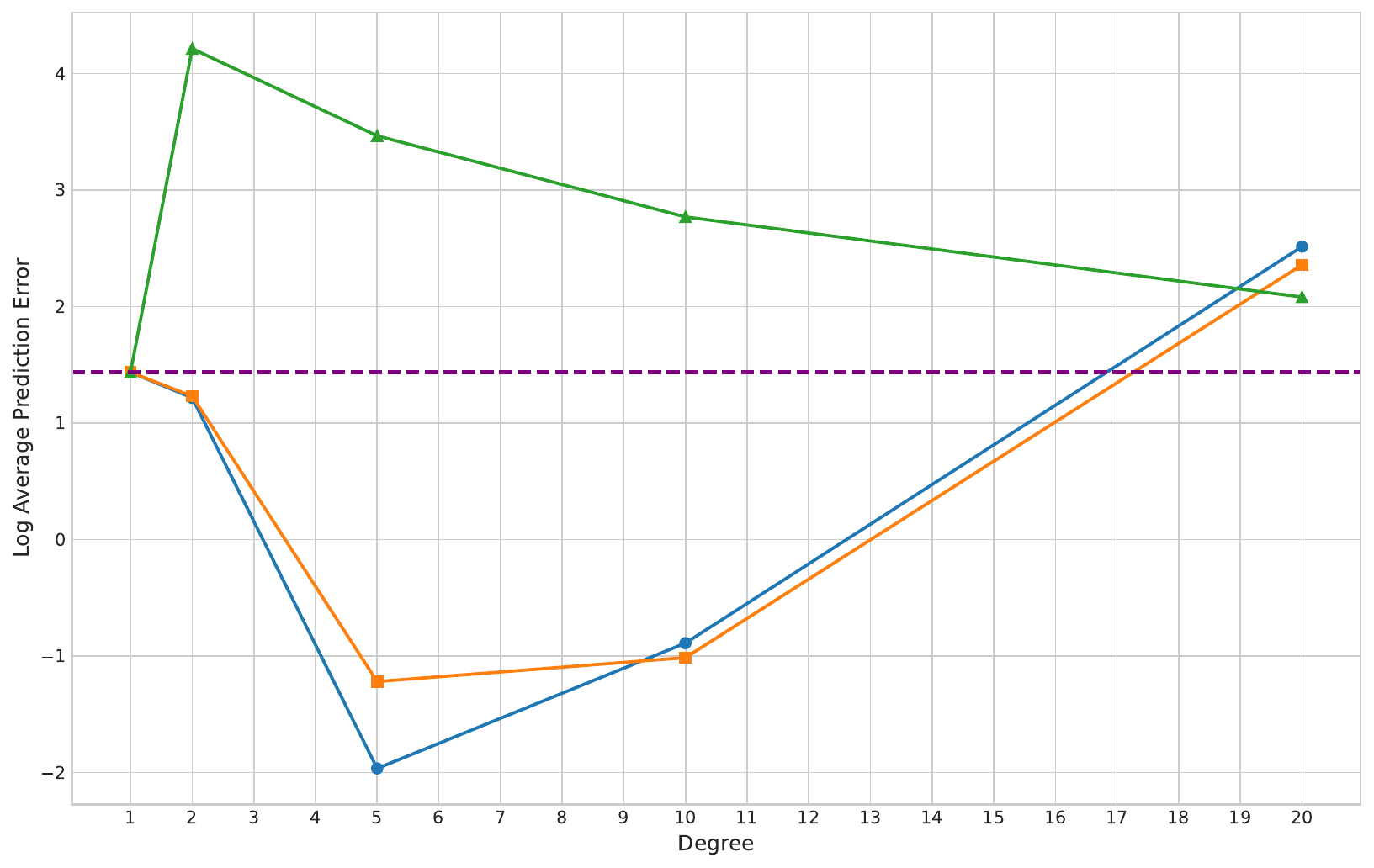}
        \caption{DNN, $\tau_{\textrm{thresh}}=0.1$}
    \end{subfigure}
    \hfill
    \begin{subfigure}[b]{0.32\textwidth}
        \includegraphics[width=\linewidth]{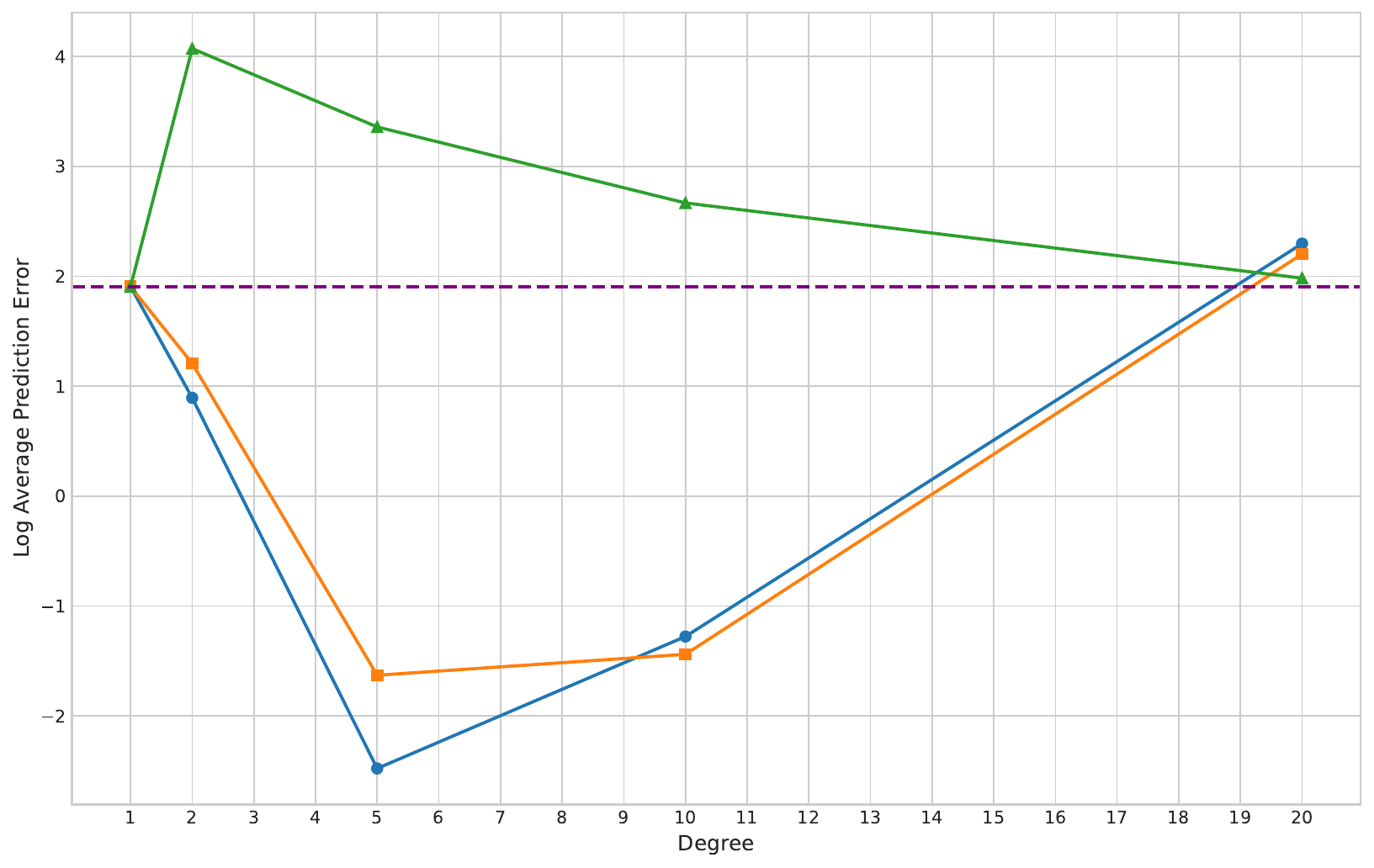}
        \caption{DNN, $\tau_{\textrm{thresh}}=0.9$}
    \end{subfigure}

    \vspace{0.5cm}

    \includegraphics[width=0.12\textwidth]{appendix/plots/per_degree_legend.pdf}

    \caption{Absolute prediction error averaged over 200 independent runs with data generated from a nonlinear dynamical system with varying complex threshold.}
\end{figure}

\newpage

\subsection{Data from a DNN}
Finally we generate data from a 10-layer sparse neural network which stacks 100-dimensional LSTMs with ReLU nonlinear activations.

\begin{figure}[H]
    \centering
    \begin{subfigure}[b]{0.4\textwidth}
        \includegraphics[width=\linewidth]{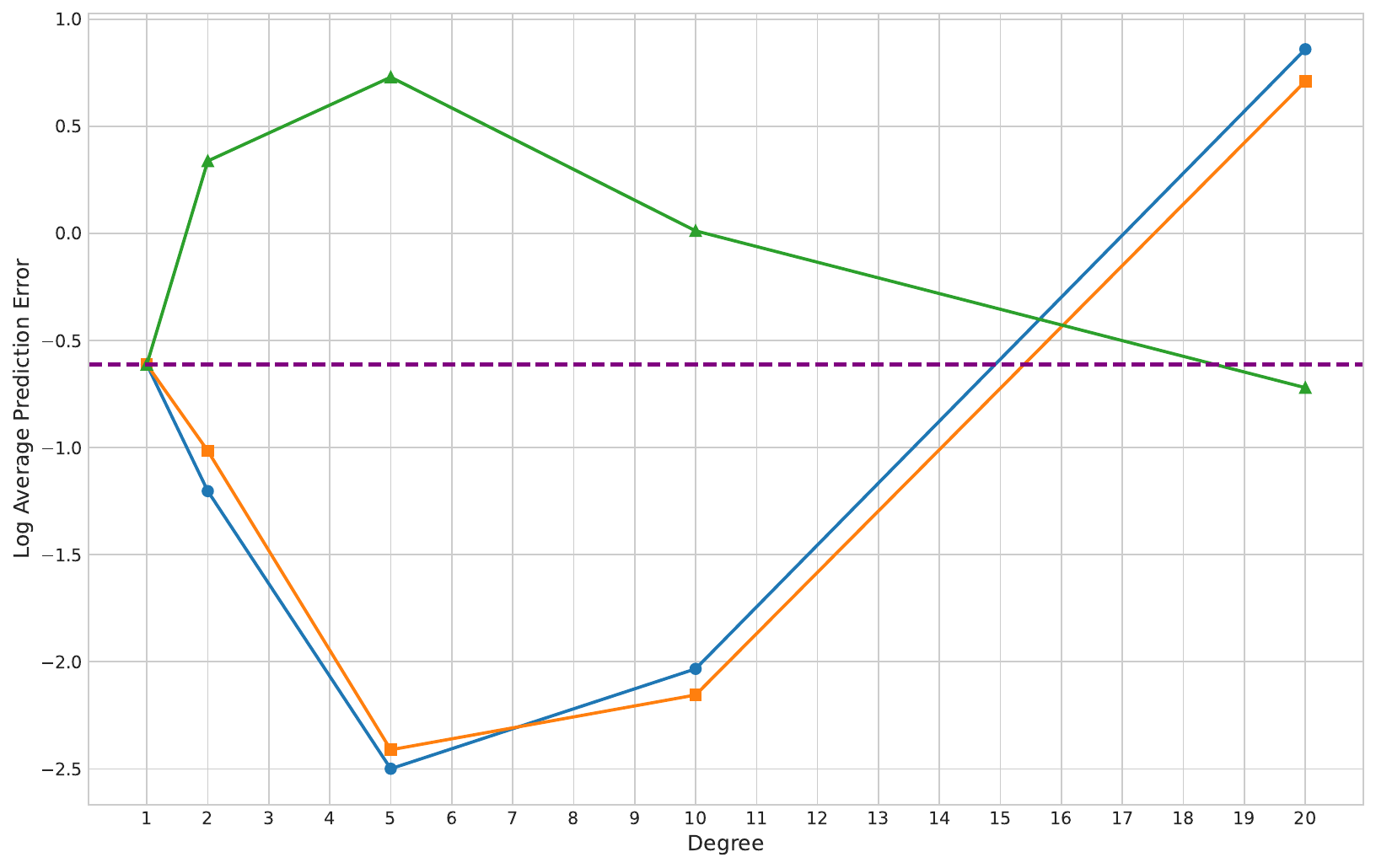}
    \end{subfigure}

    \vspace{0.5cm}

    \includegraphics[width=0.12\textwidth]{appendix/plots/per_degree_legend.pdf}

    \caption{Absolute prediction error of a 10-layer DNN model averaged over 200 independent runs.}
\end{figure}

\end{document}